\documentclass[examplefnt,biber]{nowfnt} 
\usepackage{adjustbox}

\usepackage[utf8]{inputenc}
\usepackage{hyperref}       
\usepackage{url}            
\usepackage{booktabs}       
\usepackage{amsfonts}       
\usepackage{amsmath}
\usepackage{amssymb}
\usepackage{mathtools}
\usepackage{amsthm}
\usepackage{bm}
\usepackage{nicefrac}       
\usepackage{microtype}      
\usepackage{xcolor}         
\usepackage{subcaption}
\usepackage{graphicx}
\usepackage{caption}
\usepackage{algorithm}
\usepackage{algorithmic}

\newtheorem{assumption}{Assumption}[chapter]

\usepackage[utf8]{inputenc} 
\usepackage[T1]{fontenc}    
\usepackage{hyperref}       
\usepackage{url}            
\usepackage{booktabs}       
\usepackage{amsfonts}       
\usepackage{nicefrac}       
\usepackage{microtype}      
\usepackage{mathtools}
\usepackage{xcolor}
\usepackage{amsmath}
\usepackage{mathtools}
\usepackage{amsthm}
\usepackage{amssymb}
\usepackage{enumitem}
\usepackage{bm}

\definecolor{Green}{rgb}{0.13, 0.65, 0.3}
\definecolor{Amber}{rgb}{0.3, 0.5, 1.0}
\usepackage[]{color-edits}
\addauthor{HL}{Green}
\addauthor{LC}{Amber}

\usepackage[algo2e, ruled, vlined]{algorithm2e}
\usepackage{algorithm}
\SetKwBlock{MyRepeat}{repeat}{}

\newcommand{\calA}{{\mathcal{A}}}

\newcommand{\calV}{{\mathcal{V}}}

\newcommand{\calS}{{\mathcal{S}}}
\newcommand{\calF}{{\mathcal{F}}}

\newcommand{\calP}{{\mathcal{P}}}

\DeclareMathOperator*{\argmax}{argmax}

\newcommand{\eat}[1]{}

\newcommand{\inner}[2]{\left\langle #1, #2 \right\rangle}

\newcommand{\rbr}[1]{\left(#1\right)}
\newcommand{\sbr}[1]{\left[#1\right]}
\newcommand{\cbr}[1]{\left\{#1\right\}}

\newcommand{\abr}[1]{\left|#1\right|}

\newcommand{\tilO}[1]{\otil\left( #1 \right)}

\newcommand{\tilo}[1]{\otil( #1 )}

\newcommand{\var}{\textsc{Var}}

\newcommand{\SA}{\calS\times\calA}

\renewcommand{\sp}{\text{\rm sp}}
\renewcommand{\P}{\bar{P}}

\newcommand{\Np}{\N^+}
\newcommand{\optJ}{J^{\star}}

\newcommand{\sumt}{\sum_{t=1}^T}

\newcommand{\sumh}{\sum_{h=1}^H}

\newcommand{\sumk}{\sum_{k=1}^K}

\newcommand{\optpi}{\pi^\star}

\newcommand{\tilP}{\widetilde{P}}

\newcommand{\tilpi}{{\widetilde{\pi}}}
\newcommand{\tiloptpi}{{\widetilde{\pi}^\star}}

\newcommand{\N}{N} 

\newcommand{\optpieps}{\pi^{\star,\epsilon}}
\newcommand{\optJeps}{J^{\star,\epsilon}}
\newcommand{\tmix}{t_{\text{mix}}}

\newcommand{\field}[1]{\mathbb{#1}}

\newcommand{\fR}{\field{R}}

\newcommand{\E}{\field{E}}
\newcommand{\fV}{\field{V}}

\newcommand{\Ind}{\field{I}}

\newcommand{\otil}{\ensuremath{\tilde{\mathcal{O}}}}

\usepackage{prettyref}
\newcommand{\pref}[1]{\prettyref{#1}}
\newcommand{\pfref}[1]{Proof of \prettyref{#1}}
\newcommand{\savehyperref}[2]{\texorpdfstring{\hyperref[#1]{#2}}{#2}}
\newrefformat{eq}{\savehyperref{#1}{Eq.~\textup{(\ref*{#1})}}}
\newrefformat{eqn}{\savehyperref{#1}{Equation~\ref*{#1}}}
\newrefformat{lem}{\savehyperref{#1}{Lemma~\ref*{#1}}}
\newrefformat{def}{\savehyperref{#1}{Definition~\ref*{#1}}}
\newrefformat{line}{\savehyperref{#1}{Line~\ref*{#1}}}
\newrefformat{thm}{\savehyperref{#1}{Theorem~\ref*{#1}}}
\newrefformat{corr}{\savehyperref{#1}{Corollary~\ref*{#1}}}
\newrefformat{cor}{\savehyperref{#1}{Corollary~\ref*{#1}}}
\newrefformat{sec}{\savehyperref{#1}{Section~\ref*{#1}}}
\newrefformat{subsec}{\savehyperref{#1}{Section~\ref*{#1}}}
\newrefformat{app}{\savehyperref{#1}{Appendix~\ref*{#1}}}
\newrefformat{assum}{\savehyperref{#1}{Assumption~\ref*{#1}}}
\newrefformat{ex}{\savehyperref{#1}{Example~\ref*{#1}}}
\newrefformat{fig}{\savehyperref{#1}{Figure~\ref*{#1}}}
\newrefformat{alg}{\savehyperref{#1}{Algorithm~\ref*{#1}}}
\newrefformat{rem}{\savehyperref{#1}{Remark~\ref*{#1}}}
\newrefformat{conj}{\savehyperref{#1}{Conjecture~\ref*{#1}}}
\newrefformat{prop}{\savehyperref{#1}{Proposition~\ref*{#1}}}
\newrefformat{proto}{\savehyperref{#1}{Protocol~\ref*{#1}}}
\newrefformat{prob}{\savehyperref{#1}{Problem~\ref*{#1}}}
\newrefformat{claim}{\savehyperref{#1}{Claim~\ref*{#1}}}
\newrefformat{que}{\savehyperref{#1}{Question~\ref*{#1}}}
\newrefformat{op}{\savehyperref{#1}{Open Problem~\ref*{#1}}}
\newrefformat{fn}{\savehyperref{#1}{Footnote~\ref*{#1}}}
\newrefformat{tab}{\savehyperref{#1}{Table~\ref*{#1}}}
\newrefformat{fig}{\savehyperref{#1}{Figure~\ref*{#1}}}
\newrefformat{prop}{\savehyperref{#1}{Property~\ref*{#1}}}

\newcommand{\norm}[1]{\left \lVert #1 \right\rVert }

\title{Constrained Reinforcement Learning with Average Reward Objective: Model-Based and Model-Free Algorithms}

\maintitleauthorlist{
Vaneet Aggarwal \\
Purdue University \\
vaneet@purdue.edu
\and
Washim Uddin Mondal\\
Indian Institute of Technology Kanpur \\
wmondal@iitk.ac.in
\and 
Qinbo Bai \\
Purdue University \\
bai113@purdue.edu
}

\usepackage{mwe}

\author[1]{Aggarwal, Vaneet}
\author[2]{Mondal, Washim Uddin}
\author[1]{Bai, Qinbo}

\affil[1]{Purdue University; \{vaneet, bai113\}@purdue.edu}
\affil[2]{Indian Institute of Technology Kanpur; wmondal@iitk.ac.in}

\articledatabox{\nowfntstandardcitation}

\renewcommand{\cite}[1]{\citep{#1}}

\begin{document}

\makeabstracttitle

\begin{abstract}
Reinforcement Learning (RL) serves as a versatile framework for sequential decision-making, finding applications across diverse domains such as robotics, autonomous driving, recommendation systems, supply chain optimization, biology, mechanics, and finance. The primary objective in these applications is to maximize the average reward. Real-world scenarios often necessitate adherence to specific constraints during the learning process.

This monograph focuses on the exploration of various model-based and model-free approaches for Constrained RL within the context of average reward Markov Decision Processes (MDPs). The investigation commences with an examination of model-based strategies, delving into two foundational methods – optimism in the face of uncertainty and posterior sampling. Subsequently, the discussion transitions to parametrized model-free approaches, where the primal dual policy gradient-based algorithm is explored as a solution for constrained MDPs. The monograph provides regret guarantees and analyzes constraint violation for each of the discussed setups.

For the above exploration, we assume the underlying MDP to be ergodic. Further, this monograph extends its discussion to encompass results tailored for weakly communicating MDPs, thereby broadening the scope of its findings and their relevance to a wider range of practical scenarios.	
\end{abstract}
\chapter{Introduction}

Reinforcement Learning (RL) describes a class of problems where an agent repeatedly interacts with an unknown environment. The environment possesses a state that changes as a result of the action executed by the agent according to some pre-determined but unknown probability law. The environment also generates feedback, which is often called the reward. The agent's goal is to choose a sequence of actions (based on the sequence of observed states and rewards) that maximizes the expected cumulative sum of rewards obtained via this procedure. This model has found its application in a wide array of areas, ranging from networking to transportation to robotics to epidemic control \citep{geng2020multi,chen2023two,al2019deeppool,manchella2021passgoodpool,gonzalez2023asap,ling2023cooperating}. RL problems are typically analyzed via three distinct setups$-$episodic, infinite horizon discounted reward, and infinite horizon average reward. In an episodic setup, the environment restores its initial state after a certain number of interactions. Examples include video game-based applications where the learner restarts the game after either winning or losing it. In a discounted setup, the learner aims to maximize the expected \textit{discounted} sum of rewards. The underlying philosophy is that the current reward, in certain applications, is deemed more valuable than the rewards obtained in the future. This idea naturally fits into financial applications where the reward (money) loses value over time due to inflation. The average reward setup, on the contrary, places both the current and future rewards on the same footing and aims to maximize the expected average reward computed over an infinitely long time horizon. The basic premise of the infinite horizon average reward setup aligns with most practical scenarios due to its ability to capture essential long-term behaviors. Some applications in real life require the learning procedure to respect the boundaries of certain constraints. In an epidemic control setup, for example, vaccination policies must take the supply shortage (budget constraint) into account. Such restrictive decision-making routines are described by a constrained Markov Decision Process (CMDP) \citep{bai2023achieving,agarwal2022concave,mondal2024sample}. This monograph aims to provide the key approaches to tackle CMDP with an average reward objective. 

To gain more insight into CMDPs, consider a wireless sensor network where a device aims to update a server with its sensed values. At time $t$, the sensor can either choose to send a packet which, upon successful transmission, fetches a reward of one unit or to queue the packet and obtain a zero reward. However, communicating a packet results in $p_t$ power consumption. The success probability of the intended packet is decided via a pre-determined but unknown function of $p_t$ and the current wireless channel condition, $s_t$. The goal is to send as many packets as possible while keeping the average power consumption, $\sum_{t=1}^Tp_t/T$, within some limit, say $C$. The \textit{state} of the environment can be described by the pair $(s_t, q_t)$ where $s_t$, as stated above, is the channel condition, and $q_t$ is the queue length at time $t$. To limit the power consumption, the agent may choose to transmit packets when the channel condition is good or when the queue length grows beyond a certain threshold. The agent aims to learn the policies in an \textit{online manner} which requires efficiently balancing exploration of state-space and exploitation of the estimated system dynamics \citep{singh2020learning}.

Similar to the example above, many applications require keeping some costs low while simultaneously maximizing the rewards \citep{altman1999constrained}. This monograph discusses model-based and model-free algorithms for the CMDP learning problem described above. A model-based algorithm aims to learn the optimal policy by creating a good estimate of the state-transition function of the underlying CMDP. The caveat of the model-based approach is the large memory requirement to store the estimated parameters which effectively curtails its applicability to large state space CMDPs. The alternative strategy, known as the model-free approach, either directly estimates the policy function or maintains an estimate of the $Q$ function, which is subsequently used for policy generation \citep{wei2022provably}. Model-free algorithms typically demand lower memory and computational resources than their model-based counterparts.  

The problem setup, where the system dynamics are known, is extensively studied \citep{altman1999constrained}. For a constrained setup, the optimal policy is possibly stochastic \citep{altman1999constrained,puterman2014markov}. Even though the problem has been widely studied in episodic and discounted reward setups \cite{gattami2021reinforcement,zheng2020constrained,ding2021provably,bai2022achieving,bai2023achieving}, the focus of this monograph is on the average reward setup, thus providing a comprehensive study for the state of the art in the area.

\section{Chapter Organization}

In Chapter \ref{chpt:mb}, we consider a model-based approach for learning CMDPs with average reward and costs. We discuss posterior sampling-based and optimism-based algorithms. We demonstrate $\widetilde{O}(\sqrt{T})$ objective regret and zero constraint violation for both of them. The presented results follow the recent works of \cite{agarwal2022regret,agarwal2022concave}. 

In Chapter \ref{chpt:param:mf}, we consider a model-free approach for learning CMDP via general parametrization. General parameterization indexes the policies by finite-dimensional parameters (e.g., weights of neural networks) to accommodate large state spaces. The learning is manifested by updating these parameters using policy gradient (PG)-type algorithms. This chapter primarily follows the works of \cite{bai2023regret,bai2023learning} and presents an algorithm that achieves $\widetilde{O}({T}^{4/5})$ objective regret and constraint violation. Note that general parameterization subsumes the tabular setup. Moreover, the best-known regret bound achieved by any tabular model-free algorithm for average-reward CMDPs is $\Tilde{\mathcal{O}}(T^{5/6})$ \citep{wei2022provably} which is worse than the above result in terms of orders. Due to this reason, we do not present any algorithm specific to the tabular model-free setup.

In the previous chapters, we assumed the underlying CMDP to be ergodic. In Chapter \ref{chpt:beyerg}, we go beyond this assumption to consider weakly communicating CMDPs. Note that the class of weakly communicating CMDPs contains the set of ergodic CMDPs, and it is the largest class for which one can hope to establish theoretical guarantees for all instances \cite{bartlett2009regal,jaksch2010near}. This chapter presents the model-based approach of \cite{chen2022learning} and proves $\tilO{T^{2/3} }$ objective regret and constraint violation. We note that no known model-free algorithm currently exists that guarantees a sublinear regret and constraint violation for weakly communicating CMDPs. This leaves multiple open questions. 

\section{Some Useful Inequalities }

In this section, we provide some important inequalities for random variables, some of which will be used in this monograph. 

\begin{lemma}[Jensen's Inequality]
    Let $f: \mathbb{R} \rightarrow \mathbb{R}$ be a convex function, and let $X$ be a random variable. If $E[X]$ is finite, then
    \[
	f(E[X]) \leq E[f(X)].
    \]
\end{lemma}

\begin{lemma}[Cauchy-Schwarz Inequality \cite{dragomir2003survey}]
    For any vectors $\mathbf{u}$ and $\mathbf{v}$ in a real or complex inner product space, the Cauchy-Schwarz Inequality holds:
    \[
        \left|\langle \mathbf{u}, \mathbf{v} \rangle\right|^2 \leq \langle \mathbf{u}, \mathbf{u} \rangle \cdot \langle \mathbf{v}, \mathbf{v} \rangle.
    \]
\end{lemma}

\begin{lemma}{\citep[Lemma 30]{chen2021implicit}}
    \label{lem:var XY}
    For a random variable $X$ such that $|X|\leq C$ almost surely, we have: $\var[X^2]\leq 4C^2\var[X]$.
\end{lemma}

\begin{lemma}[Azuma-Hoeffding's Inequality \cite{serfling1974probability}]\label{lem:azuma_inequality}
    Let $X_1,\cdots,X_n$ be a Martingale difference sequence such that $|X_i|\leq c$ almost surely for all $i\in\{1,2,\cdots, n\}$, then,
    \begin{align}
        \mathbb{P}\left(\left|\sum_{i=1}^nX_i\right|\geq \epsilon\right)\leq 2\exp{\left(-\frac{\epsilon^2}{2nc^2}\right)}\label{eq:azuma_hoeffdings}
    \end{align}
\end{lemma}

\begin{lemma}[Any interval Azuma's inequality, \cite{chen2022learning}] 
    \label{lem:azuma}
    Let $\{X_i\}_{i=1}^{\infty}$ be a martingale difference sequence and $|X_i|\leq B$ almost surely. Then with probability at least $1-\delta$, for any $l, n$: $\abr{\sum_{i=l}^{l+n-1}X_i}\leq B\sqrt{2n\ln\frac{4(l+n-1)^3}{\delta}}$.
\end{lemma}
 
\begin{lemma}\citep[Lemma 38]{chen2021improved}
    \label{lem:freedman}
    Let $\{X_i\}_{i=1}^{\infty}$ be a martingale difference sequence adapted to the filtration $\{\calF_i\}_{i=0}^{\infty}$ and $|X_i|\leq B$ for some $B>0$. Then with probability at least $1-\delta$, for all $n\geq 1$ simultaneously,
    \begin{align*}
        \abr{\sum_{i=1}^nX_i}\leq 3\sqrt{\sum_{i=1}^n\E[X_i^2|\calF_{i-1}]\ln\frac{4B^2n^3}{\delta}} + 2B\ln\frac{4B^2n^3}{\delta}.
    \end{align*}
\end{lemma}

\begin{lemma}\citep{weissman2003inequalities}
    \label{lem:weiss}
    Let $p$ be an $m$-dimensional distribution and $\bar{p}$ be its empirical estimate obtained by averaging over $n$ samples. Then, $\norm{p-\bar{p}}_1\leq\sqrt{m\ln\frac{2}{\delta}/n}$ with probability at least $1-\delta$.
\end{lemma}

\begin{lemma}\citep[Theorem D.3]{cohen2020near}
    \label{lem:bernstein}
    Let $\{X_n\}_{n=1}^{\infty}$ be a sequence of i.i.d random variables with expectation $\mu$ and $X_n\in[0, B]$ almost surely. Then with probability at least $1-\delta$, for any $n\geq 1$:
    \begin{align*}
        &\abr{\sum_{i=1}^n(X_i-\mu)} \\ 
        &\leq \min\cbr{2\sqrt{B\mu n\ln\frac{2n}{\delta}} + B\ln\frac{2n}{\delta}, 2\sqrt{B\sum_{i=1}^nX_i\ln\frac{2n}{\delta}} + 7B\ln\frac{2n}{\delta}}.
    \end{align*}
\end{lemma}

\begin{lemma}{\citep[Lemma D.4]{cohen2020near} and \citep[Lemma E.2]{cohen2021minimax}}
    \label{lem:e2r}
    Let $\{X_i\}_{i=1}^{\infty}$ be a sequence of random variables w.r.t to the filtration $\{\calF_i\}_{i=0}^{\infty}$ and $X_i\in[0,B]$ almost surely. Then with probability at least $1-\delta$, for all $n\geq 1$ simultaneously:
    \begin{align*}
        \sum_{i=1}^n\E[X_i|\calF_{i-1}] &\leq 2\sum_{i=1}^n X_i + 4B\ln\frac{4n}{\delta},\\
        \sum_{i=1}^n X_i &\leq 2\sum_{i=1}^n\E[X_i|\calF_{i-1}] + 8B\ln\frac{4n}{\delta}.
    \end{align*}
\end{lemma}
\chapter{Model-Based RL}
\label{chpt:mb}

We introduce the model and key assumptions studied in this chapter in Section \ref{sec:mb_model}. Further, we describe optimism-based and posterior sampling-based algorithms in Section \ref{sec:mb:algos}. Their regret guarantees are provided in Section \ref{sec:mb:regopt} and \ref{sec:mb:regpos}, respectively. Evaluation results are discussed in Section \ref{sec:mb:eval} and Section \ref{sec:mb:notes} suggests some possible future directions.


\section{Overall Model and Assumptions}
\label{sec:mb_model}

We consider an infinite-horizon constrained Markov Decision Process $\mathcal{M} = (\mathcal{S}, \mathcal{A}, r,  c,  P, \rho)$ where $\mathcal{S}$ is finite set of $S$ states, $\mathcal{A}$ denotes a finite set of $A$ actions, $P:\mathcal{S}\times\mathcal{A}\times\mathcal{S}\to [0, 1]$ defines the transition probability kernel such that after executing action $a\in\mathcal{A}$ in state $s\in\mathcal{S}$, the system moves to state $s'\in\mathcal{S}$ with probability $P(s'|s,a)$. $r:\mathcal{S}\times\mathcal{A}\to[0,1]$ and $c:\mathcal{S}\times\mathcal{A}\to[-1,1]$ denotes the average reward obtained and average costs incurred in the state action pair $(s, a)\in \mathcal{S}\times\mathcal{A}$, and finally, $\rho$ is the distribution over the initial state.

The agent interacts with $\mathcal{M}$ in time-steps $t\in{1, 2, \cdots}$ for a total of $T$ time-steps. We note that $T$ is possibly unknown {and $s_1\sim\rho$}. At each time $t$, the agent observes state $s_t$ and plays action $a_t$. The agent selects an action on observing the state $s$  using a policy $\pi:\mathcal{S}\to\Delta(\mathcal{A})$, where $\Delta(\mathcal{A})$ is the probability simplex over the action space. On following a policy $\pi$, the long-term average reward and cost of the agent is given as:
\begin{align}
    \label{eq:chap_2_def_J_function}
	J_{g, \rho}^{\pi, P} &= {\lim}_{\tau\to\infty}\mathbf{E}_{\pi,P}\left[\frac{1}{\tau}\sum\nolimits_{t=1}^\tau g(s_t, a_t)\bigg| s_1\sim \rho\right]
\end{align}
where $\mathbb{E}_{\pi, P}[\cdot]$ denotes the expectation over the state and action trajectory generated from following $\pi$ on transitions $P$, starting from an initial distribution, $\rho$, and $g = r, c$ for the average reward and cost functions respectively. The above function can also be expressed as:
\begin{align}
    J_{g, \rho}^{\pi, P} = \sum\nolimits_{s, a}\nu_{ \rho}^{\pi, P}(s,a)g(s,a)\nonumber
\end{align}
where  $\nu_{\rho}^{\pi, P}$ is the state-action occupancy measure generated by following policy $\pi$ on MDP with transitions $P$, starting from an initial distribution, $\rho$ \citep{puterman2014markov}. Mathematically, $\nu_\rho^{\pi, P}(s, a) = d_{\rho}^{\pi, P}(s)\pi(a|s)$ where
\begin{align}
    \label{eq:chap_2_def_d_pi_P}
	d_{\rho}^{\pi, P}(s) &= {\lim}_{\tau\to\infty}\frac{1}{\tau}\sum\nolimits_{t=1}^\tau \mathbf{E}_{\pi, P}\left[\mathbf{1}(s_t = s)\bigg| s_1\sim \rho\right]
\end{align}
where $\mathbf{1}(\cdot)$ is the indicator function. For any policy $\pi$, define the induced state transition matrix $P^{\pi}\in \mathbb{R}^{S\times S}$ as: $P^{\pi}(s, s') = \sum_{a}\pi(a|s)P(s'|s, a)$, $\forall s, s'$. If for each policy, $\pi$, the CMDP $\mathcal{M}$ is irreducible and aperiodic, then it is defined to be ergodic.

Let $(P^\pi)^t$ be the $t$-step state transition probability matrix induced by $\pi$. Let $T_{s\to s'}^\pi$ denote the time taken by the Markov chain induced by $\pi$ to hit state $s'$ starting from state $s$. Further, define $T_M := \max_\pi \mathbb{E}[T^\pi_{s\to s'}]$ as the max reach time of the MDP $\mathcal{M}$. We now introduce our assumption.

\begin{assumption}
    \label{ergodic_assumption}
    The MDP, $\mathcal{M}$, is assumed to be ergodic, which implies: (a) $T_M < \infty$, (b) $d_\rho^{\pi, P}$ is independent of $\rho$, and (c) there exists $C > 0$ and $\zeta < 1$ such that $\|(P^{\pi})^t(s, \cdot) - d^{\pi, P}\|_{\mathrm{TV}} \le C\zeta^t$, $\forall s$, $\forall t\in\{1, 2, \cdots\}$.   
\end{assumption}

Ergodicity comes with many perks. For example, for ergodic CMDPs, the distribution $\nu_{\rho}^{\pi, P}$ is always well-defined and independent of $\rho$. This allows us to drop the dependence of $\rho$ from the notation of the occupancy measures and average reward and cost functions. Moreover, $\nu^{\pi, P}$ obeys the following stationary relation.
\begin{align*}
    \sum_{s, a} P(s'|s, a)\nu^{\pi, P}(s, a) = \sum_{a} \nu^{\pi, P}(s', a), ~\forall s'
\end{align*}

The agent aims to maximize $J_{r}^{\pi, P}$ while ensuring that $J_{c}^{\pi, P}$ does not exceed a certain threshold. Without loss of generality, we can express the above problem as the following optimization.
\begin{align}
    \label{eq:chap_2_main_constrained_optimization}
    &\max_\pi J_{r}^{\pi, P}~~~\text{s.t.}~~~J_{c}^{\pi, P} \leq 0
\end{align}
 
We shall denote a solution to the above problem as $\pi^*$.

\begin{remark} 
    We would like to emphasize the generality of the formulation stated above. Despite considering a single constraint, the algorithm and the results developed in this chapter can be easily extended to the case of multiple constraints. 
\end{remark}

\begin{assumption} \label{known_rewards}
    The rewards $r(s,a)$, the costs $c(s,a)$ are known to the agent.
\end{assumption}

\begin{assumption}
    \label{slaters_conditon}
    There exists a policy $\pi$, and a constant $\delta > \\ST_M\sqrt{(A\log T)/T} + (CSA\log T)/(T(1-\zeta))$ such that 
    \begin{align}
        J_{c}^{\pi, P} \le -\delta \label{eq:slackmb}
    \end{align}
\end{assumption}

This assumption is standard in the constrained RL literature \citep{efroni2020exploration,ding2021provably,ding2020natural,wei2021provably}. Here $\delta$ is referred to as the Slater's constant. \citet{ding2021provably} assumes that the Slater's constant $\delta$ is known. \citet{wei2021provably} assumes that the number of iterations of the algorithm is at least $ \Tilde{\Omega}(SAH/\delta)^5$ for episode length $H$. On the contrary, we simply assume the existence of $\delta$ and a lower bound on the value of $\delta$, which gets relaxed as the agent acquires more time to interact with the environment.

Any model-based online algorithm starting with no prior knowledge will need to obtain estimates of transition probabilities $P$ and obtain reward and cost values for each state-action pair. Initially, when the algorithm does not have a good estimate of the model, it accumulates regret and violates constraints as it does not know the optimal policy. We define reward regret $R(T)$ as the difference between the cumulative reward obtained by the algorithm and the expected cumulative reward that could have been obtained by running the optimal policy $\pi^*$ for $T$ steps. Formally, we have the following.
\begin{align}
    R(T)& \triangleq T J_{r}^{\pi^*, P} - \sum\nolimits_{t=1}^Tr(s_t, a_t) \label{eq:mbregret_rewards}
\end{align}
Additionally, we define constraint violation $C(T)$ as follows.
\begin{align}
    C(T)& \triangleq \left(\sum\nolimits_{t=1}^Tc(s_t, a_t)\right)_+\text{, where } (x)_+ = \max(0, x)\label{mbeqnviol}
\end{align}


\section{Algorithms for Model-Based RL}
\label{sec:mb:algos} 

Note that if the agent is aware of the true transition $P$, it can solve the following optimization to obtain the optimal occupancy measure.
\begin{align}
    \max_{\nu} \sum\nolimits_{s,a}r(s, a)\nu(s,a) \label{eq:optimization_equation1}
\end{align}
with the following set of constraints,
\begin{align}
    &\sum\nolimits_{s,a} \nu(s, a) = 1,\ \  \nu(s, a) \geq 0, ~\forall (s, a) \label{eq:valid_prob_constrant1}\\
    &\sum\nolimits_{a\in\mathcal{A}}\nu(s',a) = \sum\nolimits_{s,a}P(s'|s, a)\nu(s, a), ~\forall s'\label{eq:transition_constraint1}\\
    &\sum\nolimits_{s,a}\nu(s,a)c(s,a)\le 0 \label{eq:actual_constraint1}
\end{align}
Equation \eqref{eq:transition_constraint1} denotes the constraint on the transition structure for the underlying Markov Process. Equation \eqref{eq:valid_prob_constrant1} ensures that the solution is a valid probability distribution. Equation \eqref{eq:actual_constraint1} is the cost constraint. Using the optimal solution $\nu^*$, we can obtain the optimal policy as:
\begin{align}
    \pi^*(a|s) = \frac{\nu^*(s,a)}{\sum_{a'\in\mathcal{A}}\nu^*(s,a')}, ~ \forall (s, a)
\end{align}

The above problem is linear optimization and can be easily solved with the knowledge of true transition $P$. However, such knowledge is not easily available in most practical scenarios. In such cases, the concepts of optimism-based estimation \citep{jaksch2010near} and posterior sampling-based estimation \citep{osband2013more,agrawal2017optimistic} come to the rescue. In the following, we will explain these approaches in detail. 


\subsection{Optimism Based Estimation}

We note that the key idea for the optimistic policy approach \citep{jaksch2010near} is to apply upper confidence bounds in the estimation process. We utilize these bounds to create a confidence set of transition kernels, and the optimization problem also optimizes over this confidence set. The entire process is described in Algorithm \ref{algmblin}. 

\begin{algorithm}[H]
    \caption{Optimism Based Reinforcement Learning (C-UCRL)}
    \label{algmblin}
    \textbf{Input}: $\mathcal{S}$, $\mathcal{A}$, $r$, $c$, $K$
    \begin{algorithmic}[1] 
        \STATE Let $t=1$, $e = 1, \epsilon_e = K\sqrt{\frac{\ln t}{t}}$
	\FOR{$(s,a) \in \mathcal{S}\times\mathcal{A}$} 
            \STATE $N^{\mathrm{curr}}_e(s,a) = N_e(s,a) = 0$, 
            \STATE $N(s, a, s')=0, \widehat{P}_e(s'|s, a) = \frac{1}{S}, \forall s'\in\mathcal{S}$
	\ENDFOR
        \STATE Solve for policy $\pi_e$ using  \eqref{eq:opt_cons_optimal_policy}
	\FOR{$t \in \{1, 2, \cdots\}$}
            \STATE Observe $s_t$, and execute $a_t\sim\pi_e(\cdot|s_t)$
		\STATE Observe $s_{t+1}\sim P(\cdot|s_t, a_t)$ and $r(s_t,a_t)$, $c(s_t, a_t)$ 
		\STATE $N^{\mathrm{curr}}_e(s_t, a_t) = N^{\mathrm{curr}}_e(s_t, a_t) + 1$
            \STATE $N(s_t, a_t, s_{t+1}) = N(s_t, a_t, s_{t+1}) + 1$
            \IF {$N^{\mathrm{curr}}_e(s,a) = \max\{1, N_e(s,a)\}$ for any $s,a$}
                \STATE $e = e+1$, $\epsilon_e = K\sqrt{\frac{\ln t}{t}}$
                \FOR{$(s,a) \in \mathcal{S}\times\mathcal{A}$}
                    \STATE $N_{e}(s,a) = N_{e-1}(s,a) + N^{\mathrm{curr}}_{e-1}(s,a)$
                    \STATE $N^{\mathrm{curr}}_{e}(s,a) = 0$
                    \IF{$N_{e}(s, a)>0$}
                    \STATE $\widehat{P}_e(s'|s, a)=\frac{N(s, a, s')}{N_e(s, a)}$, $\forall s'\in\mathcal{S}$
                    \ENDIF
                \ENDFOR
                \STATE Solve for policy $\pi_e$ using Eq. \eqref{eq:opt_cons_optimal_policy}
            \ENDIF
	\ENDFOR
    \end{algorithmic}
\end{algorithm}

The algorithm proceeds in multiple epochs whose lengths are not fixed apriori but depend on the observations made. In epoch $e$, the algorithm maintains three key variables: $N^{\mathrm{curr}}_e(s,a)$, $N_e(s,a)$, and $N(s,a,s')$ for all state-action pairs $(s, a)$. $N^{\mathrm{curr}}_e(s,a)$ stores the number of times  $(s,a)$ are visited in epoch $e$ and $N_e(s,a)$ indicates the number of times $(s,a)$ are visited till the start of epoch $e$. Finally, $N(s, a, s')$ indicates the number of times the system transitions to state $s'$ after taking action $a$ in state $s$. If $N_e(s, a)>0$ at the beginning of epoch $e$, we estimate the transition probability $P(s'|s, a)$, $\forall s'$ as: $\widehat{P}_e(s'|s, a)=N(s, a, s')/N_e(s, a)$. Using the above variables, the agent obtains the optimal policy for the optimistic MDP by solving the following constrained optimization.
\begin{align}
    \max_{\nu, \tilde{P}_e}~ &\sum_{s, a} r(s, a)\nu(s, a)\\
    \text{subject to:}~& \sum\nolimits_{s,a} \nu(s, a) = 1,\ \  \nu(s, a) \geq 0, \\
    &\sum\nolimits_{a}\nu(s',a) = \sum\nolimits_{s,a}\Tilde{P}_e(s'|s, a)\nu(s, a)\label{eq:opt_eps_transition_constraint}, \\
    & \Tilde{P}_e(s'|s, a) > 0,~ \sum\nolimits_{s'}\Tilde{P}_e(s'|s, a)=1\label{eq:boundary}\\
    &\left\Vert\Tilde{P}_e(\cdot|s,a) - \widehat{P}_e(\cdot|s, a)\right\Vert_1 \le \sqrt{\frac{14S\log(2At_e)}{1\vee N_e(s,a)}}\label{eq:opt_transition_probability} \\
    & \sum\nolimits_{s,a}c(s,a)\nu(s,a)  \leq -\epsilon_e \label{eq:eps_cmdp_constraints}
\end{align}
where $s', s\in\mathcal{S}$, $a\in\mathcal{A}$, $x \vee y \triangleq \max\{x, y\}$, and $t_e$ is the initial time of the epoch $e$. Compared to \eqref{eq:optimization_equation1}$-$\eqref{eq:actual_constraint1}, a couple of changes have been made to formulate the above optimization. Firstly, the above optimization solves for both $\nu$ and $\tilde{P}_e$. Secondly, due to the lack of knowledge of the true transition, $P$, the constraint \eqref{eq:transition_constraint1} is replaced by \eqref{eq:opt_eps_transition_constraint}$-$\eqref{eq:opt_transition_probability}. We would like to clarify that \eqref{eq:opt_eps_transition_constraint} is similar to the stationarity condition \eqref{eq:transition_constraint1} while \eqref{eq:boundary} dictates the criteria of a valid transition kernel. Finally, \eqref{eq:opt_transition_probability} specifies the boundary of the confidence set. Thirdly, instead of \eqref{eq:actual_constraint1}, we consider  \eqref{eq:eps_cmdp_constraints} where $\epsilon_e$ is an epoch-dependent constant given by $\epsilon_e = K\sqrt{(\log t_e)/t_e}$ and $K$ is a configurable parameter. The intuition behind this modification is to choose policies conservatively to allow room to violate constraints.

A few remarks are in order. Firstly, since \eqref{eq:opt_transition_probability} defines an $L_1$ ball with a nonzero radius around the estimated transition, one can always find a transition function, $\tilde{P}_e$, in it that has all strictly positive elements. This, in turn, ensures that, for any policy, $\pi$, the Markov chain generated by $\pi$ and $\tilde{P}_e$ is irreducible and aperiodic, leading to the existence of an occupancy measure obeying the stationarity condition \eqref{eq:opt_eps_transition_constraint}. Secondly, due to Assumption \ref{slaters_conditon}, one can judiciously choose a small enough $\epsilon_e$ to ensure that the constraint \eqref{eq:eps_cmdp_constraints} is also satisfied. In summary, the feasible set of the above optimization can be ensured to be non-empty.

This optimization problem stated above is convex. Let $\nu_{e}$ be an optimal solution of the optimistic MDP for the epoch $e$. We obtain its corresponding policy as follows.
\begin{align}
    \pi_e(a|s) = \frac{\nu_e(s,a)}{\sum_{a'\in\mathcal{A}}\nu_e(s,a')},~ \forall (s, a) \label{eq:opt_cons_optimal_policy}
\end{align}

The agent executes the optimistic policy $\pi_e$ for epoch $e$. The system model improves as the agent interacts with the environment, allowing us to bound the regret. Note that the epochs change when the number of visits to any $(s, a)$ in the current epoch equals the aggregated number of prior visits to $(s, a)$. Therefore, the total number of occurrences of any state-action pair is at most doubled during an epoch. We note that the concept of the approach presented here is similar to that in \citep{jaksch2010near}, while the extended value iteration has been replaced by the solution of the optimization problem. 


\subsection{Posterior Sampling Based Estimation}

In this estimation approach, we estimate the transition probabilities by interacting with the environment. In order to solve the problem \eqref{eq:optimization_equation1}-\eqref{eq:actual_constraint1}, we need an estimate of $P$ such that the steady-state distribution exists for any policy $\pi$. In order to ascertain that, we note that when the priors of the transition probabilities $P(\cdot|s, a)$ are a Dirichlet distribution for each state and action pair, such a steady state distribution exists. Proposition \ref{steady_state_proposition} formalizes the result of the existence of a steady state distribution when the transition probability is sampled from a Dirichlet distribution.
\begin{proposition}\label{steady_state_proposition}
    For any MDP $\mathcal{M}$ with state space $\mathcal{S}$ and action space $\mathcal{A}$, let the transition probabilities ${P}$ be a Dirichlet distribution. Then, for any policy $\pi$, $\mathcal{M}$ will have an associated state-action distribution $\nu^{\pi}$ that obeys the following expression.
    \begin{align}
        \sum_{a}\nu^{\pi}(s', a) = \sum_{s, a} \nu^{\pi}(s, a)P(s'|s, a),~ \forall s'\in\mathcal{S}.
    \end{align}
\end{proposition}

\begin{proof}
    Since the transition probabilities $P(s'|s, a)$ follow the Dirichlet distribution, they are strictly positive. Further, as $\pi(\cdot|\cdot)$ is a probability distribution on actions conditioned on state, $\pi(a|s) \geq 0,\ \sum_a \pi(a|s) = 1$. So, there is a nonzero transition probability to reach from state $s\in{\mathcal{S}}$ to state $s'\in{\mathcal{S}}$.
	
    Now, note that all the entries of the transition probability matrix are strictly positive. Hence, the Markov Chain induced over the MDP $\mathcal{M}$ by any policy $\pi$ is irreducible (since it is possible to reach any state from any other state) and aperiodic (since it is possible to reach any state in a single time step from any other state). Combined, these two facts prove the existence of the steady-state distribution \citep{lawler2018introduction}.
\end{proof}

\begin{algorithm}[H]
    \caption{Model-Based Posterior Sampling Algorithm}
    \label{alg:model_based_algo_post}
    \textbf{Parameters}: $K$\\
    \textbf{Input}: $S$, $A$, $r$, $c$
    \begin{algorithmic}[1] 
        \STATE Let $t=1$, $e = 1, \epsilon_e = K\sqrt{\frac{\ln t}{t}}$
	\STATE $N_e^{\mathrm{curr}}(s,a) = 0, N_e(s,a) = 0,~\forall (s,a)$
        \STATE Obtain the steady-state distribution $\nu_e(s, a)$ as the solution of the optimization problem given by \eqref{eq:optimization_equation1}-\eqref{eq:transition_constraint1}, \eqref{eq:eps_cmdp_constraints} with $P_e(s'|s, a) \sim \mathrm{Dir}(N_e(s, a,\cdot)),~ \forall (s, a)$ used as the transition function.
        \STATE Solve for policy $\pi_e$ using Eq. \eqref{eq:opt_cons_optimal_policy}
	\FOR{$t \in \{1, 2, \cdots\}$}
            \STATE Observe $s_t$, and execute $a_t\sim\pi_e(\cdot|s_t)$
            \STATE Observe $s_{t+1}$, $r(s_t,a_t)$ and $c_i(s_t, a_t)\ \forall\ i\in[d]$
            \STATE $N^{\mathrm{curr}}_e(s_t, a_t) = N^{\mathrm{curr}}_e(s_t, a_t) + 1$
            \IF {$N^{\mathrm{curr}}_e(s,a) = \max\{1, N_e(s,a)\}$ for any $s,a$}
                \FOR{$(s,a) \in \mathcal{S}\times\mathcal{A}$}
                    \STATE $N_{e+1}(s,a) = N_e(s,a) + N^{\mathrm{curr}}_e(s,a)$
                    \STATE $e = e+1$, $N^{\mathrm{curr}}_e(s,a) = 0$
                \ENDFOR
                \STATE $\epsilon_e = K\sqrt{\frac{\ln t}{t}}$
                \STATE Obtain the steady-state distribution $\nu_e(s, a)$ by solving the optimization problem given by \eqref{eq:optimization_equation1}-\eqref{eq:transition_constraint1}, \eqref{eq:eps_cmdp_constraints} with $P_e(s'|s, a) \sim \mathrm{Dir}(N_e(s, a,\cdot)),~ \forall (s, a)$ used as the transition function.
                \STATE Solve for policy $\pi_e$ using Eq. \eqref{eq:opt_cons_optimal_policy}
            \ENDIF
        \ENDFOR
    \end{algorithmic}
\end{algorithm}

To complete the setup for our algorithm, we make a few assumptions, which are described below.

\begin{assumption}\label{dirichlet_prior_assumption}
    The transition probabilities $P(\cdot|s, a)$ of the Markov Decision Process have a Dirichlet prior for all state-action pairs $(s, a)$.
\end{assumption}

Since we assume that transition probabilities of the MDP $\mathcal{M}$ follow Dirichlet distributions, all policies on $\mathcal{M}$ have a steady-state distribution. Therefore, the key approach for the algorithm is to solve the optimization problem \eqref{eq:optimization_equation1}-\eqref{eq:actual_constraint1} using a sample of the transition probability with Dirichlet priors. We further notice that the algorithm proceeds in epochs similar to the case of optimistic reinforcement learning, and the policy is modified only at the start of the epoch and kept the same within the epoch. Note that the conservative constraint given by \eqref{eq:eps_cmdp_constraints} is still used to achieve zero constraint violations. The formal algorithm is given in Algorithm \ref{alg:model_based_algo_post}. 

The approach given here is an adaptation of the posterior sampling approach \citep{osband2013more,agrawal2017optimistic}, where the optimal policy for $P_e$ is solved using an optimization problem. 


\subsection{Optimistic Decision Making vs. Posterior Sampling Decision Making}

There is a well-known connection between posterior sampling and optimistic algorithms \citep{russo2014learning}.  The authors of  \citep{osband2017posterior} discussed that posterior sampling can be considered a stochastically optimistic algorithm. Before each epoch, a typical optimistic algorithm constructs a confidence set to represent the range of MDPs that are statistically plausible given the prior knowledge and observations. Then, a policy is selected by maximizing value simultaneously over policies and MDPs within this set. The agent then follows this policy over the epoch. It is interesting to contrast this approach against posterior sampling-based reinforcement learning, where instead of maximizing over a confidence set, it samples a single statistically plausible MDP and selects a policy to maximize the value for that MDP.

We would like to mention that optimistic approaches lack statistical efficiency as compared to posterior sampling-based approaches due to the sub-optimal construction of the confidence sets. Even if the state transition dynamics are entirely known, the $L_1$ ball of the uncertainty in the confidence set leads to an extremely poor evaluation performance of the regret \citep{osband2017posterior,agarwal2022concave}. Further, the computational complexity of the optimistic algorithm is higher since the search of the policy also has to consider an appropriate MDP in the confidence bounds. 


\section{Regret Analysis and Constraint Violation for Optimism Based  Approach}\label{sec:mb:regopt}

In this section, we study the regret and constraint violations for the proposed optimism-based algorithm. We note that the standard analysis for infinite horizon tabular MDPs of UCRL2 \citep{jaksch2010near} cannot be directly applied as the policy $\pi_e$ is possibly stochastic for every epoch. Another peculiar aspect of analyzing infinite horizon CMDPs is that the regret grows linearly with the number of epochs (or policy switches). This is because a new policy induces a new Markov chain, which takes time to converge to the stationary distribution. The analysis still bounds the regret by $\Tilde{O}(T_MS\sqrt{AT})$ as the number of epochs is upper bounded by $O(SA\log T)$. 
 
Before diving into the details, we introduce a few terms that are key to our analysis. The first one is the $Q$-function. We define $Q_\gamma^{\pi, P}$ as the long-term $\gamma$-discounted expected reward for taking action $a$ in state $s$ and following policy $\pi$ for the MDP with transition function $P$. We emphasize that the notion of discounted $Q$ function is introduced purely for analytical convenience, although the original objective is maximizing the average reward. It can be shown that $Q^{\pi, P}_{\gamma}$ satisfies the Bellman equation, as shown below.
\begin{align}
    &Q_\gamma^{\pi, P}(s,a) = r(s,a) + \gamma\sum\nolimits_{s'}P(s'|s,a)V_\gamma^{\pi,P}(s'), \nonumber \\
    &\text{where}~V_\gamma^{\pi,P}(s) = \mathbb{E}_{a\sim\pi(\cdot|s)}\left[Q_\gamma^{\pi, P}(s,a)\right]\nonumber
\end{align}
We define the Bellman error $B_\gamma^{\pi, \Tilde{P}}(s, a)$ for the $\gamma$-discounted MDP as:
\begin{align}
    B_\gamma^{\pi, \Tilde{P}}(s,a)&= Q_\gamma^{\pi, \Tilde{P}}(s,a) - r(s,a) - \gamma\sum\nolimits_{s'}P(s'|s,a)V_\gamma^{\pi, \Tilde{P}}(s,a) \label{eq:Bellman_error_definition}
\end{align}
We also define $B^{\pi, \Tilde{P}}(\cdot, \cdot)\triangleq \lim_{\gamma\rightarrow 1}B_\gamma^{\pi, \Tilde{P}}(\cdot, \cdot)$ as the Bellman error for the average reward setup. Note that the error arises if a transition different from the true one is followed. The terms defined for the discounted MDP are connected with the average reward setup. For example, \cite{puterman2014markov}[Corollary 8.2.5] showed the following $\forall g\in\{r, c\}$, $\forall s$.
\begin{align}
    \label{eq:chap2_dis_2_avg_1}
    J_g^{\pi, P} &= \lim_{\gamma\rightarrow 1} (1-\gamma) V_\gamma^{\pi, P}(s), \\
    \label{eq:chap2_dis_2_avg_2}
    d^{\pi, P}(\cdot) &= \lim_{\gamma\rightarrow 1} (1-\gamma)(I-\gamma P^{\pi})^{-1}(s, \cdot)
\end{align}
\cite{puterman2014markov} also showed that, for every average-reward MDP with transition kernel, $P$ and policy $\pi$, there exists a bias function $h\in \mathbb{R}^{\mathcal{S}}$ that satisfies the Bellman equation stated below $\forall s$.
\begin{align}
    \label{eq:chap2_avg_reward_Bellman}
    h(s) = r^{\pi}(s) - J_r^{\pi, P} + {\sum}_{s'} P^{\pi}(s, s') h(s')
\end{align}
where $r^{\pi}(s)\triangleq \sum_{a}r(s, a)\pi(a|s)$, $\forall s$. The span of $h$ is defined as $\mathrm{sp}(h)\triangleq \max_s h(s)-\min_s h(s)$. It can be further proven that, $\forall s, s'$,
\begin{align}
    \label{eq:chap2_diff_bias}
    h(s) - h(s') = \lim_{\gamma\rightarrow 1} \left(V_\gamma^{\pi, P}(s) - V_\gamma^{\pi, P}(s')\right)
\end{align}

After defining the key variables, we can now jump into bounding the objective regret $R(T)$. Intuitively, the algorithm incurs regret on three accounts. The first component originates from following the conservative policy, which is needed to limit the constraint violations. The second component emerges from solving for the optimal policy for the optimistic MDP. The third component arises from the system's stochastic behavior. Note that the constraints are violated because of both the imperfect MDP knowledge and the stochastic behavior. However, the conservative policy enforces the constraint violation to stay within some limits, as shown in the next result.
\begin{theorem}
    \label{thm:regret_bound}
    For all $T$ and $K = \Theta(T_MS\sqrt{A} + CSA/(1-\zeta))$, the regret $R(T)$ of C-UCRL algorithm is bounded as 
    \begin{align}
        R(T) = O\left(\frac{1}{\delta}T_MS\sqrt{AT\log AT} + \frac{CT_MS^2A\log T}{(1-\zeta)}\right)
    \end{align}
    and $C(T) = 0$, with probability at least $1-\frac{1}{T^{5/4}}$.
\end{theorem}
 
The following subsections will provide the formal proof of regret and constraint violation bounds stated in the above Theorem. 

 
\subsection{Objective Regret Bound}
 
We first provide a Lemma that bounds the gap between the long-term average expected reward obtained by executing the optimal solution to the original non-conservative problem for the true MDP and that obtained via the solution to the conservative optimization of the same MDP with $\epsilon_e$ constraint bound.
 
\begin{lemma}
    \label{lem:optimality_gap_objective}
    Let $\pi^*$ be the optimal non-conservative policy for the true MDP with transition function, $P$ and $\pi^*_e$ denote the optimal policy with $\epsilon_e$ conservative bound corresponding to the same MDP. If $\epsilon_e \le \delta$, we have the following inequality.
    \begin{align}
        &J_r^{\pi^*, P} - J_r^{\pi_e^*, P} \le 2\epsilon_e/\delta
    \end{align}
\end{lemma}
\begin{proof}[Proof Sketch:] We first construct a policy for which the steady-state distribution is the weighted average of two steady-state distributions. The first distribution is for the optimal non-conservative policy corresponding to the true MDP $\mathcal{M}$. The second one is for another policy satisfying Assumption \ref{slaters_conditon}. We show that this constructed policy satisfies the $\epsilon_e$ conservative constraint. Finally, utilizing the optimality of $\pi^*_e$ for the conservative problem with $\epsilon_e$ constrained bound and the boundedness of the reward function, we arrive at the conclusion.  
\end{proof}
 	
\begin{proof}[Detailed Proof:] Note that $\nu^{\pi^*, P}$, the stationary distribution corresponding to the optimal non-conservative policy, $\pi^*$, satisfies
    \begin{align}
        \sum_{s,a}\nu^{\pi^*, P}(s,a)c(s,a)\leq 0\label{eq:true_constraint_bound_lem_1}
    \end{align}
    Further, from Assumption \ref{slaters_conditon}, we have a feasible policy $\pi$ for which
    \begin{align}
        \sum_{s,a}\nu^{\pi, P}(s,a)c(s,a)\leq - \delta\label{eq:tight_constraint_bound_lem_2}
    \end{align}
    We construct a stationary distribution $\nu^P$ and obtain its corresponding policy $\pi_e'$ using the following equations $\forall (s, a)$.
    \begin{align}
        \nu^P(s,a) &= \left(1-\frac{\epsilon_e}{\delta}\right)\nu^{\pi^*, P}(s,a) + \frac{\epsilon_e}{\delta}\nu^{\pi, P}(s,a)\label{eq:avg_stationary_dist}\\
        \pi_{e}'(a|s) &= \nu^P(s,a)/\left(\sum_{s,a'}\nu^P(s, a')\right)
    \end{align}
    For this new policy, we observe that
    \begin{align}
        \begin{split}
            &\sum_{s,a}\nu^P(s,a)c(s,a) = \sum_{s,a}\left(1-\frac{\epsilon_e}{\delta}\right)\nu^{\pi^*, P} + \frac{\epsilon_e}{\delta}\nu^{\pi, P}(s,a)c(s,a)\\
            &=\left(1-\frac{\epsilon_e}{\delta}\right)\sum_{s,a}\nu^{\pi^*, P}(s,a)c(s,a)+ \frac{\epsilon_e}{\delta}\sum_{s,a}\nu^{\pi, P}(s,a)c(s,a)\\
            &\overset{(a)}{\leq} \frac{\epsilon_e}{\delta}\left( - \delta\right) = - \epsilon_e
        \end{split}
        \label{eq:constraint_bounds_lem_1}
    \end{align}
    where $(a)$ follows from \eqref{eq:true_constraint_bound_lem_1}--\eqref{eq:tight_constraint_bound_lem_2}.
 	
    As a consequence of \eqref{eq:constraint_bounds_lem_1}, one can claim that $\pi_e'$  constructed in  \eqref{eq:avg_stationary_dist} satisfies the $\epsilon_e$ conservative constraint. Further, it is given that $\pi_e^*$ is the optimal solution corresponding to the $\epsilon_e$ conservative constrained optimization problem. Hence, we have
    \begin{align}
        \begin{split}
            &\sum_{s,a}\nu^{\pi^*, P}(s,a)r(s,a)- \sum_{s,a}\nu^{\pi_e^*, P}(s,a)r(s,a)\\
            &\leq \sum_{s,a}\nu^P_{\pi^*}(s,a)r(s,a) -  \sum_{s,a}\nu^P(s,a)r(s,a)\\
            &\leq \left|\sum_{s,a}\left(\nu^{\pi^*, P}(s,a) - \nu^P(s,a)\right)r(s,a)\right|\\
            &\leq \left|\sum_{s,a}\left(\nu^{\pi^*, P}(s,a) - \left(1-\frac{\epsilon_e}{\delta}\right)\nu^{\pi^*, P}(s,a) - \frac{\epsilon_e}{\delta}\nu^{\pi, P}(s,a)\right)r(s,a)\right|\\
            &\leq \frac{\epsilon_e}{\delta}\left|\sum_{s,a}\left(\nu^{\pi^*, P}(s,a) - \nu^{\pi, P}(s,a)\right)r(s,a)\right|\\
            &\leq \frac{\epsilon_e}{\delta}\left|\sum_{s,a}\nu^{\pi^*, P}(s, a) r(s,a)\right| + \frac{\epsilon_e}{\delta}\left|\sum_{s,a} \nu^{\pi, P}(s,a)r(s,a)\right|\overset{(a)}{\leq} \frac{2\epsilon_e}{\delta}
        \end{split}
        \label{eq:r_bounded_by_1}
    \end{align}
    where $(a)$ follows from the fact that $r(s,a) \leq 1$, $\forall (s,a)$.	
\end{proof}
 
Lemma \ref{lem:optimality_gap_objective} and our construction of $\epsilon_e$ sequence is crucial to restrict the growth of the regret due to conservative policies by $\Tilde{O}(T_MS\sqrt{AT})$. We will now decompose the regret into multiple components and analyze each part individually.
 
\subsubsection{Regret breakdown}
Let $E$ denote the number of epochs observed in duration $T$ and $T_e, t_e$ denote the length and starting time of the epoch $e$.  Moreover, following our earlier notations, let $\pi_e^*, \pi_e$ be the optimal $\epsilon_e$-conservative policies corresponding to MDPs with true transition, $P$ and optimistic transition, $\tilde{P}_e$, respectively. We have the following decomposition of the regret. 
\begin{align}
    \begin{split}
	&R(T) = TJ_r^{\pi^*, P} - \sum_{t=1}^Tr_t(s_t, a_t)\\
        & = TJ_r^{\pi^*, P} - \sum_{e=1}^E T_e J_r^{\pi_e^*, P} + \sum_{e=1}^E T_e J_r^{\pi_e^*, P} -  \sum_{t=1}^Tr_t(s_t, a_t)\\
        &= \sum_{e=1}^E T_e \left(J_r^{\pi^*, P} - J_r^{\pi_e^*, P}\right) + \sum_{e=1}^E T_e J_r^{\pi_e^*, P} -  \sum_{t=1}^Tr_t(s_t, a_t)\\
        &\overset{(a)}{\leq} \sum_{e=1}^E T_e \left(J_r^{\pi^*, P} - J_r^{\pi_e^*, P}\right) + \sum_{e=1}^E T_e J_r^{\pi_e, \Tilde{P}_e} - \sum_{t=1}^T r_t(s_t, a_t)\\
        &\le \sum_{e=1}^E T_e \left(J_r^{\pi^*, P} - J_r^{\pi_e^*, P}\right) +  \left|\sum_{e=1}^E T_e J_r^{\pi_e, \Tilde{P}_e} - \sum_{t=1}^T r_t(s_t, a_t)\right|\\
        &= \sum_{e=1}^E T_e \left(J_r^{\pi^*, P} - J_r^{\pi_e^*, P}\right) \\
        &\hspace{1cm}+  \left|\sum_{e=1}^E \sum_{t = t_e}^{t_{e+1}-1} \left(J_r^{\pi_e, \Tilde{P}_e} - J_r^{\pi_e, P} + J_r^{\pi_e, P} - r_t(s_t, a_t)\right)\right|\\
        &\le \sum_{e=1}^E T_e \left(J_r^{\pi^*, P} - J_r^{\pi_e^*, P}\right) + \left|\sum_{e=1}^E \sum_{t = t_e}^{t_{e+1}-1} \left(J_r^{\pi_e, \Tilde{P}_e} - J_r^{\pi_e, P}\right)\right|\\
        &\hspace{1cm}+ \left|\sum_{e=1}^E \sum_{t = t_e}^{t_{e+1}-1} \left(J_r^{\pi_e, P} - r_t(s_t, a_t)\right)\right|\\
        &= R_1(T) + R_2(T) + R_3(T)
    \end{split}
    \label{eq:define_broken_regret}
\end{align}
where $(a)$ is a consequence of the fact that the policy $\pi_e$ is optimal for the optimistic CMDP with transition $\tilde{P}_e$ corresponding to the constrained bound $\epsilon_e$. The first term in \eqref{eq:define_broken_regret} is as follows.
\begin{align}
    R_1(T) &= \sum_{e=1}^E T_e \left(J_r^{\pi^*, P} - J_r^{\pi_e^*, P}\right)
\end{align}

$R_1(T)$ denotes the regret incurred from not playing the optimal non-conservative policy $\pi^*$ for the true MDP with transition, $P$ but rather the optimal conservative policy $\pi_e^*$ for the same MDP with constraint $\epsilon_e$ at epoch $e$. The second component is given as follows.
\begin{align}
    R_2(T) &= \left|\sum_{e=1}^E \sum_{t = t_e}^{t_{e+1}-1} \left(J_r^{\pi_e, \Tilde{P}_e} - J_r^{\pi_e, P}\right)\right|
\end{align}

$R_2(T)$ is the difference between the expected rewards generated from playing $\pi_e$ on the optimistic MDP instead of the true MDP.  
\begin{align}
    R_3(T) &= \left|\sum_{e=1}^E \sum_{t = t_e}^{t_{e+1}-1} \left(J_r^{\pi_e, P} - r_t(s_t, a_t)\right)\right|
\end{align}
$R_3(T)$ denotes the gap between the observed and expected rewards obtained from playing the policy $\pi_e$ on the true MDP.

\subsubsection{Bounding $R_1(T)$}
Bounding $R_1(T)$ uses Lemma \ref{lem:optimality_gap_objective}. We have the following set of equations:
\begin{align}
    \begin{split}
        R_1(T) &= \sum_{e=1}^E\sum_{t = t_e}^{t_{e+1}-1}\left(J_r^{\pi^*, P} - J_r^{\pi_e^*, P}\right)\\
        &\overset{}{\le} \sum_{e=1}^E \sum_{t = t_e}^{t_{e+1}-1}\frac{2\epsilon_e}{\delta} \\
        &\overset{(a)}{=} \frac{2K}{\delta} \sum_{t = 1}^{T}\sqrt{\frac{\log T}{t}}\\
        &=\frac{2K\log T}{\delta}\left(1 + \sum_{t = 2}^{T}\sqrt{\frac{1}{t}}\right)\\
        &\le\frac{2K\log T}{\delta}\left(1 +  \int_{t = 1}^{T}\sqrt{\frac{1}{t}}\mathrm{d}t\right) \le \frac{2K\log T}{\delta}(2\sqrt{T})
    \end{split}
\end{align}
where $(a)$ follows from the definition of $\epsilon_e$ and the fact that $\log t \le \log T$ for all $t \le T$.

\subsubsection{Bounding $R_2(T)$} \label{app:bounding_imperferct_model_regret}

As stated earlier, the term $R_2(T)$ captures the difference between the expected rewards arising from running the policy $\pi_e$ on the optimistic and true MDPs. Here, the idea of the Bellman error introduced earlier becomes useful. Formally, we have the following lemma.

\begin{lemma}\label{lem:bound_average_by_bellman}
    The difference between $J_r^{\pi_e, \Tilde{P}_e}$, i.e., the long-term average reward generated via running the optimistic policy $\pi_e$ on the optimistic MDP, and $J_r^{\pi_e, P}$, i.e., the same generated via running the optimistic policy $\pi_e$ on the true MDP, can be expressed as the long-term average Bellman error. Mathematically,
    \begin{align}
        J_r^{\pi_e, \Tilde{P}_e} - J_r^{\pi_e, P} = \sum_{s,a}\nu^{\pi_e, P}(s, a) B^{\pi_e, \Tilde{P}_e}(s, a)
    \end{align}
\end{lemma}
\begin{proof} Note that for all $s\in\mathcal{S}$, we have:
    \begin{align}
        \begin{split}
            &V_\gamma^{\pi_e, \Tilde{P}_e}(s) = \mathbb{E}_{a\sim\pi_e}\left[Q_\gamma^{\pi_e, \Tilde{P}_e}(s,a)\right]\\
            &\overset{(a)}{=} \mathbb{E}_{a\sim\pi_e}\left[B_\gamma^{\pi_e, \Tilde{P}_e}(s,a) + r(s,a) + \gamma\sum_{s'\in\mathcal{S}}P(s'|s,a)V_\gamma^{\pi_e, \Tilde{P}_e}(s')\right]
        \end{split}  
        \label{eq:optimistic_MDP_lambda}
    \end{align}
    where (a) follows from the definition of the Bellman error. Similarly, for the true MDP, we have,
    \begin{align}
        \begin{split}
            V_\gamma^{\pi_e, P}(s) &= \mathbb{E}_{a\sim\pi_e}\left[Q_\gamma^{\pi_e, P}(s,a)\right]\\
            &= \mathbb{E}_{a\sim\pi_e}\left[r(s,a)+ \gamma\sum_{s'\in\mathcal{S}}P(s'|s,a)V_\gamma^{\pi_e, P}(s')\right] 
        \end{split}
        \label{eq:true_MDP_lambda}
    \end{align}
    The above equation essentially applies the Bellman equation for discounted reward MDPs. Subtracting \eqref{eq:true_MDP_lambda} from \eqref{eq:optimistic_MDP_lambda}, we get:
    \begin{align}
        \begin{split}
            &V_\gamma^{\pi_e,\Tilde{P}_e}(s) - V_\gamma^{\pi_e,P}(s)\\ &= \mathbb{E}_{a\sim\pi_e}\left[B_\gamma^{\pi_e, \Tilde{P}_e}(s,a) + \gamma\sum_{s'\in\mathcal{S}}P(s'|s,a)\left(V_\gamma^{\pi_e, \Tilde{P}_e} - V_\gamma^{\pi_e, \Tilde{P}_e}\right)(s')\right]\\
            &= \mathbb{E}_{a\sim\pi_e}\left[B_\gamma^{\pi_e, \Tilde{P}_e}(s,a)\right] + \gamma\sum_{s'\in\mathcal{S}}P^{\pi_e}(s, s')\left(V_\gamma^{\pi_e, \Tilde{P}_e} - V_\gamma^{\pi_e, \Tilde{P}_e}\right)(s')
        \end{split}
    \end{align}
    Let $\bar{B}_\gamma^{\pi_e, \tilde{P}_e}(\cdot)\triangleq \sum_{a}\pi_e(a|s)B_\gamma^{\pi_e, \tilde{P}_e}(\cdot, a)$. Using the vector format, we get,
    \begin{align}
        {V}_\gamma^{\pi_e, \Tilde{P}_e} - {V}_\gamma^{\pi_e, P} &= \left(I-\gamma P^{\pi_e}\right)^{-1}\bar{B}_\gamma^{\pi_e, \Tilde{P}_e}
    \end{align}
    Applying \eqref{eq:chap2_dis_2_avg_1}, \eqref{eq:chap2_dis_2_avg_2}, we deduce the following.
    \begin{align}
        \begin{split}
            &J_r^{\pi_e, \Tilde{P}_e} - J_r^{\pi_e, P} = \lim_{\gamma\to1}(1-\gamma)\left({V}_\gamma^{\pi_e,\Tilde{P}_e}(s) - {V}_\gamma^{\pi_e,P}(s)\right)\\
            &= \sum_{s'}\left[\lim_{\gamma\to1}(1-\gamma)\left(I-\gamma P^{\pi_e}\right)^{-1}\right](s, s')\lim_{\gamma\rightarrow 1}\bar{B}_\gamma^{\pi_e, \Tilde{P}_e}(s')\\
            &= \sum_{s',a}d^{\pi_e, P}(s')\pi_e(a|s') B^{\pi_e, \Tilde{P}_e}(s',a)
        \end{split}
    \end{align}
    Using $\nu^{\pi_e, P}(s, a)=d^{\pi_e, P}(s)\pi_e(a|s)$, $\forall (s, a)$, we complete the proof.
\end{proof}

\begin{remark}
    Note that the Bellman error is unrelated to the Advantage function and policy improvement lemma \cite{langford2002approximately}. The policy improvement lemma relates the performance of two policies on the same MDP, whereas we bound the performance of one policy on two different MDPs in Lemma \ref{lem:bound_average_by_bellman}.
\end{remark}

The lemma described below bounds the Bellman error, which in turn, bounds the gap between $J_r^{\pi_e, \Tilde{P}_e}$ and $J_r^{\pi_e, P}$. 

\begin{lemma}\label{lem:bound_bellman_s_a}
    The Bellman error $B^{\pi_e, \Tilde{P}_e}(s, a)$ for state-action pair $(s, a)$ in epoch $e$ is upper bounded as follows with probability at least $1- t_e^{-6}$,
    \begin{align}
        B^{\pi_e, \Tilde{P}_e}(s,a) \le \min\left\{2,\sqrt{\frac{14S\log(2AT)}{1\vee N_e(s,a)}}\right\}\mathrm{sp}(\Tilde{h})
    \end{align}
where $\Tilde{h}$ denotes a bias function associated with an MDP with transition kernel $\Tilde{P}_e$, and a policy $\pi_e$.
\end{lemma}
\begin{proof}[Proof Sketch:]
    Bellman error bounds the impact of the difference in the value obtained due to the difference in transition probability to the immediate next state. We bound the difference in transition probability between optimistic and true MDPs using the result from \cite{weissman2003inequalities}. This approach gives the desired result. 
\end{proof}
\begin{proof}[Detailed Proof:] The definition of the Bellman error gives
    \begin{align}
        \begin{split}
            &B^{\pi_e, \Tilde{P}_e}(s,a) = \lim_{\gamma\to1}\left(Q_\gamma^{\pi_e, \tilde{P}_e}(s,a) - \left(r(s,a) +\gamma \sum_{s'}P(s'|s,a)V_\gamma^{\pi_e, \Tilde{P}_e} \right)\right)\\
            &\overset{(a)}{=}\lim_{\gamma\to1}\left(\left(r(s,a) + \gamma\sum_{s'}\Tilde{P}_e(s'|s,a)V_\gamma^{\pi_e,\Tilde{P}_e}(s')\right)\right.\\
            &\hspace{1.5cm}- \left.\left(r(s,a) +\gamma \sum_{s'}P(s'|s,a)V_\gamma^{\pi_e, \Tilde{P}_e}(s') \right)\right)\\
            &\overset{}{=} \lim_{\gamma\to1}\gamma\sum_{s'}\left(\Tilde{P}_e(s'|s,a) - P(s'|s,a)\right)V_\gamma^{\pi_e,\Tilde{P}_e}(s')\\
            &= \lim_{\gamma\to1}\gamma\Bigg(\sum_{s'}\left(\Tilde{P}_e(s'|s,a) - P(s'|s,a)\right)V_\gamma^{\pi_e,\Tilde{P}_e}(s')\\
            &~~- \sum_{s'}\Tilde{P}_e(s'|s,a)V_\gamma^{\pi_e,\Tilde{P}_e}(s) + \sum_{s'} P(s'|s,a)V_\gamma^{\pi_e,\Tilde{P}_e}(s)\Bigg)\\
            &= \lim_{\gamma\to1}\gamma\left(\sum_{s'}\left(\Tilde{P}_e(s'|s,a) - P(s'|s,a)\right)\left(V_\gamma^{\pi_e,\Tilde{P}_e}(s') - V_\gamma^{\pi_e,\Tilde{P}_e}(s)\right)\right)\\
            &\overset{}{=} \sum_{s'}\left(\Tilde{P}_e(s'|s,a) - P(s'|s,a)\right)\lim_{\gamma\to1}\gamma\left(V_\gamma^{\pi_e,\Tilde{P}_e}(s') - V_\gamma^{\pi_e,\Tilde{P}_e}(s)\right)\\
            &\overset{}{=} \sum_{s'}\left(\Tilde{P}_e(s'|s,a) - P(s'|s,a)\right)\left(\Tilde{h}(s')-\tilde{h}(s)\right)\\
            &\overset{(b)}{\leq} \Big\|\left(\Tilde{P}_e(\cdot|s,a) - P(\cdot|s,a)\right)\Big\|_1\mathrm{sp}(\Tilde{h}) \overset{(c)}{\le} \sqrt{\frac{14S\log(2At_e)}{1\vee N_e(s,a)}}\mathrm{sp}(\Tilde{h})
        \end{split}
    \end{align}
    where $(a)$ is an application of the Bellman equation and $(b)$ follows from H\"{o}lder's inequality. In $(c)$, the $L_1$ difference of the transition functions can be trivially bounded by $2$. The bound demonstrated above holds with probability at least $1-t_e^{-6}$ following the result by \cite{weissman2003inequalities}. 
\end{proof}

The next lemma bounds the average Bellman error of an epoch using its realizations at state-action pairs visited in an epoch. 

\begin{lemma}\label{lem:bound_expected_bellman}
    With probability at least $1-1/T^6$, the cumulative expected Bellman error is bounded as:
    \begin{align}
        \begin{split}
            \sum_{e=1}^E(t_{e+1}-t_e)&\mathbb{E}_{\pi_e,P}\left[B^{\pi_e, \Tilde{P}_e}(s,a)\right] \\
            &\le \sum_{e=1}^E\sum_{t=t_e}^{t_{e+1}-1}B^{\pi_e, \Tilde{P}_e}(s_t,a_t) + 4T_M\sqrt{7T\log(T)}
        \end{split}
    \end{align}
\end{lemma}
\begin{proof}
    Let $\mathcal{F}_t = \{s_1, a_1, \cdots, s_t, a_t\}$ be the filtration generated by the running the algorithm for $t$ time-steps. We now use Assumption \ref{ergodic_assumption} to obtain the following.
    \begin{align}
        \begin{split}
            &\mathbb{E}_{(s,a)\sim\pi_e, P}[B^{\pi_e, \Tilde{P}_e}(s,a)] \\
            &=\mathbb{E}_{(s_t,a_t)\sim\pi_e, P}[B^{\pi_e, \Tilde{P}_e}(s_t,a_t)|\mathcal{F}_{t_e-1}] \\
            &~~+ \left(\mathbb{E}_{(s,a)\sim\pi_e, P}[B^{\pi_e, \Tilde{P}_e}(s,a)]- \mathbb{E}_{(s_t,a_t)\sim\pi_e, P}[B^{\pi_e, \Tilde{P}_e}(s_t,a_t)|\mathcal{F}_{t_e-1}]\right)\\
            &\overset{(a)}{\leq}\mathbb{E}_{(s_t,a_t)\sim\pi_e, P}[B^{\pi_e, \Tilde{P}_e}(s_t,a_t)|\mathcal{F}_{t_e-1}] \\
            &~~+ 2~\mathrm{sp}(\Tilde{h})\sum_{s,a}\left|\pi_e(a|s)d^{\pi_e, P}(s) - \pi_e(a|s)(P^{\pi_e})^{t-t_e+1}(s_{t_e-1}, s)\right|\\
            &\le\mathbb{E}_{(s_t,a_t)\sim\pi_e, P}[B^{\pi_e, \Tilde{P}_e}(s_t,a_t)|\mathcal{F}_{t_e-1}] \\
            &~~+ 2~\mathrm{sp}(\Tilde{h})\sum_{s,a}\pi_e(a|s)\left|d^{\pi_e, P}(s) - (P^{\pi_e})^{t-t_e+1}(s_{t_e-1}, s)\right|\\
            &\le\mathbb{E}_{(s_t,a_t)\sim\pi_e, P}[B^{\pi_e, \Tilde{P}_e}(s_t,a_t)|\mathcal{F}_{t_e-1}] \\
            &~~+ 2~\mathrm{sp}(\Tilde{h})\sum_{s,a}\pi(a|s)\|d^{\pi_e, P} - (P^{\pi_e})^{t-t_e+1}(s_{t_e-1}, \cdot)\|_{\mathrm{TV}}\\
            &\le\mathbb{E}_{(s_t,a_t)\sim\pi_e, P}[B^{\pi_e, \Tilde{P}_e}(s_t,a_t)|\mathcal{F}_{t_e-1}] + 2~\mathrm{sp}(\Tilde{h})\sum_{s,a}\pi_e(a|s)C\zeta^{t-t_e}\\
            &\overset{(b)}{=}\mathbb{E}_{(s_t,a_t)\sim\pi_e, P}[B^{\pi_e, \Tilde{P}_e}(s_t,a_t)|\mathcal{F}_{t_e-1}] + 2~\mathrm{sp}(\Tilde{h})CS\zeta^{t-t_e}
        \end{split}
    \end{align}
    where Equation $(a)$ follows from Assumption \ref{ergodic_assumption} for running policy $\pi_e$ starting from state $s_{t_e-1}$ for $t-t_e+1$ steps and from Lemma \ref{lem:bound_bellman_s_a}. Equation $(b)$ follows from bounding the total-variation distance for all states and from the fact that $\sum_a\pi_e(a|s) = 1$. 
	
    Using this, and the fact that $\mathbb{E}_{\pi_e,P}\left[B^{\pi_e, \Tilde{P}_e}(s_t,a_t)|\mathcal{F}_{t_e-1}\right] - B^{\pi_e, \Tilde{P}_e}(s_t,a_t)$ forms a Martingale difference sequence conditioned on filtration $\mathcal{F}_{t-1}$  with $|\mathbb{E}_{\pi_e,P}\left[B^{\pi_e, \Tilde{P}_e}(s_t,a_t)|\mathcal{F}_{t-1}\right] - B^{\pi_e, \Tilde{P}_e}(s_t,a_t)| \le 4~\mathrm{sp}(\Tilde{h})$, we can use the Azuma-Hoeffding inequality to bound the summation as
    \begin{align}
        \begin{split}
            &\sum_{e=1}^E(t_{e+1}-t_e)\mathbb{E}_{\pi_e,P}\left[B^{\pi_e, \Tilde{P}_e}(s,a)\right] \\
            &= \sum_{e=1}^E\Big((t_{e+1}-t_e)\mathbb{E}_{\pi_e,P}\left[B^{\pi_e, \Tilde{P}_e}(s_t,a_t)|\mathcal{F}_{t_e-1}\right] + \sum_{t = t_e}^{t_{e+1}-1}2CS\mathrm{sp}(\Tilde{h})\zeta^{t-t_e} \Big)\\
            &\le \sum_{e=1}^E\left(\sum_{t=t_e}^{t_{e+1}-1}\mathbb{E}_{\pi_e,P}\left[B^{\pi_e, \Tilde{P}_e}(s_t,a_t)|\mathcal{F}_{t_e-1}\right] + \frac{2CS\mathrm{sp}(\Tilde{h})}{1-\zeta} \right)\\
            &\le \sum_{e=1}^E\sum_{t=t_e}^{t_{e+1}-1}B^{\pi_e, \Tilde{P}_e}(s_t,a_t) + 4\mathrm{sp}(\Tilde{h})\sqrt{7T\log 2T} + \frac{2CES\mathrm{sp}(\Tilde{h})}{1-\zeta}
        \end{split}
        \label{eq:azuma_hoeffding_bellman}
    \end{align}
    where the last inequality follows from the Azuma-Hoefdding inequality with probability at least $1-T^{-6}$.
\end{proof}

\subsubsection{Bounding the term $\mathrm{sp}(\Tilde{h})$}

\begin{lemma}
    For a MDP with rewards $r(s,a)$ and transition probabilities $\Tilde{P}_e$, using policy $\pi_e$, the difference of bias of any two states $s$, and $\bar{s}$ is bounded as $\Tilde{h}(s) - \Tilde{h}(\bar{s}) \leq T_M,~ \forall s, \bar{s}\in\mathcal{S}$.
\end{lemma}
\begin{proof}
    Note that $J_r^{\pi_e, \Tilde{P}_e} \ge J_r^{\pi_e, P'}$ for all $P'$ in the confidence set. Consider the following Bellman equation.
    \begin{align*}
        \Tilde{h}(s) &= r^{\pi_e}(s) - J_r^{\pi_e, \Tilde{P}_e} + \sum_{s'} P^{\pi_e,e}(s, s') \Tilde{h}(s')~= T\Tilde{h}(s)
    \end{align*}
    where $r^{\pi_e}(s) = \sum_{a}\pi_e(a|s)r(s,a)$ and $P^{\pi_e,e}(s, s') = \sum_{a}\pi_e(a|s)\Tilde{P}_e(s'|s,a)$. Consider arbitrary two states $s, \bar{s}\in \mathcal{S}$. Let $\tau = \min\{t\geq 1: s_t = \bar{s}, s_1 = s\}$ be a random variable. With $P^{\pi_e}(s, s')$ $= \sum_{a}\pi_e(a|s)P(s'|s,a)$, define another operator as follows.
    \begin{align*}
        \bar{T}\tilde{h}(s)=&\left(\min_{s,a}r(s,a) - J_r^{\pi_e, \Tilde{P}_e} + \sum_{s'} P^{\pi_e}(s, s') \tilde{h}(s')\right)\mathbf{1}(s\neq \bar{s}) \\
        & \hspace{2.5cm}+ \Tilde{h}(\bar{s})\mathbf{1}(s=\bar{s}).
    \end{align*}
    Note that $\bar{T}\Tilde{h}(s) \le T\Tilde{h}(s) = \Tilde{h}(s)$ for all $s$ since $\Tilde{P}_e$ maximizes the reward $r$ over all the transition probabilities in the confidence set including the true transition probability $P$. Further, for any two vectors $u, v\in\mathbb{R}^S$ with $u(s) \ge v(s), \forall s$, we have $\bar{T}u \ge \bar{T}v$. Hence, we have $\bar{T}^n\Tilde{h}(s) \le \Tilde{h}(s)$ for all $s$. Hence, we have	
    \begin{align*}
        \Tilde{h}(s) &\ge \bar{T}^n(s) = \mathbb{E}\big[-(J_r^{\pi_e,\Tilde{P}_e} - \min_{s,a}r(s,a))(n\wedge\tau) + \Tilde{h}(s_{n\wedge\tau})\big]
    \end{align*}
    Taking limit as $n\to \infty$, we have $\Tilde{h}(s) \ge \Tilde{h}(s') - T_M$, thus completing the proof. 
\end{proof}

We are now ready to bound $R_2(T)$ using Lemma \ref{lem:bound_average_by_bellman}, Lemma \ref{lem:bound_bellman_s_a}, and Lemma \ref{lem:bound_expected_bellman}. We have the following set of equations:

\begin{align}
    \begin{split}
        &R_2(T) = \left|\sum_{e=1}^E\sum_{t=t_e}^{t_{e+1}-1}\left(J_r^{\pi_e,\Tilde{P}_e} - J_r^{\pi_e, P}\right)\right|\\
        &\overset{(a)}{=} \left|\sum_{e=1}^E\sum_{t=t_e}^{t_{e+1}-1}\sum_{s,a}\nu^{\pi_e, P}(s, a) B^{\pi_e, \Tilde{P}_e}(s,a)\right|\\
        &\overset{(b)}{\le} \left|\sum_{e=1}^E\sum_{t=t_e}^{t_{e+1}-1} B^{\pi_e, \Tilde{P}_e}(s_t,a_t) + 4T_M\sqrt{7T\log(2T)} + \frac{2CT_MSE}{1-\zeta}\right|\\
        &\overset{(c)}{\leq} \left|\underbrace{\sum_{e=1}^E\sum_{t=t_e}^{t_{e+1}-1} T_M\sqrt{\frac{14S\log(2AT)}{1\vee N_e(s,a)}}}_{R_0} + 4T_M\sqrt{7T\log(2T)} + \frac{2CT_MSE}{1-\zeta}\right|\\
    \end{split}
    \label{eq:R_2-partial-bound}
\end{align}
where $(a)$, $(b)$, $(c)$ are consequences of Lemma \ref{lem:bound_average_by_bellman}, \ref{lem:bound_expected_bellman}, and \ref{lem:bound_bellman_s_a} respectively. Note that the term $R_0$ in the above inequality can be bounded as follows.

\begin{align}
    \begin{split}	
        &R_0\le \sum_{e=1}^E\sum_{s,a}N^{\mathrm{curr}}_e(s,a)T_M\sqrt{\frac{14S\log(2AT)}{1\vee N_e(s,a)}} \\
        &\le \sum_{s,a}T_M\sqrt{14SA\log(2AT)}\sum_{e=1}^E\frac{N^{\mathrm{curr}}_e(s,a)}{\sqrt{1\vee N_e(s,a)}}\\
        &\overset{(d)}{\leq} \sum_{s,a}T_M(\sqrt{2}+1)\sqrt{14SA\log(2AT)}\sqrt{N(s,a)} + 4T_M\sqrt{7T\log(2T)}\\
        &\overset{(e)}{\leq} T_M(\sqrt{2}+1)\sqrt{14SA\log(2AT)}\sqrt{\left(\sum_{s,a} 1\right)\left(\sum_{s,a}N(s,a)\right)}\\
	&\le T_M(\sqrt{2}+1)\sqrt{14SA\log(2AT)}\sqrt{SAT} 
    \end{split}
    \label{eq:R_0-bound}
\end{align}
Equation $(d)$ follows from \citep{jaksch2010near}, and $(e)$ is a consequence of Cauchy-Schwarz inequality. Combining $\eqref{eq:R_2-partial-bound}$ and \eqref{eq:R_0-bound}, we obtain a bound on $R_2(T)$.

\subsubsection{Bounding $R_3(T)$}

Bounding $R_3(T)$ follows mostly similar to  Lemma \ref{lem:bound_expected_bellman}. At each epoch, the agent visits states according to the occupancy measure $\nu_{\pi_e, P}$ and obtains the rewards. We bound the deviation of the observed visitations to the expected visitations to each state-action pair in each epoch.

\begin{lemma}\label{lem:bound_average_observed_reward_gap}
    With probability at least $1-1/T^6$, the difference between the observed rewards and the expected rewards is bounded as:
    \begin{align}
        \Big|\sum_{e=1}^E\sum_{t=t_e}^{t_{e+1}-1}\mathbb{E}_{\pi_e,P}\left[r(s,a)\right]  - \sum_{e=1}^E\sum_{t=t_e}^{t_{e+1}-1}r(s_t,a_t)\Big| \le 2\sqrt{7T\log(2T)}
    \end{align}
\end{lemma}
\begin{proof}
    We note that $\mathbb{E}_{\pi_e,P}\left[r(s,a)|\mathcal{F}_{t-1}\right] - r(s_t, a_t)$ is a Martingale difference sequence bounded by $2$ because the rewards are bounded by $1$. Hence, following the proof of Lemma \ref{lem:bound_expected_bellman} we get the required result.
\end{proof}

\subsubsection{Bounding the number of episodes $E$}

The number of episodes $E$ of the C-UCRL algorithm are bounded by $1 + 2SA + SA\log(T/SA)$ from Proposition 18 of \citep{jaksch2010near}. 


\subsection{Bounding Constraint Violations}\label{app:constraint_violation_bounds}

To bound the constraint violations $C(T)$, we break it into multiple components. We can then bound these components individually.

\subsubsection{Constraint breakdown}
We first break down our constraint violations into multiple parts.
\begin{align}
    \begin{split}
        &C(T) = \left(\sum_{t=1}^Tc(s_t, a_t)\right)_+\\
        &= \Bigg(\sum_{t=1}^Tc(s_t, a_t)+ \sum_{e=1}^E T_e J_c^{\pi_e,\Tilde{P}_e} - \sum_{e=1}^E T_e J_c^{\pi_e, \Tilde{P}_e}\Bigg)_+\\
        &\overset{(a)}{\leq} \left(\sum_{t=1}^Tc(s_t, a_t) - \sum_{e=1}^E T_e J_c^{\pi_e, \Tilde{P}_e}  - \sum_{e=1}^E T_e\epsilon_e\right)_+\\
        &\le \left(\left|\sum_{e=1}^E\sum_{t=t_e}^{t_{e+1}-1}\left(c(s_t, a_t)  - J_c^{\pi_e, \Tilde{P}_e} \right)\right|  - \sum_{e=1}^E T_e\epsilon_e\right)_+\\
        &\le \left(\left|\sum_{e=1}^E\sum_{t=t_e}^{t_{e+1}-1}\left(c(s_t, a_t) - J_c^{\pi_e, P} + J_c^{\pi_e, P} - J_c^{\pi_e, \Tilde{P}_e} \right)\right|  - \sum_{e=1}^E T_e\epsilon_e\right)_+\\
        &\le \left(\left|\sum_{e=1}^E\sum_{t=t_e}^{t_{e+1}-1}\left(c(s_t, a_t) - J_c^{\pi_e, P}\right)\right| \right.\\
        &\hspace{1cm}\left.+ \left|\sum_{e=1}^E\sum_{t=t_e}^{t_{e+1}-1}\left(J_c^{\pi_e, P} - J_c^{\pi_e, \Tilde{P}_e} \right)\right|  - \sum_{e=1}^E T_e\epsilon_e\right)_+\\
        &\le \left(C_3(T) + C_2(T)  - C_1(T)\right)_+
    \end{split}
    \label{eq:define_broken_regreta}
\end{align}
where $(a)$ is a result of the fact the policy $\pi_e$ is solution of a conservative optimization equation. The three components of the constraint violation are defined below. The first component, $C_1(T)$ can be written as:
\begin{align}
    C_1(T) &= \sum_{e=1}^E T_e \epsilon_e
\end{align}
$C_1(T)$ denotes the gap left by playing the policy for $\epsilon_e$-tight optimization problem on the optimistic MDP.
\begin{align}
    C_2(T) &= \left|\sum_{e=1}^E\sum_{t=t_e}^{t_{e+1}-1}\left(J_c^{\pi_e, P}-J_c^{\pi_e, \Tilde{P}_e} \right)\right|
\end{align}
$C_2(T)$ denotes the difference between long-term average costs incurred by playing the policy $\pi_e$ on the true MDP with transitions $P$ and the optimistic MDP with transitions $\Tilde{P}$. This term is bounded similarly to the bound of $R_2(T)$.
\begin{align}
    C_3(T) &= \left|\sum_{e=1}^E\sum_{t=t_e}^{t_{e+1}-1}\left(c(s_t, a_t) - J_c^{\pi_e, P}\right)\right|
\end{align}
$C_3(T)$ denotes the difference between long-term average costs incurred by playing the policy $\pi_e$ on the true MDP with transitions $P$ and the realized costs. This term is bounded similarly to the bound of $R_3(T)$.

\subsubsection{Bounding $C_1(T)$}\label{sec:conservative_effect}
Note that $C_1(T)$ is a component of the constraint violation that occurs due to the lack of knowledge of the true MDP and deviations of incurred costs from the expected costs. We bound $C_1(T)$ for sufficient slackness. With this idea, we have the following set of equations.
\begin{align}
    \begin{split}
        C_1(T) = \sum_{e=1}^E\sum_{t=t_e}^{t_{e+1}-1}\epsilon_e &= \sum_{e=1}^E\sum_{t=t_e}^{t_{e+1}-1}K\sqrt{\frac{\log t_e}{t_e}}\\
        &\ge K\sum_{e=E'}^E\sum_{t=t_e}^{t_{e+1}-1}\sqrt{\frac{\log (T/4)}{t_e}}\\
        &\ge K\sum_{e=E'}^E\sum_{t=t_e}^{t_{e+1}-1}\sqrt{\frac{\log (T/4)}{T}}\\
        &=K\left(T-t_{E'}\right)\sqrt{\frac{\log (T/4)}{T}}\\
        &\ge K\frac{T}{2}\sqrt{\frac{\log (T/4)}{T}}\ge \frac{K}{4}\sqrt{T\log T}
    \end{split}
\end{align}
where $E'$ is some epoch for which $T/4\le t_{E'} < {T/2}$.

\subsubsection{Bounding $C_2(T)$, and $C_3(T)$}
We observe that the terms $C_2(T)$, and $C_3(T)$ follows the same bound as $R_2(T)$ and $R_3(T)$ respectively. Thus, replacing $r$ with $c$, we obtain constraint violations because of imperfect system knowledge and system stochasticity as $\Tilde{O}(T_MS\sqrt{AT})$. Summing the three terms and choosing $K = \Theta(T_MS\sqrt{A})$ gives the required bound on constraint violation.


\section{Regret Analysis and Constraint Violation for Posterior Sampling Based  Approach}\label{sec:mb:regpos}

For the C-UCRL algorithm, the true MDP lies in the confidence interval with high probability, and hence the solution of the optimization problem was guaranteed. However, the same is not true for the MDP with sampled transition probabilities. We want the existence of a policy $\pi_e$ such that Equation \eqref{eq:eps_cmdp_constraints} holds. We obtain the condition for the existence of such a policy in the following lemma. To obtain the lemma, we first state a tighter Slater assumption as:

\begin{assumption}
    \label{conservative_slaters_conditon}
    There exists two constants $\delta > ST_M\sqrt{(A\log T)/T} + (CSA\log T)/(T(1-\zeta))$ and $\Gamma > 2ST_M\sqrt{14A\log AT/T^{1/3}} + CST_M/((1-\zeta)T^{1/3})$, and a policy $\pi$ such that 
    \begin{align}
	J_c^{\pi, P} \le -\delta - \Gamma 
    \end{align}
\end{assumption}

\begin{lemma}
    \label{lem:psrl_loose_slater}
    If there exists $\pi$, such that $J_c^{\pi, P} \le -\delta - \Gamma$, and there exist episodes $e$ and $e+1$ with start timesteps $t_{e}$ and $t_{e+1}$ respectively satisfying $t_{e+1} - t_e \ge T^{1/3}$, then for $\|\Tilde{P}_e(\cdot|s,a) - P(\cdot|s,a)\|_1\le \sqrt{\frac{14S\log(2At)}{N_e(s,a)}}$, the policy $\pi$ satisfies the following inequality.
    \begin{align}
	J_c^{\pi, \Tilde{P}_e} \le -\delta.
    \end{align}
\end{lemma}

\begin{proof}
    We have,
    \begin{align}
        J_c^{\pi, \Tilde{P}_e}\le |J_c^{\pi, \Tilde{P}_e} - J_c^{\pi, P}| + J_c^{\pi, P}
        \label{eq:lipschtiz_choose_side}
    \end{align}
    We now bound the term $|J_c^{\pi,\Tilde{P}_e}-J_c^{\pi,P}|$ using Bellman error. We have,
    \begin{align}
        J_c^{\pi, \Tilde{P}_e} - J_c^{\pi, P} = \sum_{s,a}\nu^{\pi, P}(s, a) B^{\pi_e, \Tilde{P}_e}_c(s,a) = \mathbb{E}\left[B^{\pi_e, \Tilde{P}_e}_c(s,a)\right]
    \end{align}
    where $B^{\pi_e, \Tilde{P}_e}_c(s,a)$ is the Bellman error for cost function. We bound the expectation using Azuma-Hoeffding's inequality as follows:
	
    \begin{align}
        \begin{split}
            &\mathbb{E}\left[B^{\pi_e, \Tilde{P}_e}_c(s,a)\right] \\
            &=\mathbb{E}\left[B^{\pi_e, \Tilde{P}_e}_c(s_t,a_t)|\mathcal{F}_{t_e-1}\right] + C\zeta^{t-t_e}\\
            &\overset{(a)}{=} \frac{1}{t_{e+1}-t_e}\sum_{t=t_e}^{t_{e+1}-1}\left(\mathbb{E}\left[B^{\pi_e, \Tilde{P}_e}_c(s_t,a_t)|\mathcal{F}_{t_e-1}\right] + C\zeta^{t-t_e}\right)\\
            &\overset{(b)}{\le} \frac{1}{t_{e+1}-t_e}\sum_{t=t_e}^{t_{e+1}-1}\left(\mathbb{E}\left[B^{\pi_e, \Tilde{P}_e}_c(s_t,a_t)|\mathcal{F}_{t_e-1}\right]\right) + \frac{CS\mathrm{sp}(\Tilde{h})}{(1-\zeta)(t_{e+1}-t_e)}\\
            &\overset{(c)}{\le} \frac{1}{t_{e+1}-t_e}\left(\mathrm{sp}(\Tilde{h})\sqrt{14S\log AT}\sum_{s,a}\frac{N^{\mathrm{curr}}_e(s,a)}{\sqrt{N_e(s,a)}} \right.\\
            &+\left. 4\mathrm{sp}(\Tilde{h})\sqrt{7(t_{e+1}-t_e)\log(t_{e+1}-t_e)}\right) + \frac{CS\mathrm{sp}(\Tilde{h})}{(1-\zeta)(t_{e+1}-t_e)}\\        
            &\overset{(d)}{\le} \frac{1}{t_{e+1}-t_e}\left(\mathrm{sp}(\Tilde{h})\sqrt{14S\log AT}\sum_{s,a}\sqrt{N^{\mathrm{curr}}_e(s,a)} \right.\\
            &+\left. 4\mathrm{sp}(\Tilde{h})\sqrt{7(t_{e+1}-t_e)\log(t_{e+1}-t_e)}\right) + \frac{CS\mathrm{sp}(\Tilde{h})}{(1-\zeta)(t_{e+1}-t_e)}\\
            &\overset{(e)}{\le} \frac{1}{t_{e+1}-t_e}\left(\mathrm{sp}(\Tilde{h}) S\sqrt{14A\log AT}\sqrt{\sum_{s,a}N^{\mathrm{curr}}_e(s,a)} \right.\\
            &+\left. 4\mathrm{sp}(\Tilde{h})\sqrt{7(t_{e+1}-t_e)\log(t_{e+1}-t_e)}\right) + \frac{CS\mathrm{sp}(\Tilde{h})}{(1-\zeta)(t_{e+1}-t_e)}\\
            &\overset{(f)}{\le} \frac{1}{t_{e+1}-t_e}\left(\mathrm{sp}(\Tilde{h}) S\sqrt{14A\log AT}\sqrt{(t_{e+1}-t_e)} \right.\\
            &+\left. 4\mathrm{sp}(\Tilde{h})\sqrt{7(t_{e+1}-t_e)\log(t_{e+1}-t_e)}\right) + \frac{CS\mathrm{sp}(\Tilde{h})}{(1-\zeta)(t_{e+1}-t_e)}\\
            &\le \mathrm{sp}(\Tilde{h})\left( S\sqrt{\frac{14A\log AT}{(t_{e+1}-t_e)}} + 4\sqrt{\frac{7\log(t_{e+1}-t_e)}{(t_{e+1}-t_e)}}+ \frac{CS(1-\zeta)^{-1}}{(t_{e+1}-t_e)}\right)
        \end{split}
        \label{eq:reduce_error_terms}
    \end{align}
    where $(a)$ is obtained by summing both sides from $t = t_e$ to $t= t_{e+1}$. Equation $(b)$ is obtained by summing over the geometric series with ratio $\zeta$. Equation $(c)$ comes from Lemma \ref{lem:bound_expected_bellman}. Equation $(d)$ comes from the fact that $N_e(s, a) \ge N^{\mathrm{curr}}_e(s, a)$ for all $s,a$, and then replacing the lower bound of $N_e(s, a)$. Equation $(e)$ follows from the Cauchy-Schwarz inequality. Equation $(f)$ follows from the fact that the epoch length $t_{e+1}-t_e$ is the same as the number of visitations to all state-action pairs in an epoch. 
	
    Combining Equation \eqref{eq:reduce_error_terms} with Equation \eqref{eq:lipschtiz_choose_side}, and bounding the $\|\Tilde{h}(\cdot)\|_\infty$ term with $T_M$, we obtain the required result as follows:
    \begin{align}
        \begin{split}
            J_c^{\pi,\Tilde{P}_e} &\le |J_c^{\pi, \Tilde{P}_e} - J_c^{\pi, P}| + J_c^{\pi, P}\\
            &\le \Big(\left(T_M S\sqrt{\frac{14A\log AT}{(t_{e+1}-t_e)}} + 4T_M\sqrt{\frac{7\log(t_{e+1}-t_e)}{(t_{e+1}-t_e)}}\right)\\
            &~~+ \frac{CT_MS}{(1-\zeta)(t_{e+1}-t_e)}\Big) + J_c^{\pi, P}\\
            &\le \Big(\left(T_M S\sqrt{\frac{14A\log AT}{(t_{e+1}-t_e)}} + 4T_M\sqrt{\frac{7\log(t_{e+1}-t_e)}{(t_{e+1}-t_e)}}\right)\\
            &~~~+ \frac{CT_MS}{(1-\zeta)(t_{e+1}-t_e)}\Big) -\delta - \Gamma \le -\delta
        \end{split}
        \label{eq:use_gamma_def_and_epoch_length_bound}
    \end{align}
    where Equation \eqref{eq:use_gamma_def_and_epoch_length_bound} is follows from the definition of $\Gamma$ in Assumption \ref{conservative_slaters_conditon} and $t_{e+1}-t_e\ge T^{1/3}$.
\end{proof}

From Lemma \ref{lem:psrl_loose_slater}, we observe that for a tighter Slater condition on the true MDP, we can only obtain a weaker Slater guarantee. However, we make that assumption to obtain the feasibility of the optimization problem in Equation \eqref{eq:eps_cmdp_constraints}. 

The Bayesian regret of the C-PSRL algorithm is defined as follows:
\begin{align}
    \mathbb{E}[R(T)]& = \mathbb{E}\left[TJ_r^{\pi^*, P} - \sum\nolimits_{t=1}^Tr(s_t, a_t)\right]
\end{align}
Similarly, we define Bayesian constraint violations, $C(T)$, as the expected gap between the constraint function and incurred and constraint bounds.
\begin{align}
    \mathbb{E}[C(T)] &= \mathbb{E}\left[\left(\sum\nolimits_{t=1}^Tc(s_t, a_t)\right)_+\right] \nonumber
\end{align}
where $(x)_+ = \max(0, x)$.

Now, we can use Posterior Sampling Lemma \cite{agarwal2022concave} to obtain $\mathbb{E}[J_r^{\pi^*, P}|\mathcal{F}_{t_e}] = \mathbb{E}[J_r^{\pi_e, \Tilde{P}_e}|\mathcal{F}_{t_e}]$ and $\mathbb{E}[J_c^{\pi^*, P}|\mathcal{F}_{t_e}] = \mathbb{E}[J_c^{\pi_e, \Tilde{P}_e}|\mathcal{F}_{t_e}]~\forall i$,
and follow the analysis similar to the analysis of Theorem \ref{thm:regret_bound} to obtain the required regret bounds.


\subsection{Bound on constraints}

We now bound the constraint violations and prove that using a conservative policy. We can reduce the constraint violations to $0$. We have:
\begin{align}
    \begin{split}
        &C(T) = \left(\sum_{t=1}^Tc(s_t, a_t) \right)_+\\
        &=\Bigg(\sum_{t=1}^T c(s_t, a_t) - \sum_{e=1}^E T_e J_c^{\pi_e, \Tilde{P}_e}  + \sum_{e=1}^E T_e J_c^{\pi_e, \Tilde{P}_e}\Bigg)_+\\
        &\le \left(\sum_{t=1}^T c(s_t, a_t)- \sum_{e=1}^E T_e J_c^{\pi_e, \Tilde{P}_e} + C_1\right)_+\\
        &\le \left(\left\vert\sum_{e=1}^E\sum_{t=t_e}^{t_{e+1}-1}\left(c(s_t, a_t)  - J_c^{\pi_e, \Tilde{P}_e} \right)\right\vert  +C_1\right)_+ \\
        &\le \left(\left\vert\sum_{e=1}^E\sum_{t=t_e}^{t_{e+1}-1}\left(c(s_t, a_t) - J_c^{\pi_e, P} + J_c^{\pi_e, P} - J_c^{\pi_e, \Tilde{P}_e} \right)\right\vert  +C_1\right)_+\\
        &\le \left(\left\vert\sum_{e=1}^E\sum_{t=t_e}^{t_{e+1}-1}\left(c(s_t, a_t) - J_c^{\pi_e, P}\right)\right\vert + \left\vert\sum_{e=1}^E\sum_{t=t_e}^{t_{e+1}-1}\left(J_c^{\pi_e, P} - J_c^{\pi_e, \Tilde{P}_e} \right)\right\vert  +C_1\right)_+\nonumber\\
        &\le\left(C_3(T) + C_2(T)  + C_1(T)\right)_+
    \end{split}
\end{align}

We bound $C_2(T) + C_3(T)$ similar to the analysis of $R(T)$ by
\begin{align}
    \Tilde{\mathcal{O}}\left(T_MS\sqrt{AT} + \frac{CT_MS^2A}{(1-\zeta)} \right)
\end{align}

We focus our attention on bounding $C_1(T)$.  We  obtain the bound on $C_1(T)$ as:
\begin{align}
    \begin{split}
        &C_1(T) = \sum_{e=1}^ET_e\left(J_c^{\pi_e, \Tilde{P}_e} \right)\\
        &=\sum_{e=1}^ET_e\left(J_c^{\pi_e, \Tilde{P}_e} \right)\bm{1}\{T_e \ge T^{1/3}\} + \sum_{e=1}^ET_e\left(J_c^{\pi_e, \Tilde{P}_e} \right)\bm{1}\{T_e < T^{1/3}\}\\
        &\overset{(a)}{\le} \sum_{e=1}^ET_e\left(J_c^{\pi_e, \Tilde{P}_e} \right)\bm{1}\{T_e \ge T^{1/3}\} + \sum_{e=1}^ET^{1/3}\\
        &\overset{(b)}{\le} -\sum_{e=1}^E T_e\epsilon_e\bm{1}\{T_e \ge T^{1/3}\} + ET^{1/3}\\
        &=-\sum_{e=1}^E T_e\epsilon_e\left(1-\bm{1}\{T_e < T^{1/3}\}\right) + ET^{1/3}\\
        &= -\sum_{e=1}^E T_e\epsilon_e +\sum_{e=1}^E T_e\epsilon_e\bm{1}\{T_e < T^{1/3}\} + ET^{1/3}\\
        &\le -\frac{K}{4}\sqrt{T\log T} +\sum_{e=1}^E T^{1/3}\delta + ET^{1/3}\\
        &= -\frac{K}{4}\sqrt{T\log T} +E\delta T^{1/3} +  ET^{1/3}
    \end{split}
\end{align}
where $(a)$ is a consequence of the fact that the maximum cost is $1$ and $(b)$ follows from following the conservative policy. Thus, choosing an appropriate $K$, we can bound constraint violations by $0$.


\section{Evaluation Results}\label{sec:mb:eval}

To validate the performance of the C-UCRL algorithm and the C-PSRL algorithm,  we run the simulation on the flow and service control in a single-serve queue, which was introduced in \citep{altman1991constrained}. Along with validating the performance of the proposed algorithms, we also compare the algorithms against the algorithms proposed in \cite{singh2020learning} and in \cite{chen2022learning} for model-based constrained reinforcement learning for infinite horizon MDPs. Compared to these algorithms, we note that our algorithm is also designed to handle concave objectives of expected rewards with convex constraints on costs with $0$ constraint violations.

In the queue environment, a discrete-time single-server queue with a buffer of finite size $L$ is considered. The number of customers waiting in the queue is considered as the state in this problem and thus $\vert S\vert=L+1$. Two kinds of actions, services, and flow, are considered in the problem and control the number of customers together. The action space for service is a finite subset $A$ in $[a_{min},a_{max}]$, where $0<a_{min}\leq a_{max}<1$. Given a specific service action $a$, the service a customer is successfully finished with the probability $b$. If the service is successful, the length of the queue will be reduced by 1. Similarly, the space for flow is also a finite subsection $B$ in $[b_{min}, b_{max}]$. In contrast to the service action, flow action will increase the queue by $1$ with probability $b$ if the specific flow action $b$ is given. Also, we assume that there is no customer arriving when the queue is full. The overall action space is the Cartesian product of the $A$ and $B$. According to the service and flow probability, the transition probability can be computed and is given in Table \ref{table:transition}.

\begin{table*}[ht]   
    \caption{Transition probability of the queue system}  
    \label{table:transition}
    \begin{adjustbox}{width=.95\textwidth,center}
	\begin{tabular}{|c|c|c|c|}  
		\hline  
            Current State & $P(x_{t+1}=x_t-1)$ & $P(x_{t+1}=x_t)$ & $P(x_{t+1}=x_t+1)$ \\ \hline
            $1\leq x_t\leq L-1$ & $a(1-b)$ & $ab+(1-a)(1-b)$ & $(1-a)b$ \\ \hline
            $x_t=L$ & $a$ & $1-a$ & $0$ \\ \hline
            $x_t=0$ & $0$ & $1-b(1-a)$ & $b(1-a)$ \\ 
            \hline  
	\end{tabular}  
    \end{adjustbox}
\end{table*}
\begin{figure}[t]
    \includegraphics[trim=1.5in 5.5in 1.7in 1.6in, clip, width=\textwidth]{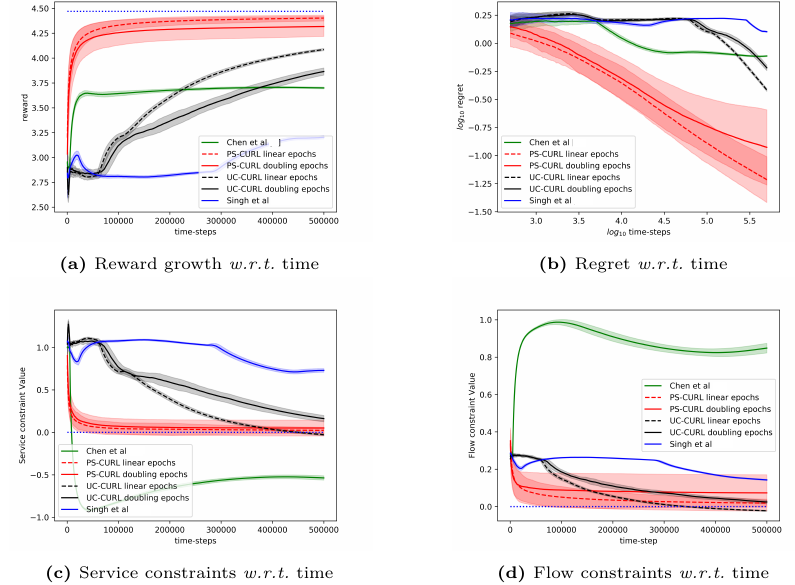}
    \caption{Performance of the proposed C-UCRL and C-PSRL algorithms on a flow and service control problem for a single queue with doubling epoch lengths and linearly increasing epoch lengths. The algorithms are compared against \citet{chen2022learning} and \citet{singh2020learning}. We note that the considered algorithms C-UCRL and C-PSRL  are labeled UC-CURL and PS-CURL, respectively, in the figure.}
    \label{fig:empirical_study}
\end{figure}

Define the reward function as $r(s, a, b)$ and the constraints for service and flow as $c^1(s, a, b)$ and $c^2(s, a, b)$, respectively. Define the stationary policy for service and flow as $\pi_a$ and $\pi_b$, respectively. Then, the problem can be defined as
\begin{equation}
    \begin{split}
        \max_{\pi_a,\pi_b} &\quad \lim\limits_{T\rightarrow\infty}\frac{1}{T}\sum_{t=1}^{T}r(s_t,\pi_a(s_t),\pi_b(s_t))\\
        s.t. &\quad \lim\limits_{T\rightarrow\infty}\frac{1}{T}\sum_{t=1}^{T}c^1(s_t,\pi_a(s_t),\pi_b(s_t))\geq 0\\
        &\quad \lim\limits_{T\rightarrow\infty}\frac{1}{T}\sum_{t=1}^{T}c^2(s_t,\pi_a(s_t),\pi_b(s_t))\geq 0
    \end{split}
\end{equation}

According to the discussion in \citep{altman1991constrained}, we define the reward function as $r(s, a,b)=5 - s$, which is a decreasing function only dependent on the state. It is reasonable to give a higher reward when the number of customers waiting in the queue is small. For the constraint function, we define $c^1(s, a,b)=-10a + 6$ and $c^2 = - 8(1-b)^2+2$, which are dependent only on service and flow action, respectively. A higher constraint value is given if the probability for the service and flow are low and high, respectively.

In the simulation, the length of the buffer is set as $L=5$. The service action space is set as $[0.2,0.4,0.6,0.8]$, and the flow action space is set as $[0.4,0.5,0.6,0.7]$. We use the length of horizon $T=5\times10^5$ and run $50$ independent simulations of all algorithms. The experiments were run on a $36$ core Intel-i9 CPU @$3.00$ GHz with $64$ GB of RAM. The result is shown in the Figure \ref{fig:empirical_study}. The average values of the cumulative reward and the constraint functions are shown in the solid lines. Further, we plot the standard deviation around the mean value in the shadow to show the random error. In order to compare this result to the optimal, we assume that the full information of the transition dynamics is known and then use  Linear Programming to solve the problem. The optimal cumulative reward for the constrained optimization is calculated to be 4.48 with both flow constraint and service constraint values to be $0$. The optimal cumulative reward for the unconstrained optimization is 4.8, with the service constraint being $-2$ and the flow constraint being $-0.88$.

We now discuss the performance of all the algorithms, starting with our algorithms C-UCRL and C-PSRL. In Figure \ref{fig:empirical_study}, we observe that the proposed C-UCRL algorithm does not perform well initially. We observe that this is because the confidence interval radius $\sqrt{S\log(At)/N(s, a)}$ for any $(s, a)$ is not tight enough in the initial period. After the algorithms collect sufficient samples to construct tight confidence intervals around the transition probabilities, the algorithm starts converging toward the optimal policy. We also note that the linear epoch modification of the algorithm (where a new episode is triggered whenever $N^{\mathrm{curr}}_e(s, a)$ becomes $\max\{1, N_{e-1}(s, a)\}$ for any state-action pair, which results in a linearly increasing episode length with total epochs bounded by $O(SA +\sqrt{SAT})$) works better than the doubling epoch algorithm. This is because the linear epoch variant updates the policy quickly, whereas the doubling epoch algorithm works with the same policy for too long and thus loses the advantages of collected samples. For our implementation, we choose the value of parameter $K$  as $K=1$, using which we observe that the constraint values start converging towards zero.

We now analyze the performance of the C-PSRL algorithm. For our implementation of the C-PSRL algorithm, we sample the transition probabilities using the Dirichlet distribution. Note that the true transition probabilities were not sampled from a Dirichlet distribution, and hence, this experiment also shows the robustness against misspecified priors. We observe that the algorithm quickly brings the reward close to the optimal rewards. The performance of the C-PSRL algorithm is significantly better than the C-UCRL algorithm. We suspect this is because the C-UCRL algorithm wastes a large number of steps to find an optimistic policy with a large confidence interval. This observation aligns with the TDSE algorithm \cite{ouyang2017learning}, where they show that the Thompson sampling algorithm with $O(\sqrt{SAT})$ epochs performs empirically better than the optimism-based UCRL2 algorithm \cite{jaksch2010near} with $O(\sqrt{SA\log T)}$ epochs. \cite{osband2013more} also made a similar observation where their PSRL algorithm worked better than the UCRL2 algorithm. Again, we set the value of parameter $K$ as $1$, and with $K=1$, the algorithm does not violate constraints. We also observe that the standard deviation of the rewards and constraints are higher for the C-PSRL algorithm as compared to the C-UCRL algorithm, as the C-PSRL algorithm has an additional stochastic component that arises from sampling the transition probabilities.

After analyzing the algorithms presented in this paper, we now analyze the performance of the algorithm by  \citet{chen2022learning}. They provide an optimistic online mirror descent algorithm that also works with conservative parameters to tightly bound constraint violations. Their algorithm also obtains a $O(\sqrt{T})$ regret bound. However, their algorithm is designed for a linear reward/constraint setup with a single constraint, and empirically, the algorithm is difficult to tune as it requires additional knowledge of $T_M$, $\zeta$, $\delta$, and $T$ to fine-tune parameters used in their algorithm. We set the value of the learning rate $\theta$ for online mirror descent as $5\times10^{-2}$ with an episode length of $5\times10^3$. Further, we scale the rewards and costs to ensure that they lie between $0$ and $1$. We analyze the behavior of the optimistic online mirror descent algorithm in Figure \ref{fig:empirical_study}(b). We observe that the algorithm has three phases. The first phase is the first episode where the algorithm uses a uniform policy, which is the initial flat area till the first $5000$ steps. In the second phase, the algorithm updates the policy for the first time and starts converging to the optimal policy with a convergence rate which matches to that of the C-PSRL algorithm. However, after a few policy updates, we observed that the algorithm has oscillatory behavior because the dual variable updates require online constraint violations.

Finally, we analyze the the algorithm by  \citet{singh2020learning}. They also provide an algorithm that proceeds in epochs and solves an optimization problem at every epoch. The algorithm considers a fixed epoch length $T^{1/3}$. Further, the algorithm considers a confidence interval on each estimate of $P(s'|s, a)$ for all $s,a,s'$ triplet. The algorithm does not perform well even though it updates the policy most frequently because of creating confidence intervals on individual transition probabilities $P(s'|s, a)$ instead of the probability vector $P(s'|s, a)$.

From the experimental observations, we note that the proposed C-UCRL algorithm is suitable in cases where parameter tuning is not possible, and the system requires tighter bounds on the deviation of the performance of the algorithm. The C-PSRL algorithm can be used in cases where the variance in the algorithm's performance can be tolerated or computational complexity is a constraint. Further, for both algorithms, it is beneficial to use the linear increasing epoch lengths. Additionally, the algorithm by \citet{chen2022learning} is suitable for cases where solving an optimization equation is not feasible, for example, an embedded system, as the algorithm updates policy using an exponential function, which can be easily computed.


\section{Notes and Open Problems}\label{sec:mb:notes}

In the presence of constraints, some works on model-based reinforcement learning include \cite{singh2020learning,chen2022learning,wei2022provably}. The results in this chapter are taken from \cite{agarwal2022regret,agarwal2022concave}, which improve upon the state-of-the-art guarantees by improving the order of either $T$ or other terms. However, we note that there is still a gap. Specifically, it remains an open question to see if the regret guarantees can be improved and matched with the lower bound in terms of $S$. Secondly, our bound has $T_M$ dependence, while the lower bound depends on the square root of the diameter. Resolving the gap is an open question. 

We further note that the presentation in this chapter considers linear utility functions with linear constraints. Recently, \cite{agarwal2022multi,agarwal2023reinforcement} considered concave utility functions in model-based reinforcement learning. Further, \cite{agarwal2022concave} studied the impact of convex constraints and concave utility function. The results in this chapter are a specialization of the results in \cite{agarwal2022concave}  for linear functions and constraints. This indicates that the results extend to the general setup of concave utility function and convex constraints. However, getting the matching lower bounds as mentioned for linear utilities and constraints will also be important with non-linear functions.  Table \ref{tab:mb:algo_comparisons} provides the key comparisons and related works for the literature. 

\begin{table*}[thbp]
    \begin{adjustbox}{width=0.99\textwidth,center}
	\begin{tabular}{|c|c|c|c|c|}
		\hline
            Algorithm(s) & Setup & Regret & Constraint Violation & Non-Linear\\
		\hline
            \textsc{con}RL \cite{brantley2020constrained} & FH & $\Tilde{O}(LH^{5/2}S\sqrt{A/K})$& $O(H^{5/2}S\sqrt{A/K})$ & {Yes}\\
		\hline
            MOMA \cite{yu2021morl} & FH & $\Tilde{O}(LH^{3/2}\sqrt{SA/K})$& $\Tilde{O}(H^{3/2}\sqrt{SA/K})$ & {Yes}\\
		\hline    
            TripleQ \cite{wei2021provably} & FH & $\Tilde{O}(\frac{1}{\delta}H^4\sqrt{SA}K^{-1/5})$ & $0$ & {No}\\
		\hline
            OptPess-LP \cite{liu2021learning}&FH &$\Tilde{O}(\frac{H^3}{\delta}\sqrt{S^3A/K})$ & $0$ & No\\
		\hline
            OptPess-Primal Dual \cite{liu2021learning}&FH &$\Tilde{O}(\frac{H^3}{\delta}\sqrt{S^3A/K})$ & $\Tilde{O}(H^4S^2A/\delta)$ & No\\        
		\hline        
            UCRL-CMDP \cite{singh2020learning} & IH & $\Tilde{O}(\sqrt{SA}T^{-1/3})$  & $\Tilde{O}(\sqrt{SA}/T^{1/3})$& {No}\\
		\hline
            \cite{chen2022learning}& IH &$\Tilde{O}(\frac{1}{\delta}T_MS\sqrt{SA/T})$ & $\Tilde{O}(\frac{1}{\delta^2}T_M^2S^3A)$& No\\
		\hline
            \cite{wei2022provably}&IH &$\Tilde{O}(\frac{1}{\delta}\sqrt{SA}T^{-1/6})$ & $0$ & No\\
		\hline            
            \cite{agarwal2022regret}&IH &$\Tilde{O}(T_MS\sqrt{A/T})$  & $\Tilde{O}(T_MS\sqrt{A/T})$  & No\\
		\hline            
            UC-CURL \cite{agarwal2022concave}& IH & $\Tilde{O}(\frac{1}{\delta}LT_MS\sqrt{A/T})$ & $0$ & \textbf{Yes}\\
		\hline         
            PS-CURL \cite{agarwal2022concave}& IH & $\Tilde{O}(\frac{1}{\delta}LT_MS\sqrt{A/T})$ & $0$ & \textbf{Yes}\\
		\hline         
	\end{tabular}
    \end{adjustbox}
    \caption{Overview of work for constrained reinforcement learning setups. For finite horizon (FH) setups,  $H$ is the episode length and $K$ is the number of episodes for which the algorithm runs. For infinite horizon (IH) setups, $T_M$ denotes the max reach time of the MDP, and $T$ is the time for which the algorithm runs. $L$ is Lipschitz constant. Note that all IH setups assume ergodic MDPs, whereas FH setups do not require the ergodic assumption as the system resets to the final state after every episode.}
    \label{tab:mb:algo_comparisons}
\end{table*}

Finally, a key assumption in this chapter is the ergodicity of the underlying MDP.  Note that it is not necessary for the MDP to be ergodic for the sub-linear regret. The authors of \cite{bartlett2009regal,jaksch2010near} have shown sub-linear regret to be possible only when the MDP is at least weakly communicating. From a technical perspective, designing provable algorithms for the infinite-horizon average-reward setting, especially for the general class of weakly communicating MDPs, has always been more challenging than other settings \cite{chen2022learning}. While the diameter quantifies the number of steps needed to ``recover” from a bad state in the worst case, the actual regret incurred while ``recovering” is related to the difference in potential reward between ``bad” and ``good” states, which is accurately measured by the span (i.e., the range) $sp(h^*)$ of the optimal bias function $h^*$. While the diameter is an upper bound on the bias span, it could be arbitrarily larger (e.g., weakly-communicating MDPs may have finite span and infinite diameter) thus suggesting that algorithms whose regret scales with the span may perform significantly better. The authors of \cite{fruit2018efficient} proposed SCOPT, which achieves a regret of $\tilde{O}(sp(h^*)S\sqrt{AT})$ in the absence of constraints. This is the best bound studied so far in the weakly communicating MDPs where the diameter may be infinite with lowest order on $T$. However, constrained versions with $\sqrt{T}$ regret bounds are open. The results for weakly mixing CMDPs will be discussed in Chapter \ref{chpt:beyerg}. 

\chapter{Parameterized Model-Free RL}\label{chpt:param:mf}

We introduce the model and key assumptions studied in this chapter in Section \ref{sec:mf_model}. Further, we provide a primal-dual policy gradient-based algorithm in Section \ref{sec:mf:algo}. The global convergence analysis of this algorithm is provided in Section \ref{sec_convergence}. The regret and constraint violations are analyzed in Section \ref{mf:reg}. Section \ref{sec:mf:notes} gives some notes and possible future directions. Many notations used in this chapter are similar to those in the previous chapter. However, we reintroduce them here for ease of reference and overall coherence of the discussion.


\section{Overall Model and Assumptions}
\label{sec:mf_model}

This paper analyzes an infinite-horizon average-reward constrained Markov Decision Process (CMDP) denoted as $\mathcal{M}=(\mathcal{S},\mathcal{A}, r, c, P,\rho)$ where  $\mathcal{S}$ is the state space, $\mathcal{A}$ is the action space of size $A$, $r:\mathcal{S}\times\mathcal{A}\rightarrow [0,1]$ is the reward function, $c:\mathcal{S}\times\mathcal{A}\rightarrow [-1,1]$ is the constraint cost function, $P:\mathcal{S}\times\mathcal{A}\rightarrow \Delta(\mathcal{S})$ defines the state transition function where $\Delta(\mathcal{S})$ denotes a probability simplex over $\mathcal{S}$, and $\rho\in\Delta(\mathcal{S})$ is the initial distribution of states. A policy $\pi:\mathcal{S}\rightarrow \Delta(\mathcal{A})$ maps the current state to an action distribution. The average reward and cost of a policy, $\pi$, is defined as follows.
\begin{equation}
    \label{Eq_reward}
    J_{g,\rho}^{\pi}\triangleq \lim\limits_{T\rightarrow \infty}\frac{1}{T}\mathbf{E}_{\pi}\bigg[\sum_{t=0}^{T-1}g(s_t,a_t)\bigg|s_0\sim \rho\bigg]
\end{equation}
where $g=r, c$ for average reward and cost respectively. The expectation is calculated over the distribution of all sampled trajectories $\{(s_t, a_t)\}_{t=0}^{\infty}$ where $a_t\sim \pi(s_t)$, $s_{t+1}\sim P(\cdot|s_t, a_t)$, $\forall t\in\{0, 1, \cdots\}$. We do not explicitly show the dependence of $J^{\pi}_{g, \rho}$ on $P$ for notational convenience. Moreover, we also drop the dependence on $\rho$ whenever there is no confusion.

We consider a class of policies, $\Pi$ whose each element is indexed by a $\mathrm{d}$-dimensional parameter, $\theta\in\mathbb{R}^{\mathrm{d}}$. We aim to maximize the average reward function while ensuring that the average cost is above a given threshold. Without loss of generality, we can mathematically represent this problem as follows.
\begin{align} 
    \label{eq:def_constrained_optimization}
    \begin{split}
        \max_{\theta\in\mathbb{R}^{\mathrm{d}}} ~ J_r^{\pi_{\theta}}~~
	\text{s.t.} ~ J_c^{\pi_{\theta}}\leq 0
    \end{split}
\end{align}

Let $J_g^{\pi_\theta} = J_g(\theta)$, $g\in\{r, c\}$, and $P^{\pi_{\theta}}:\mathcal{S}\rightarrow \Delta(\mathcal{S})$ be the transition function induced by $\pi_\theta$ and given by, $P^{\pi_{\theta}}(s, s') = \sum_{a\in\mathcal{A}}P(s'|s,a)\pi_{\theta}(a|s)$, $\forall s, s'$.
If $\mathcal{M}$ is such that for every policy $\pi$, the induced function, $P^{\pi}$
is irreducible, and aperiodic, then $\mathcal{M}$ is called ergodic.

\begin{assumption}
    \label{ass_1}
    The CMDP $\mathcal{M}$ is ergodic.
\end{assumption}

Ergodicity is a common assumption in the literature \citep{pesquerel2022imed, gong2020duality}. If $\mathcal{M}$ is ergodic, then $\forall \theta$, there exists a unique stationary distribution, $d^{\pi_{\theta}}\in \Delta^{|\mathcal{S}|}$ given as follows.
\begin{align}
    d^{\pi_{\theta}}(s) = \lim_{T\rightarrow \infty}\dfrac{1}{T}\left[\sum_{t=0}^{T-1} \mathrm{Pr}(s_t=s|s_0\sim \rho, \pi_{\theta})\right]
\end{align}

Ergodicity implies that $d^{\pi_{\theta}}$ is independent of the initial distribution, $\rho$, and obeys $P^{\pi_{\theta}}d^{\pi_{\theta}}=d^{\pi_{\theta}}$. Hence, the average reward and cost functions can be expressed as,
\begin{align}
    \label{eq_r_pi_theta}
    J_g(\theta) = \mathbf{E}_{s\sim d^{\pi_{\theta}}, a\sim \pi_{\theta}(s)}[g(s, a)] = (d^{\pi_{\theta}})^T g^{\pi_{\theta}}
\end{align}
where $g^{\pi_{\theta}}(s) \triangleq \sum_{a\in\mathcal{A}}g(s, a)\pi_{\theta}(a|s), ~g\in\{r, c\}$. Note that the functions $J_g(\theta)$, $g\in\{r, c\}$ are independent of the initial distribution, $\rho$. Furthermore, $\forall \theta$, there exist a function $Q_g^{\pi_{\theta}}: \mathcal{S}\times \mathcal{A}\rightarrow \mathbb{R}$ such that the following Bellman equation is satisfied $\forall (s, a)\in\mathcal{S}\times\mathcal{A}$.
\begin{equation}
    \label{eq_bellman}
    Q_g^{\pi_{\theta}}(s,a)=g(s,a)-J_g(\theta)+\mathbf{E}_{s'\sim P(\cdot|s, a)}\left[V_g^{\pi_{\theta}}(s')\right]
\end{equation}
where $g\in\{r, c\}$ and  $V_g^{\pi_{\theta}}:\mathcal{S}\rightarrow \mathbb{R}$ is given as,
\begin{align}
    \label{eq_V_Q}
    V_g^{\pi_{\theta}}(s) = \sum_{a\in\mathcal{A}}\pi_{\theta} (a|s)Q_g^{\pi_{\theta}}(s, a), ~\forall s\in\mathcal{S}
\end{align}
Observe that if $Q_g^{\pi_{\theta}}$ satisfies $(\ref{eq_bellman})$, then it is also satisfied by $Q_g^{\pi_{\theta}}+c$ for any arbitrary, $c$. To uniquely define the value functions, we assume that $\sum_{s\in\mathcal{S}}d^{\pi_{\theta}}(s)V_g^{\pi_{\theta}}(s)=0$. In this case, $V_g^{\pi_{\theta}}(s)$ turns out to be,
\begin{equation}\label{def_v_pi_theta_s}
    \begin{aligned}
        V_g^{\pi_{\theta}}(s) &= \sum_{t=0}^{\infty} \sum_{s'\in\mathcal{S}}\left[(P^{\pi_{\theta}})^t(s, s') - d^{\pi_{\theta}}(s')\right]g^{\pi_{\theta}}(s')\\
        &=\mathbf{E}\left[\sum_{t=0}^{\infty}\bigg\{g(s_t,a_t)-J_g(\theta)\bigg\}\bigg\vert s_0=s\right]
    \end{aligned}
\end{equation}
where the expectation is computed over all $\pi_{\theta}$-induced trajectories. Similarly, $\forall (s, a)$, one can uniquely define $Q_g^{\pi_{\theta}}(s, a)$, $g\in\{r, c\}$ as,
\begin{equation}\label{def_q_pi_theta_s}
    Q_g^{\pi_{\theta}}(s,a)=\mathbf{E}\left[\sum_{t=0}^{\infty}\bigg\{ g(s_t,a_t)-J_g(\theta)\bigg\}\bigg\vert s_0=s,a_0=a\right]
\end{equation}

The advantage function $A_g^{\pi_{\theta}}:\mathcal{S}\times \mathcal{A}\rightarrow \mathbb{R}$ is defined as, 
\begin{align}\label{def_A_pi_theta_s}
    A_g^{\pi_{\theta}}(s, a) \triangleq Q_g^{\pi_{\theta}}(s, a) - V_g^{\pi_{\theta}}(s), ~\forall (s, a),\forall g\in\{r, c\}
\end{align}
Assumption \ref{ass_1} also implies the existence of a finite mixing time.
\begin{definition}
    The mixing time of the CMDP $\mathcal{M}$ (under the ergodicity assumption) with respect to a parameterized policy, $\pi_\theta$, is given as $t_{\mathrm{mix}}^{\theta}\triangleq \min\left\lbrace t\geq 1\big| \norm{(P^{\pi_{\theta}})^t(s, \cdot) - d^{\pi_\theta}}\leq \frac{1}{4}, \forall s\right\rbrace$. The overall mixing time, $t_{\mathrm{mix}}$, is given as $t_{\mathrm{mix}}\triangleq \sup_{\theta\in\Theta} t^{\theta}_{\mathrm{mix}}$ which is finite due to ergodicity.
\end{definition}

Mixing time states how fast a CMDP converges to its stationary distribution for a given policy. We define the hitting time below.
\begin{definition}
    The hitting time of an ergodic CMDP $\mathcal{M}$ with respect to a policy $\pi_\theta$, is given as, $t_{\mathrm{hit}}^{\theta}\triangleq  \max_{s\in\mathcal{S}} \frac{1}{d^{\pi_{\theta}}(s)}$. The overall hitting time is defined as $t_{\mathrm{hit}}\triangleq \sup_{\theta\in\Theta} t^{\theta}_{\mathrm{hit}} $ which is finite due to ergodicity.
\end{definition}

Let $J_r^* \triangleq  J_r(\theta^*)$ where $\theta^*$ solves $\eqref{eq:def_constrained_optimization}$. For a given CMDP $\mathcal{M}$, and a time horizon $T$, the regret and constraint violation is defined as follows.
\begin{align}
    R(T) &\triangleq \sum_{t=0}^{T-1} \left(J_r^*-r(s_t, a_t)\right), ~C(T) \triangleq \sum_{t=0}^{T-1} c(s_t, a_t)
\end{align}
where the algorithm  executes the actions, $\{a_t\}$, $t\in\{0, 1, \cdots \}$ based on the history up to time, $t$, and the state, $s_{t+1}$ is decided according to the state transition function, $P$. Our goal is to design an algorithm that achieves low regret and constraint violation bounds. 


\section{Algorithm for Parameterized Model-Free RL}
\label{sec:mf:algo}

We solve the constrained problem $\eqref{eq:def_constrained_optimization}$ using a primal-dual algorithm, which is based on the following saddle point optimization.
\begin{align}
    \label{eq:def_saddle_point_opt}
    \begin{split}
        \max_{\theta\in\Theta}\min_{\lambda\geq 0}&~ J_{\mathrm{L}}(\theta, \lambda)~~\text{where}~J_{\mathrm{L}}(\theta, \lambda) \triangleq J_r(\theta) -\lambda J_c(\theta)
    \end{split}
\end{align}
The function, $J_{\mathrm{L}}(\cdot, \cdot)$, is called the Lagrange function and $\lambda$ the Lagrange multiplier. Our algorithm updates  $(\theta, \lambda)$ following the iteration shown below $\forall k\in\{1,  \cdots, K\}$ with an initial point $(\theta_1, \lambda_1=0)$. 
\begin{align}
    \label{eq:update}
    \begin{split}
        \theta_{k+1} &= \theta_k+\alpha\nabla_\theta J_{\mathrm{L}}(\theta_k, \lambda_k), 	
        ~~\lambda_{k+1} = \mathcal{P}_{[0,\frac{2}{\delta}]}[\lambda_k +\beta J_c(\theta_k)]
    \end{split}
\end{align}
where $\alpha$ and $\beta$ are learning parameters and $\delta$ is the Slater parameter introduced in the following assumption. Finally, for any $\Lambda$, the function $\mathcal{P}_{\Lambda}[\cdot]$ denotes projection onto $\Lambda$.
\begin{assumption}[Slater condition]
	\label{ass_slater}
	There exists $\delta>0$ and $\bar{\pi} \in \Pi$ such that $J_{c}^{\bar{\pi}}(\rho) \leq -\delta$.
\end{assumption}

Notice that in Eq. \eqref{eq:update}, the dual update is projected onto the set $[0,\frac{2}{\delta}]$ because the optimal dual variable for the parameterized problem is bounded in Lemma \ref{lem.boundness}. The gradient of $J_{\mathrm{L}}(\cdot, \lambda)$ can be computed by invoking a variant of the policy gradient theorem \citep{sutton1999policy}. 
\begin{lemma}
	\label{lemma_grad_compute}
	The gradient of  $J_{\mathrm{L}}(\cdot, \lambda)$ is computed as,
	\begin{align}
		\nabla_{\theta} J_{\mathrm{L}}(\theta, \lambda)
		=\mathbf{E}_{s\sim d^{\pi_{\theta}},a\sim\pi_{\theta}(s)}\bigg[A_{\mathrm{L},\lambda}^{\pi_{\theta}}(s,a)\nabla_{\theta}\log\pi_{\theta}(a|s)\bigg]
	\end{align}
 where $A_{\mathrm{L},\lambda}^{\pi_\theta}(s,a)\triangleq A_r^{\pi_\theta}(\theta)-\lambda A_c^{\pi_\theta}(\theta)$ and $\{A_g^{\pi_\theta}\}_{g\in\{r, c\}}$ are given in $\eqref{def_A_pi_theta_s}$. 
\end{lemma}

\begin{proof}
    The proof works similarly for reward and cost functions. We use the notation $J_g, V_g, Q_g$ where $g=r,c$ and derive the following.
    \begin{equation}
	\label{eq_22}
	\begin{aligned}
            &\nabla_{\theta} V_g^{\pi_{\theta}}(s)=\nabla_{\theta}\bigg(\sum_{a}\pi_{\theta}(a|s)Q_g^{\pi_{\theta}}(s, a)\bigg)\\
            &=\sum_{a}\bigg(\nabla_{\theta}\pi_{\theta}(a|s)\bigg)Q_g^{\pi_{\theta}}(s, a)+\sum_{a}\pi_\theta(a|s)\nabla_{\theta} Q_g^{\pi_{\theta}}(s, a)\\
            &\overset{(a)}=\sum_{a}\pi_{\theta}(a|s)\bigg(\nabla_{\theta}\log\pi_{\theta}(a|s)\bigg)Q_g^{\pi_{\theta}}(s, a)\nonumber\\
            &+\sum_{a}\pi_\theta(a|s)\nabla_{\theta} \bigg(g(s, a)-J_g(\theta)+\sum_{s'}P(s'|s, a)V_g^{\pi_{\theta}}(s')\bigg)\\
			&=\sum_{a}\pi_\theta(a|s)\bigg(\nabla_{\theta}\log\pi_{\theta}(a|s)\bigg)Q_g^{\pi_{\theta}}(s, a)\\
            &+\sum_{a}\pi_\theta(a|s) \bigg(\sum_{s'}P(s'|s,a)\nabla_{\theta} V_g^{\pi_{\theta}}(s')\bigg) - \nabla_{\theta}J_g(\theta)
	\end{aligned}
    \end{equation}
    where the step (a) is a consequence of $\nabla_{\theta}\log\pi_{\theta}=\frac{\nabla\pi_{\theta}}{\pi_{\theta}}$ and the Bellman equation. Multiplying both sides by $d^{\pi_{\theta}}(s)$, taking a sum over $s\in\mathcal{S}$, and rearranging the terms, we obtain the following.
    \begin{align}
	\begin{split}
            &\nabla_{\theta}J_g(\theta)=\sum_{s}d^{\pi_{\theta}}(s)\nabla_{\theta}J(\theta)\\
            &=\sum_{s}d^{\pi_{\theta}}(s)\sum_{a}\pi_\theta(a|s)\bigg(\nabla_{\theta}\log\pi_{\theta}(a|s)\bigg)Q_g^{\pi_{\theta}}(s, a)+\sum_{s}d^{\pi_{\theta}}(s)\\
            &\sum_{a}\pi_\theta(a|s) \bigg(\sum_{s'}P(s'|s,a)\nabla_{\theta} V_g^{\pi_{\theta}}(s')\bigg)- \sum_{s}d^{\pi_{\theta}}(s)\nabla_{\theta}V_g^{\pi_\theta}(s)\\
            &\overset{}{=}\mathbf{E}_{s\sim d^{\pi_\theta}, a\sim \pi_\theta(\cdot|s)}\bigg[Q_g^{\pi_{\theta}}(s, a)\nabla_{\theta}\log\pi_{\theta}(a|s)\bigg]\\
            &+\sum_{s}d^{\pi_{\theta}}(s) \sum_{s'}P^{\pi_\theta}(s'|s)\nabla_{\theta} V_g^{\pi_{\theta}}(s') - \sum_{s}d^{\pi_{\theta}}(s)\nabla_{\theta}V_g^{\pi_\theta}(s)\\
            &\overset{(a)}{=}\mathbf{E}_{s\sim d^{\pi_\theta}, a\sim \pi_\theta(\cdot|s)}\bigg[Q_g^{\pi_{\theta}}(s, a)\nabla_{\theta}\log\pi_{\theta}(a|s)\bigg] + \sum_{s'}d^{\pi_{\theta}}(s')\nabla_{\theta} V_g^{\pi_{\theta}}(s')\\
            &- \sum_{s}d^{\pi_{\theta}}(s)\nabla_{\theta}V_g^{\pi_\theta}(s)
			=\mathbf{E}_{s\sim d^{\pi_\theta}, a\sim \pi_\theta(\cdot|s)}\bigg[Q_g^{\pi_{\theta}}(s, a)\nabla_{\theta}\log\pi_{\theta}(a|s)\bigg]
	\end{split}
    \end{align}
    where $(a)$ uses the fact that $d^{\pi_\theta}$ is a stationary distribution. Note that,
    \begin{equation}
	\begin{aligned}
            &\mathbf{E}_{s\sim d^{\pi_{\theta}},a\sim\pi_{\theta}(\cdot|s)}\bigg[ V_g^{\pi_{\theta}}(s)\nabla\log\pi_{\theta}(a|s)\bigg]\\
            &=\mathbf{E}_{s\sim d^{\pi_{\theta}}}\left[ \sum_{a}V_g^{\pi_{\theta}}(s)\nabla_{\theta}\pi_{\theta}(a|s)\right]\\
            &=\mathbf{E}_{s\sim d^{\pi_{\theta}}}\bigg[ V_g^{\pi_{\theta}}(s)\nabla_{\theta}\left(\sum_{a}\pi_{\theta}(a|s)\right)\bigg]\\
            &=\mathbf{E}_{s\sim d^{\pi_{\theta}}}\bigg[ V_g^{\pi_{\theta}}(s)\nabla_{\theta}\left(1\right)\bigg]=0
	\end{aligned}
    \end{equation}
    We can, therefore, replace $Q^{\pi_{\theta}}$ in the policy gradient with the advantage function $A_g^{\pi_{\theta}}(s, a)=Q_g^{\pi_{\theta}}(s, a)-V_g^{\pi_{\theta}}(s)$, $\forall (s, a)\in\mathcal{S}\times \mathcal{A}$. Thus,
    \begin{equation}
        \nabla_{\theta} J_g(\theta)=\mathbf{E}_{s\sim d^{\pi_{\theta}},a\sim\pi_{\theta}(\cdot|s)}\bigg[ A_g^{\pi_{\theta}}(s,a)\nabla_{\theta}\log\pi_{\theta}(a|s)\bigg]
    \end{equation}
    Notice that the above equation works for both $J_r$ and $J_c$, and thus the proof is completed by the definition of $J_{\mathrm{L},\lambda}$ and $A_{\mathrm{L},\lambda}$.
\end{proof}

In typical RL scenarios, the learners do not have access to $P$, the state transition function, and thereby to the functions $d^{\pi_\theta}$ and $A^{\pi_\theta}_{\mathrm{L}, \lambda}$. This makes computing the exact gradient a difficult task. In Algorithm \ref{alg:PG_MAG}, we demonstrate how one can still obtain good estimates of the gradient using sampled trajectories. It is worthwhile to point out that PG-type algorithms have been widely studied in discounted reward setups. For example, \citep{agarwal2021theory} characterizes the sample complexities of the PG and the Natural PG (NPG) algorithms with softmax and direct parameterization. Similar results for general parameterization are obtained by \citep{liu2020improved, mondal2023improved}. However, the main difference between a discounted and an average-reward setup is that while the former assumes access to a simulator that leads to unbiased estimates of the gradient, the latter framework primarily works on a single sample path.

Algorithm \ref{alg:PG_MAG} runs $K=T/H$ epochs, where $H=16t_{\mathrm{hit}}t_{\mathrm{mix}}T^{\xi}(\log T)^2$ is the duration of each epoch. The constant $\xi$ is specified later. Observe that the learner is assumed to know $T$. This can be relaxed utilizing the well-known doubling trick \citep{lattimore2020bandit}. Additionally, it is assumed that the algorithm is aware of $t_{\mathrm{mix}}$, and $t_{\mathrm{hit}}$. This is commonly assumed in the literature \citep{bai2023regret, wei2020model}. The first step in obtaining a gradient estimate is estimating the advantage value for a given pair $(s, a)$. This is obtained via Algorithm \ref{alg:estQ}. At the $k$th epoch, a $\pi_{\theta_k}$ -induced trajectory, $\mathcal{T}_k=\{(s_t, a_t)\}_{t=(k-1)H}^{kH-1}$ is obtained and passed to Algorithm \ref{alg:estQ} that searches for subtrajectories within it that start with a state $s$, are of length $N=4t_{\mathrm{mix}}(\log T)$, and are at least $N$ distance apart from each other. Assume that there are $M$ such subtrajectories. Let the total reward and cost obtained in the $i$th subtrajectory be $\{r_i, c_i\}$ respectively and $\tau_i$ be its starting time. The value functions for the $k$th epoch are defined as,
\begin{align}
    \begin{split}
        \hat{Q}_g^{\pi_{\theta_k}}(s, a) &= \dfrac{1}{\pi_{\theta_k}(a|s)} \left[\dfrac{1}{M}\sum_{i=1}^M g_i \mathrm{1}(a_{\tau_i}=a)\right],\\
        \hat{V}_g^{\pi_{\theta_k}}(s) &=  \left[\dfrac{1}{M}\sum_{i=1}^M g_i \right], ~~\forall g\in\{r, c\}
    \end{split}
\end{align}
This leads to the following advantage estimator.
\begin{align}
    \begin{split}
        \hat{A}^{\pi_{\theta_k}}_{\mathrm{L}, \lambda_k} (s, a) &= \hat{A}^{\pi_{\theta_k}}_{r} (s, a) - \lambda_k \hat{A}^{\pi_{\theta_k}}_{c} (s, a)\\
        ~~\text{where}~ \hat{A}^{\pi_{\theta_k}}_{g}(s, a) &= \hat{Q}^{\pi_{\theta_k}}_{g}(s, a) - \hat{V}^{\pi_{\theta_k}}_{g}(s)
    \end{split}
\end{align}
where $g\in\{r,c\}$. Finally, the gradient estimator is obtained as follows.
\begin{equation}
    \begin{aligned}\label{eq_grad_estimate}
        \omega_k&\triangleq\hat{\nabla}_{\theta} J_{\mathrm{L}}(\theta_k,\lambda_k)=\dfrac{1}{H}\sum_{t=t_k}^{t_{k+1}-1}\hat{A}_{\mathrm{L},\lambda_k}^{\pi_{\theta_k}}(s_{t}, a_{t})\nabla_{\theta}\log \pi_{\theta_k}(a_{t}|s_{t})
    \end{aligned}
\end{equation}
where $t_k=(k-1)H$ is the starting time of the $k$th epoch. The parameters are updated following \eqref{udpates_algorotihm}. To update the Lagrange multiplier, we need an estimation of $J_g(\theta_k)$, which is obtained as the average cost of the $k$th epoch. It should be noticed that we remove the first $N$ samples from the $k$th epoch because we require the state distribution to be close enough to the stationary distribution $d^{\pi_{\theta_k}}$, which is the key to make $\hat{J}_g(\theta_k)$ close to $J_g(\theta_k)$. The following lemma shows that $\hat{A}_{\mathrm{L},\lambda_k}^{\pi_{\theta_k}}(s, a)$ is a good estimator of $A_{\mathrm{L},\lambda_k}^{\pi_{\theta_k}}(s, a)$.

\begin{lemma}
    \label{lemma_good_estimator}
    The following holds $\forall k$, $\forall (s, a)$ and sufficiently large $T$.
    \begin{align}
        \begin{split}
            &\mathbf{E}\bigg[\bigg(\hat{A}_{\mathrm{L}, \lambda_k}^{\pi_{\theta_k}}(s, a)-A^{\pi_{\theta_k}}_{\mathrm{L}, \lambda_k}(s, a)\bigg)^2\bigg] \nonumber\\
            & \leq \mathcal{O}\left(\dfrac{ t_{\mathrm{hit}}N^3\log T}{\delta^2 H\pi_{\theta_k}(a|s)}\right) = \mathcal{O}\left(\dfrac{t_{\mathrm{mix}}^2(\log T)^2}{\delta^2T^{\xi}\pi_{\theta_k}(a|s)}\right) 
        \end{split}
        \label{eq_good_estimator}
    \end{align}
\end{lemma}

Lemma \ref{lemma_good_estimator} establishes that the $L_2$ error of our proposed advantage estimator can be bounded above as $\Tilde{\mathcal{O}}(T^{-\xi})$. We later utilize the above result to prove the goodness of the gradient estimator. It is to be clarified that our Algorithm \ref{alg:estQ} is inspired by Algorithm 2 of \citep{wei2020model}. However, while \citep{wei2020model} take $H=\tilde{\mathcal{O}}(1)$, we have $H=\tilde{\mathcal{O}}(T^{\xi})$. This subtle change is important in proving the desired sublinear regret for general parametrization.

\begin{algorithm}[t]
    \caption{Primal-Dual Parameterized Policy Gradient}
    \label{alg:PG_MAG}
    \begin{algorithmic}[1]
        \STATE \textbf{Input:} Episode length $H$, learning rates $\alpha,\beta$, initial parameters $\theta_1,\lambda_1$, initial state $s_0 \sim \rho(\cdot)$, 
	\vspace{0.1cm}
	\FOR{$k\in\{1, \cdots, K=T/H\}$}
		\STATE $\mathcal{T}_k\gets \phi$
		\FOR{$t\in\{(k-1)H, \cdots, kH-1\}$}
                \STATE Execute $a_t\sim \pi_{\theta_k}(\cdot|s_t)$, observe $r(s_t,a_t)$, $g(s_t,a_t)$ and  $s_{t+1}$
                \STATE $\mathcal{T}_k\gets \mathcal{T}_k\cup \{(s_t, a_t)\}$
		\ENDFOR	
		\FOR{$t\in\{(k-1)H, \cdots, kH-1\}$}
                \STATE Obtain $\hat{A}_{\mathrm{L},\lambda_k}^{\pi_{\theta_k}}(s_t, a_t)$ using Algorithm \ref{alg:estQ} and $\mathcal{T}_k$ 
		\ENDFOR
		\vspace{0.1cm}
            \STATE  Using \eqref{eq_grad_estimate}, compute  $\omega_k$ 
		\STATE Update parameters as
		\begin{equation}\label{udpates_algorotihm}
			\begin{aligned}
				&\theta_{k+1}=\theta_k+\alpha\omega_k,\\
                    &\lambda_{k+1}= \max\{0, \lambda_k + \beta \hat{J}_c(\theta_k)\}
			\end{aligned}
		\end{equation}
            \STATE where $\hat{J}_c(\theta_k)=\frac{1}{H-N}\sum_{t=(k-1)H+N}^{kH-1}c(s_t, a_t)$
		\ENDFOR
    \end{algorithmic}
\end{algorithm}

\begin{algorithm}[t]
    \caption{Advantage Estimation}
    \label{alg:estQ}
    \begin{algorithmic}[1]
        \STATE \textbf{Input:} Trajectory $(s_{t_1}, a_{t_1},\ldots, s_{t_2}, a_{t_2})$, state $s$, action $a$, Lagrange multiplier $\lambda$ and policy parameter $\theta$
        \STATE \textbf{Initialize:} $M \leftarrow 0$, $\tau\leftarrow t_1$
	\STATE \textbf{Define:} $N=4t_{\mathrm{mix}}\log_2T$.		
	\vspace{0.1cm}
	\WHILE{$\tau\leq t_2-N$}
            \IF{$s_{\tau}=s$}
                \STATE $M\leftarrow M+1$, $\tau_M\gets \tau$
                \STATE $g_{M}\gets \sum_{t=\tau}^{\tau+N-1}g(s_t, a_t)$, ~$\forall g\in\{r, c\}$
                \STATE $\tau\leftarrow\tau+2N$.
            \ELSE
                \STATE {$\tau\leftarrow\tau+1$.}
            \ENDIF
	\ENDWHILE
	\vspace{0.1cm}
	\IF{$M>0$}
            \STATE $\hat{Q}_g(s,a) = \dfrac{1}{\pi_{\theta}(a|s)}\left[\dfrac{1}{M}\sum_{i=1}^M g_{i}\mathrm{1}(a_{\tau_i}=a)\right]$,
            \STATE $\hat{V}_g(s)=\dfrac{1}{M}\sum_{i=1}^M g_{i}$,~~$\forall g\in\{r, c\}$
	\ELSE
            \STATE $\hat{V}_g(s)=0$, $\hat{Q}_g(s, a)=0$, ~~$\forall g\in\{r, c\}$
	\ENDIF
        \STATE \textbf{return}  $(\hat{Q}_r(s,a)-\hat{V}_r(s))- \lambda(\hat{Q}_c(s,a)-\hat{V}_c(s))$ 
    \end{algorithmic}
\end{algorithm}


\begin{proof} [Proof of Lemma \ref{lemma_good_estimator}] Fix an epoch $k$ and assume that $\pi_{\theta_k}$ is denoted as $\pi$ for notational convenience. Let, $M$ be the number of disjoint sub-trajectories of length $N$ that start with the state $s$ and are at least $N$ distance apart (found by Algorithm \ref{alg:estQ}). Let, $g_{k, i}$ be the sum of rewards or constraint ($g=r, c$ accordingly) observed in the $i$th sub-trajectory and $\tau_i$ denote its starting time. The advantage function estimate is,
    \begin{align}
        \label{def_A_hat_appndx}
        \hat{A}_g^{\pi}(s, a) = \begin{cases}
            \dfrac{1}{\pi(a|s)}\left[\dfrac{1}{M}\sum\limits_{i=1}^M g_{k,i}\mathrm{1}(a_{\tau_i}=a)\right] - \dfrac{1}{M}\sum\limits_{i=1}^M g_{k,i} &\text{if}~M>0\\
            0 &\text{if}~M=0
        \end{cases}
    \end{align}

    Note the following,
    \begin{align}
        \begin{split}
           &\mathbf{E}\left[g_{k,i}\bigg|s_{\tau_i}=s, a_{\tau_i}=a\right]\\
           &=g(s, a) + \mathbf{E}\left[\sum_{t=\tau_i+1}^{\tau_i+N}g(s_t, a_t)\bigg| s_{\tau_i}=s, a_{\tau_i}=a\right]\\
           &=g(s, a) + \sum_{s'}P(s'|s, a)\mathbf{E}\left[\sum_{t=\tau_i+1}^{\tau_i+N}g(s_t, a_t)\bigg| s_{\tau_i+1}=s'\right]\\
           &=g(s, a) + \sum_{s'}P(s'|s, a)\left[\sum_{j=0}^{N-1}(P^{\pi})^j(s', \cdot)\right]^Tg^{\pi}\\
           &=g(s, a) + \sum_{s'}P(s'|s, a)\left[\sum_{j=0}^{N-1}(P^{\pi})^j(s', \cdot)-d^{\pi}\right]^Tg^{\pi} + N(d^{\pi})^Tg^{\pi}\\
           &\overset{(a)}{=}g(s, a) + \sum_{s'}P(s'|s, a)\left[\sum_{j=0}^{\infty}(P^{\pi})^j(s', \cdot)-d^{\pi}\right]^Tg^{\pi} + NJ_g^{\pi}\\
           &\hspace{1.7cm}-\underbrace{\sum_{s'}P(s'|s, a)\left[\sum_{j=N}^{\infty}(P^{\pi})^j(s', \cdot)-d^{\pi}\right]^Tg^{\pi}}_{\triangleq \mathrm{E}^{\pi}_T(s, a)}\\
           &\overset{(b)}{=} g(s, a) + \sum_{s'}P(s'|s, a)V_g^{\pi}(s') + NJ_g^{\pi}-\mathrm{E}^{\pi}_T(s, a)\\
           &\overset{(c)}{=} Q_g^{\pi}(s, a) + (N+1)J_g^{\pi} - \mathrm{E}^{\pi}_T(s, a)
        \end{split}
    \end{align}
    where $(a)$ follows from the definition of $J_g^{\pi}$ as given in $(\ref{eq_r_pi_theta})$, $(b)$ is an application of the definition of $V_g^{\pi}$ given in $(\ref{def_v_pi_theta_s})$, and $(c)$ follows from the Bellman equation. Define the following quantity.
    \begin{align}
        \label{def_error_1}
        \delta^{\pi}(s, T) \triangleq \sum_{t=N}^{\infty}\norm{(P^{\pi})^t({s,\cdot}) - d^{\pi}}_1 ~~\text{where} ~N=4t_{\mathrm{mix}}(\log_2 T)
    \end{align}
    
    Using Lemma \ref{lemma_aux_3}, we get $\delta^{\pi}(s, T)\leq \frac{1}{T^3}$ which implies, $|\mathrm{E}^{\pi}_T(s, a)|\leq \frac{1}{T^3}$. Observe the following relations.
    \begin{align}
    \label{eq_appndx_47}
        \begin{split}
            &\mathbf{E}\left[\left(\dfrac{1}{\pi(a|s)}g_{k,i}\mathrm{1}(a_{\tau_i}=a) - g_{k,i}\right)\bigg| s_{\tau_i}=s\right] \\
            &= \mathbf{E}\left[g_{k,i}\bigg| s_{\tau_i}=s, a_{\tau_i}=a\right] - \sum_{a'}\pi(a'|s)\mathbf{E}\left[g_{k,i}\bigg| s_{\tau_i}=s, a_{\tau_i}=a'\right]\\
            &=Q_g^{\pi}(s, a) + (N+1)J_g^{\pi} - \mathrm{E}^{\pi}_T(s, a) \\
            &\hspace{1.9cm}- \sum_{a'}\pi(a'|s)[Q^{\pi}(s, a) + (N+1)J_g^{\pi} - \mathrm{E}^{\pi}_T(s, a)]\\
            &=Q_g^{\pi}(s, a) - V_g^{\pi}(s)-\left[\mathrm{E}_T(s, a) - \sum_{a'}\pi(a'|s)\mathrm{E}_T^{\pi}(s, a')\right]\\
            &= A_g^{\pi}(s, a) -\Delta^{\pi}_T(s, a)
        \end{split}
    \end{align}
    where $\Delta^{\pi}_T(s, a)\triangleq\mathrm{E}_T(s, a) - \sum_{a'}\pi(a'|s)\mathrm{E}_T^{\pi}(s, a')$. Using the bound on $\mathrm{E}^{\pi}_T(s, a)$, we derive, $|\Delta_T^{\pi}(s, a)|\leq \frac{2}{T^3}$, which implies,
    \begin{align}
        \label{eq_appndx_48}
        \begin{split}
            &\left|\mathbf{E}\left[\left(\dfrac{1}{\pi(a|s)}g_{k,i}\mathrm{1}(a_{\tau_i}=a) - g_{k,i}\right)\bigg| s_{\tau_i}=s\right] - A_g^{\pi}(s, a)\right|\\
            &\hspace{1.5cm}\leq |\Delta_T^{\pi}(s, a)|\leq\dfrac{2}{T^3}
        \end{split}
    \end{align}

    Note that \eqref{eq_appndx_48} cannot directly bound the bias of $\hat{A}_g^{\pi}(s, a)$. This is because the random variable $M$ is correlated with the  variables $\{g_{k,i}\}_{i=1}^M$. To decorrelate them, imagine a CMDP where the state distribution resets to the stationary distribution, $d^{\pi}$ after exactly $N$ time steps since the completion of a sub-trajectory. In other words, if a sub-trajectory starts at $\tau_{i}$, and ends at $\tau_i+N$, then the system `rests' for additional $N$ steps before rejuvenating with the state distribution, $d^{\pi}$ at $\tau_i+2N$. Clearly, the wait time between the reset after the $(i-1)$th sub-trajectory and the start of the $i$th sub-trajectory is, $w_{i}=\tau_{i}-(\tau_{i-1}+2N)$, $i>1$. Let $w_1$ be the difference between the start time of the $k$th epoch and the start time of the first sub-trajectory. Observe that,

    $(a)$ $w_1$ only depends on the initial state, $s_{(k-1)H}$ and the induced transition function, $P^{\pi}$,

    $(b)$ $w_i$, where $i>1$, depends on the stationary distribution, $d^{\pi}$, and the induced transition function, $P^{\pi}$,

    $(c)$ $M$ only depends on $\{w_1, w_2, \cdots\}$ as other segments of the epoch have fixed length, $2N$.

    Clearly, in this imaginary CMDP, the sequence, $\{w_1, w_2, \cdots\}$, and hence, $M$ is independent of $\{g_{k,1}, g_{k, 2}, \cdots\}$. Let, $\mathbf{E}'$ denote the expectation operation and $\mathrm{Pr}'$ denote the probability of events in this imaginary system. Define the following.
    \begin{align}
        \label{def_delta_i}
        \Delta_i \triangleq \dfrac{g_{k,i}\mathrm{1}(a_{\tau_i}=a)}{\pi(a|s)} - g_{k,i} - A_g^{\pi}(s, a) + \Delta^{\pi}_T(s, a)
    \end{align}
    where $\Delta^{\pi}_T(s, a)$ is defined via $(\ref{eq_appndx_47})$. Note that we have suppressed the dependence on $T$, $s, a$, and $\pi$ while defining $\Delta_i$ to remove clutter. Using $(\ref{eq_appndx_47})$, one can write $ \mathbf{E}'\left[\Delta_i(s, a)|\{w_i\}\right]=0$. Moreover, 
    \begin{align}
        \label{eq_appndx_50}
        \begin{split}
            &\mathbf{E}'\left[\left(\hat{A}_g^{\pi}(s, a) - A_g^{\pi}(s, a)\right)^2\right]\\ 
            &= \mathbf{E}'\left[\left(\hat{A}_g^{\pi}(s, a) - A_g^{\pi}(s, a)\right)^2\bigg| M>0\right]\times \mathrm{Pr}'(M>0) \\
            & \hspace{1cm}+ \left(A_g^{\pi}(s, a)\right)^2\times \mathrm{Pr}'(M=0)\\
            &= \mathbf{E}'\left[\left(\dfrac{1}{M}\sum_{i=1}^M\Delta_i - \Delta_T^{\pi}(s, a)\right)^2\bigg| M>0\right]\times \mathrm{Pr}'(M>0) \\
            &\hspace{1cm}+ \left(A_g^{\pi}(s, a)\right)^2\times \mathrm{Pr}'(M=0)\\
            & \overset{}{\leq} 2\mathbf{E}_{\{w_i\}}'\left[\mathbf{E}'\left[\left(\dfrac{1}{M}\sum_{i=1}^M\Delta_i \right)^2\bigg| \{w_i\}\right]\bigg| w_1\leq H-N\right]\times \mathrm{Pr}'(w_1\leq H-N) \\
            &\hspace{1cm} + 2\left(\Delta_T^{\pi}(s, a)\right)^2+\left(A_g^{\pi}(s, a)\right)^2\times \mathrm{Pr}'(M=0)\\
            &\overset{(a)}{\leq} 2\mathbf{E}_{\{w_i\}}'\left[\dfrac{1}{M^2}\sum_{i=1}^M \mathbf{E}'\left[\Delta_i^2\big|\{w_i\}\right]\bigg| w_1\leq H-N\right]\times \mathrm{Pr}'(w_1\leq H-N) \\
            &\hspace{1cm}+ \dfrac{8}{T^6} +\left(A_g^{\pi}(s, a)\right)^2\times \mathrm{Pr}'(M=0)\\
        \end{split}
    \end{align}
    where $(a)$ uses the bound $|\Delta_T^{\pi}(s, a)|\leq \frac{2}{T^3}$ derived in $(\ref{eq_appndx_48})$, and the fact that $\{\Delta_i\}$ are zero mean independent random variables conditioned on $\{w_i\}$. Note that $|g_{k,i}|\leq N$ almost surely, $|A_g^{\pi}(s, a)|\leq \mathcal{O}(t_{\mathrm{mix}})$ via Lemma \ref{lemma_aux_2}, and $|\Delta^{\pi}_T(s, a)|\leq \frac{2}{T^3}$ as shown in $(\ref{eq_appndx_48})$. Combining, we get, $\mathbf{E}'[|\Delta_i|^2\big|\{w_i\}]\leq \mathcal{O}(N^2/\pi(a|s))$ (see the definition of $\Delta_i$ in (\ref{def_delta_i})). Invoking this bound into $(\ref{eq_appndx_50})$, we get the following result.
    \begin{align}
        \label{eq_appndx_51_}
        \begin{split}
            \mathbf{E}'&\left[\left(\hat{A}_g^{\pi}(s, a) - A_g^{\pi}(s, a)\right)^2\right]\leq 2\mathbf{E}'\left[\dfrac{1}{M}\bigg|w_1\leq H-N\right]\mathcal{O}\left(\dfrac{N^2}{\pi(a|s)}\right)\\
            &\hspace{3cm}+\dfrac{8}{T^6}+\mathcal{O}(t_{\mathrm{mix}}^2)\times \mathrm{Pr}'(w_1>H-N)
        \end{split}
    \end{align}

    We use Lemma \ref{lemma_aux_4} to bound the following violation probability.
    \begin{align}
        \label{eq_appndx_52_}
        \begin{split}
            \mathrm{Pr}'(w_1>H-N)&\leq \left(1-\dfrac{3d^{\pi}(s)}{4}\right)^{4t_{\mathrm{hit}}T^{\xi}(\log T)-1}\\
            &\overset{(a)}{\leq} \left(1-\dfrac{3d^{\pi}(s)}{4}\right)^{\dfrac{4}{d^{\pi}(s)}(\log T)}\leq \dfrac{1}{T^3}
        \end{split}
    \end{align}
    where $(a)$ follows from the fact that $4t_{\mathrm{hit}}T^{\xi}(\log_2 T) - 1 \geq \frac{4}{d^{\pi}(s)}\log_2 T$ for sufficiently large $T$. Finally, if $M<M_0$, where $M_0$ is defined as,
    \begin{align}
        M_0\triangleq \dfrac{H-N}{2N+ \dfrac{4N\log T}{d^{\pi}(s)}}
    \end{align}
    then there exists at least one $w_i$ that exceeds $4N\log_2 T/d^{\pi}(s)$, which can happen with the following maximum probability (Lemma \ref{lemma_aux_4}).
    \begin{align}
        \mathrm{Pr}'\left(M<M_0\right) \leq \left(1-\dfrac{3d^{\pi}(s)}{4}\right)^{\frac{4\log T}{d^{\pi(s)}}}\leq \dfrac{1}{T^3}
    \end{align}

    The above probability bound can be used to obtain the following,
    \begin{align}
        \label{eq_appndx_55_}
        \begin{split}
            \mathbf{E}'\left[\dfrac{1}{M}\bigg| M>0\right]&=\dfrac{\sum_{m=1}^{\infty}\dfrac{1}{m}\mathrm{Pr}'(M=m)}{\mathrm{Pr}'(M>0)}\\
            &\leq \dfrac{1\times \mathrm{Pr}'(M\leq M_0)+\dfrac{1}{M_0}\mathrm{Pr}'(M>M_0)}{\mathrm{Pr}'(M>0)}\\
            &\leq  \dfrac{\dfrac{1}{T^3}+\dfrac{2N+\dfrac{4N \log T}{d^{\pi}(s)}}{H-N}}{1-\dfrac{1}{T^3}}\leq \mathcal{O}\left(\dfrac{N\log T}{H d^{\pi}(s)}\right)
        \end{split}
    \end{align}

    Injecting $(\ref{eq_appndx_52_})$ and $(\ref{eq_appndx_55_})$ into $(\ref{eq_appndx_51_})$, we finally obtain,
    \begin{align}
        \label{eq_appndx_56_}
        \begin{split}
            \mathbf{E}'&\left[\left(\hat{A}_g^{\pi}(s, a) - A_g^{\pi}(s, a)\right)^2\right]\leq \mathcal{O}\left(\dfrac{N^3\log T}{H d^{\pi}(s)\pi(a|s)}\right)\\
            &\hspace{2cm}=\mathcal{O}\left(\dfrac{N^3t_{\mathrm{hit}}\log T}{H \pi(a|s)}\right)=\mathcal{O}\left(\dfrac{t^2_{\mathrm{mix}}(\log T)^2}{T^{\xi}\pi(a|s)}\right)
        \end{split}
    \end{align}

    $(\ref{eq_appndx_56_})$ shows that our desired inequality is satisfied in the imaginary system. We now need a mechanism to translate this result to our CMDP. Observe that, one can express $(\hat{A}_g^{\pi}(s, a)-A_g^{\pi}(s, a))^2=f(X)$ where $X=(M, \tau_1, \mathcal{T}_1, \cdots, \tau_M, \mathcal{T}_M)$, and $\mathcal{T}_i = (a_{\tau_i}, s_{\tau_i+1}, a_{\tau_i+1}, \cdots, s_{\tau_i+N}, a_{\tau_i+N})$. This leads to the following inequality.
    \begin{align}
        \label{eq_appndx_57_}
        \dfrac{\mathbf{E}[f(X)]}{\mathbf{E}'[f(X)]} = \dfrac{\sum_{X} f(X)\mathrm{Pr}(X)}{\sum_{X} f(X)\mathrm{Pr}'(X)}\leq \max_{X}\dfrac{\mathrm{Pr}(X)}{\mathrm{Pr'}(X)}
    \end{align}

    The above inequality uses the non-negativity of $f(\cdot)$. Observe that, 
    \begin{align}
        \begin{split}
            \mathrm{Pr}(X) = &\mathrm{Pr}(\tau_1)\times \mathrm{Pr}(\mathcal{T}_1|\tau_1)\times \mathrm{Pr}(\tau_2|\tau_1, \mathcal{T}_1)\times \mathrm{Pr}(\mathcal{T}_2|\tau_2) \times \cdots \\
            &\times \mathrm{Pr}(\tau_M|\tau_{M-1}, \mathcal{T}_{M-1})\times \mathrm{Pr}(\mathcal{T}_M|\tau_M)\\
            &\times \mathrm{Pr}(s_t\neq s, \forall t\in[\tau_M+2N, kH-N]|\tau_M, \mathcal{T}_M),
        \end{split}\\
        \begin{split}
            \mathrm{Pr}'(X) = &\mathrm{Pr}(\tau_1)\times \mathrm{Pr}(\mathcal{T}_1|\tau_1)\times \mathrm{Pr}'(\tau_2|\tau_1, \mathcal{T}_1)\times \mathrm{Pr}(\mathcal{T}_2|\tau_2) \times \cdots \\
            &\times \mathrm{Pr}'(\tau_M|\tau_{M-1}, \mathcal{T}_{M-1})\times \mathrm{Pr}(\mathcal{T}_M|\tau_M)\\
            &\times \mathrm{Pr}(s_t\neq s, \forall t\in[\tau_M+2N, kH-N]|\tau_M, \mathcal{T}_M),
        \end{split}
    \end{align}

    The difference between $\mathrm{Pr}(X)$ and $\mathrm{Pr}'(X)$ arises because $\mathrm{Pr}(\tau_{i+1}|\tau_i, \mathcal{T}_i)\neq \mathrm{Pr}'(\tau_{i+1}|\tau_i, \mathcal{T}_i)$, $\forall i\in\{1, \cdots, M-1\}$. Note that the ratio of these two terms can be bounded as follows,
    \begin{align}
        \begin{split}
            \dfrac{\mathrm{Pr}(\tau_{i+1}|\tau_i, \mathcal{T}_i)}{\mathrm{Pr}'(\tau_{i+1}|\tau_i, \mathcal{T}_i)}
            &\leq \max_{s'}\dfrac{\mathrm{Pr}(s_{\tau_i+2N}=s'|\tau_i, \mathcal{T}_i)}{\mathrm{Pr}'(s_{\tau_i+2N}=s'|\tau_i, \mathcal{T}_i)}\\
            &=\max_{s'}1+\dfrac{\mathrm{Pr}(s_{\tau_i+2N}=s'|\tau_i, \mathcal{T}_i)-d^{\pi}(s')}{d^{\pi}(s')}\\
            &\overset{(a)}{\leq} \max_{s'}1+\dfrac{1}{T^3d^{\pi}(s')}\leq 1+\dfrac{t_{\mathrm{hit}}}{T^3}\leq 1+\dfrac{1}{T^2}
        \end{split}
    \end{align}
    where $(a)$ is a consequence of Lemma \ref{lemma_aux_3}. We have,
    \begin{align}
        \label{eq_appndx_61_}
        \dfrac{\mathrm{Pr}(X)}{\mathrm{Pr}'(X)}\leq \left(1+\dfrac{1}{T^2}\right)^M\leq e^{\frac{M}{T^2}}\overset{(a)}{\leq} e^{\frac{1}{T}}\leq \mathcal{O}\left(1+\dfrac{1}{T}\right) 
    \end{align}
    where $(a)$ uses $M\leq T$. Combining $(\ref{eq_appndx_57_})$ and $(\ref{eq_appndx_61_})$, we get,
    \begin{align}
        \begin{split}
            &\mathbf{E}\left[\left(\hat{A}_g^{\pi}(s, a) - A_g^{\pi}(s, a)\right)^2\right]\\
            &\leq \mathcal{O}\left(1+\dfrac{1}{T}\right)\mathbf{E}'\left[\left(\hat{A}_g^{\pi}(s, a) - A_g^{\pi}(s, a)\right)^2\right]\overset{(a)}{\leq} \mathcal{O}\left(\dfrac{t^2_{\mathrm{mix}}(\log T)^2}{T^{\xi}\pi(a|s)}\right)
        \end{split}
    \end{align}
    where $(a)$ follows from $(\ref{eq_appndx_56_})$. Using the definition of $A_{\mathrm{L},\lambda}$, we get,
    \begin{equation}
        \begin{aligned}
            &\mathbf{E}\left[\left(\hat{A}_{\mathrm{L},\lambda}^{\pi}(s, a) - A_{\mathrm{L},\lambda}^{\pi}(s, a)\right)^2\right]\\
            &=\mathbf{E}\left[\left((\hat{A}_r^{\pi}(s, a) - A_r^{\pi}(s, a))-\lambda(\hat{A}_c^{\pi}(s, a) - A_c^{\pi}(s, a))\right)^2\right]\\
            &\leq 2 \mathbf{E}\left[\left(\hat{A}_r^{\pi}(s, a) - A_r^{\pi}(s, a)\right)^2\right]+ 2\lambda^2\mathbf{E}\left[\left(\hat{A}_c^{\pi}(s, a) - A_c^{\pi}(s, a)\right)^2\right]\\
            &\leq\mathcal{O}\left(\dfrac{t^2_{\mathrm{mix}}(\log T)^2}{\delta^2T^{\xi}\pi(a|s)}\right)
        \end{aligned}
    \end{equation}
    This completes the proof.
\end{proof}


\section{Global Convergence Analysis}
\label{sec_convergence}
This section first shows that the sequence $\{\theta_k, \lambda_k\}_{k=1}^K$ produced by Algorithm \ref{alg:PG_MAG} is such that their associated Lagrange sequence $\{J_{\mathrm{L}}(\theta_k, \lambda_k)\}_{k=1}^{\infty}$ converges globally. By expanding the Lagrange function, we then exhibit convergence of each of its components $\{J_g(\theta_k, \lambda_k)\}_{k=1}^K$, $g\in\{r,c\}$. This is later used for regret and constraint violation analysis. Before delving into the details, we would like to state a few necessary assumptions.
\begin{assumption}
    \label{ass_score}
    The score function (expressed below) is $G$-Lipschitz and $B$-smooth. Specifically, $\forall \theta, \theta_1,\theta_2 \in\mathbb{R}^{\mathrm{d}}$, $\forall (s, a)$,
    \begin{equation*}
	\begin{aligned}
            \Vert \nabla_\theta\log\pi_\theta(a\vert s)\Vert&\leq G,\\
            \Vert \nabla_\theta\log\pi_{\theta_1}(a\vert s)-\nabla_\theta\log\pi_{\theta_2}(a\vert s)\Vert&\leq B\Vert \theta_1-\theta_2\Vert\quad
	\end{aligned}
    \end{equation*}
\end{assumption}

\begin{remark}
    The Lipschitz and smoothness properties of the score function are commonly assumed for policy gradient analyses \citep{Alekh2020, Mengdi2021, liu2020improved}. These assumptions can be verified for simple parameterization class such as Gaussian policies.  
\end{remark}

Combining Assumption \ref{ass_score} with Lemma \ref{lemma_good_estimator} and using the gradient estimator as given in $(\ref{eq_grad_estimate})$, one can deduce the following result.

\begin{lemma}
    \label{lemma_grad_est_bias}
    The following inequality holds $\forall k$ provided that assumptions \ref{ass_1} and \ref{ass_score} are true.
    \begin{align}
        \mathbf{E}\left[\norm{\omega_k-\nabla_{\theta}J_{\mathrm{L}}(\theta_k,\lambda_k)}^2\right]\leq \tilde{\mathcal{O}}\left(\dfrac{AG^2t_{\mathrm{mix}}^2}{\delta^2 T^{\xi}}\right)
    \end{align}
\end{lemma}

\begin{proof} 
    Note the expression of the gradient estimator from \eqref{eq_grad_estimate}. Define:
    \begin{align}
        &\bar{\omega}_k= \dfrac{1}{H}\sum_{t=t_k}^{t_{k+1}-1}A_{\mathrm{L},\lambda_k}^{\pi_{\theta_k}}(s_{t}, a_{t})\nabla_{\theta}\log \pi_{\theta_k}(a_{t}|s_{t})
    \end{align}
 
    Recall the expression of the true gradient in Lemma \ref{lemma_grad_compute}. Assumption \ref{ass_score} and Lemma \ref{lemma_aux_3}, establish that $|A_{\mathrm{L},\lambda_k}^{\pi_{\theta_k}}(s, a)\nabla_{\theta}\log\pi_{\theta_k}(a|s)|\leq \mathcal{O}(\frac{1}{\delta} t_{\mathrm{mix}}G)$, $\forall (s, a)\in \mathcal{S}\times \mathcal{A}$ which implies    $|\nabla_{\theta}J_{\mathrm{L}}(\theta_k, \lambda_k)|\leq \mathcal{O}(\frac{1}{\delta} t_{\mathrm{mix}}G)$. Applying Lemma \ref{lemma_aux_6}, we, thus, arrive at,
    \begin{align}
        \label{eq_appndx_67_}
        \begin{split}
            \mathbf{E}\left[\norm{\bar{\omega}_k-\nabla_{\theta}J_{\mathrm{L}}(\theta_k, \lambda_k)}^2\right]&\leq \mathcal{O}\left(\frac{1}{\delta^2}G^{2}t^2_{\mathrm{mix}}\log T \times \dfrac{t_{\mathrm{mix}}\log T}{H}\right)\\
            &=\mathcal{\Tilde{O}}\left(\dfrac{G^2t_{\mathrm{mix}}^2}{\delta^2t_{\mathrm{hit}}T^\xi}\right)
        \end{split}
    \end{align}
	
    Finally, the term, $\mathbf{E}\norm{\omega_k-\bar{\omega}_k}^2$ can be bounded as follows.
    \begin{align}
        \label{eq_appndx_68_}
        \begin{split}
            &\mathbf{E}\norm{\omega_k-\bar{\omega}_k}^2\\
            &=\mathbf{E}\left[\bigg\Vert\dfrac{1}{H}\sum_{t=t_k}^{t_{k+1}-1}\hat{A}_{\mathrm{L},\lambda_k}^{\pi_{\theta_k}}(s_{t}, a_{t})\nabla_{\theta}\log \pi_{\theta_k}(a_{t}|s_{t})\right.\\
            &\hspace{2cm} \left. -\dfrac{1}{H}\sum_{t=t_k}^{t_{k+1}-1}A_{\mathrm{L},\lambda_k}^{\pi_{\theta_k}}(s_{t}, a_{t})\nabla_{\theta}\log \pi_{\theta_k}(a_{t}|s_{t})\bigg\Vert^2\right]\\
            &\overset{(a)}{\leq} \dfrac{G^2}{H}\sum_{t=t_k}^{t_{k+1}-1}\mathbf{E}\left[\left(\hat{A}_{\mathrm{L},\lambda_k}^{\pi_{\theta_k}}(s_t, a_t)-A_{\mathrm{L},\lambda_k}^{\pi_{\theta_k}}(s_t, a_t)\right)^2\right]\\
            &\leq \dfrac{G^2}{H}\sum_{t=t_k}^{t_{k+1}-1}\mathbf{E}\left[\sum_{a}\pi_{\theta_k}(a|s_t)\mathbf{E}\left[\left(\hat{A}_{\mathrm{L},\lambda_k}^{\pi_{\theta_k}}(s_t, a)-A_{\mathrm{L},\lambda_k}^{\pi_{\theta_k}}(s_t, a)\right)^2\bigg| s_t\right]\right]\\
            &\overset{(b)}{\leq}\mathcal{O}\left(\dfrac{AG^2t^2_{\mathrm{mix}}(\log T)^2}{\delta^2 T^\xi}\right)
	\end{split}
    \end{align}
    where $(a)$ follows from Assumption \ref{ass_score} and Jensen's inequality while $(b)$ results from Lemma \ref{lemma_good_estimator}. Using $(\ref{eq_appndx_67_})$, $(\ref{eq_appndx_68_})$, we conclude the result.
\end{proof}

Lemma \ref{lemma_grad_est_bias} claims that the gradient estimation error can be bounded as $\tilde{\mathcal{O}}(T^{-\xi})$. This result will be later used in proving the global convergence of our algorithm.

\begin{assumption}
    \label{ass_transfer_error}
    Let the transferred compatible function approximation error be defined as follows.
    \begin{align}
        \label{eq:transfer_error}
        L_{d^{\pi^*},\pi^*}(\omega^*_{\theta, \lambda},\theta, \lambda)=\mathbf{E}_{(s, a)\sim \nu^{\pi^*}}\bigg[\bigg(\nabla_\theta\log\pi_{\theta}(a\vert s)\cdot\omega^*_{\theta, \lambda}-A_{\mathrm{L}, \lambda}^{\pi_\theta}(s,a)\bigg)^2\bigg]
    \end{align}
    where $\pi^*=\pi_{\theta^*}$, $\theta^*$ solves $\eqref{eq:def_constrained_optimization}$, $\nu^{\pi}(s, a)\triangleq d^{\pi}(s)\pi(a|s)$, $\forall s, a, \pi$ and
    \begin{align}
        \label{eq:NPG_direction}
        \omega^*_{\theta, \lambda}=\arg\min_{\omega\in\mathbb{R}^{\mathrm{d}}}~\mathbf{E}_{(s, a)\sim \nu^{\pi_{\theta}}}\bigg[\bigg(\nabla_\theta\log\pi_{\theta}(a\vert s)\cdot\omega-A_{\mathrm{L}, \lambda}^{\pi_{\theta}}(s,a)\bigg)^2\bigg]
    \end{align}
    We presume that $L_{d^{\pi^*},\pi^*}(\omega^*_{\theta, \lambda},\theta, \lambda)\leq \epsilon_{\mathrm{bias}}$ for arbitrary $\lambda\in (0, \frac{2}{\delta}]$ and $\theta\in\mathbb{R}^{\mathrm{d}}$ where $\epsilon_{\mathrm{bias}}$ is a positive constant.
\end{assumption}

\begin{remark}
    The transferred compatible function approximation error quantifies the expressivity of the parameterized policy class. One can prove  that $\epsilon_{\mathrm{bias}}=0$ for softmax parameterization \citep{agarwal2021theory} and linear MDPs \citep{Chi2019}. If the policy class is restricted i.e., it does not contain all stochastic policies, $\epsilon_{\mathrm{bias}}$ is strictly positive. However, if the parameterization is done by a rich neural network, then $\epsilon_{bias}$ can be assumed to be negligible \citep{Lingxiao2019}. Such assumptions are common in the literature \citep{liu2020improved,agarwal2021theory}. 
\end{remark}

\begin{remark}
    Note that $\omega^*_{\theta, \lambda}$ defined in \eqref{eq:NPG_direction} can be re-written as,
    \begin{align*}
        \omega^*_{\theta, \lambda} = F_{\rho}(\theta)^{\dagger} \mathbf{E}_{s\sim d_\rho^{\pi_{\theta}}}\mathbf{E}_{a\sim\pi_{\theta}( s)}\left[\nabla_{\theta}\log\pi_{\theta}(a|s)A_{\mathrm{L},\lambda}^{\pi_{\theta}}(s, a)\right]
    \end{align*}
    where $\dagger$ is the Moore-Penrose pseudoinverse operation and $F_{\rho}(\theta)$ is the Fisher information matrix  defined as,
    \begin{align*}
        F_{\rho}(\theta) = \mathbf{E}_{s\sim d_\rho^{\pi_{\theta}}}\mathbf{E}_{a\sim\pi_{\theta}(\cdot\vert s)}[\nabla_{\theta}\log\pi_{\theta}(a|s)(\nabla_{\theta}\log\pi_{\theta}(a|s))^T]
    \end{align*}
\end{remark}

\begin{assumption}
    \label{ass_4}
    There exists a constant $\mu_F>0$ such that $F_{\rho}(\theta)-\mu_F I_{\mathrm{d}}$ is positive semidefinite where $I_{\mathrm{d}}$ is an identity matrix of dimension, $\mathrm{d}$.
\end{assumption}

Assumption \ref{ass_4} is also common in the policy gradient analysis \citep{liu2020improved}. This holds for Gaussian policies with a linearly parameterized mean. \cite{mondal2023mean} provides a concrete example of a class of policies obeying assumptions \ref{ass_score}-\ref{ass_4}. The Lagrange difference lemma stated below is an important result in proving the global convergence.
\begin{lemma}
    \label{lem_performance_diff}
    For any two policies $\pi_\theta$, $\pi_{\theta'}$, the following holds $\forall \lambda>0$.
    \begin{equation*}
        J_{\mathrm{L}}(\theta,\lambda)-J_{\mathrm{L}}(\theta',\lambda)=\mathbf{E}_{s\sim d^{\pi_\theta}}\mathbf{E}_{a\sim\pi_\theta( s)}\big[A_{\mathrm{L},\lambda}^{\pi_{\theta'}}(s,a)\big]
    \end{equation*}
\end{lemma}

\begin{proof}
    Using the Lemma \ref{lemma_aux_5}, it is obvious to see that
    \begin{equation}
        \begin{aligned}
            J_g^{\pi}-J_g^{\pi'}&=\sum_{s}\sum_{a}d^{\pi}(s)(\pi(a|s)-\pi'(a|s))Q_g^{\pi'}(s,a)\\
            &=\sum_{s}\sum_{a}d^{\pi}(s)\pi(a|s)Q_g^{\pi'}(s,a)-\sum_{s}d^{\pi}(s)V_g^{\pi'}(s)\\
            &=\sum_{s}\sum_{a}d^{\pi}(s)\pi(a|s)Q_g^{\pi'}(s,a)-\sum_{s}\sum_{a}d^{\pi}(s)\pi(a|s)V_g^{\pi'}(s)\\
            &=\sum_{s}\sum_{a}d^{\pi}(s)\pi(a|s)[Q_g^{\pi'}(s,a)-V_g^{\pi'}(s)]\\
            &=\mathbf{E}_{s\sim d^{\pi}}\mathbf{E}_{a\sim\pi(\cdot\vert s)}\big[A_g^{\pi'}(s,a)\big]
	\end{aligned}
    \end{equation}
    Since the above equation works for both reward and constraint, we can conclude the lemma by the definition of $J_{\mathrm{L}}(\cdot, \lambda)$ and $A_{\mathrm{L},\lambda}$.
\end{proof}

We now present a general framework for the convergence analysis. 

\begin{lemma}
    \label{lem_framework} 
    If the parameters $\{\theta_k, \lambda_k\}_{k=1}^K$ are updated via $\eqref{udpates_algorotihm}$ and assumptions \ref{ass_score}, \ref{ass_transfer_error}, and \ref{ass_4} hold, then the following holds for any $K$,
    \begin{equation}
        \label{eq:general_bound}
        \begin{split}
            &\frac{1}{K}\mathbf{E}\sum_{k=1}^{K}\bigg(J_{\mathrm{L}}(\theta^*, \lambda_k)-J_{\mathrm{L}}(\theta_k,\lambda_k)\bigg)\leq \sqrt{\epsilon_{\mathrm{bias}}} + \frac{G}{K}\sum_{k}^{K}\Vert(\omega_k-\omega^*_k)\Vert\\
            &+\frac{B\alpha}{2K}\sum_{k=1}^{K}\Vert\omega_k\Vert^2
            +\frac{1}{\alpha K}\mathbf{E}_{s\sim d_\rho^{\pi^*}}[KL(\pi^*(\cdot\vert s)\Vert\pi_{\theta_1}(\cdot\vert s))]
        \end{split}
    \end{equation}
    where $\omega^*_k\triangleq \omega^*_{\theta_k, \lambda_k}$, $\omega^*_{\theta_k,\lambda_k}$ is defined via \eqref{eq:NPG_direction}, and $\pi^*=\pi_{\theta^*}$ where $\theta^*$ is the optimal parameter solving $\eqref{eq:def_constrained_optimization}$.
\end{lemma}

\begin{proof}
    We start with the definition of KL divergence.
    \begin{equation}
        \begin{aligned}
            &\mathbf{E}_{s\sim d^{\pi^*}}[KL(\pi^*(\cdot\vert s)\Vert\pi_{\theta_k}(\cdot\vert s))-KL(\pi^*(\cdot\vert s)\Vert\pi_{\theta_{k+1}}(\cdot\vert s))]\\
            &=\mathbf{E}_{s\sim d^{\pi^*}}\mathbf{E}_{a\sim\pi^*(\cdot\vert s)}\bigg[\log\frac{\pi_{\theta_{k+1}(a\vert s)}}{\pi_{\theta_k}(a\vert s)}\bigg]\\
            &\overset{(a)}\geq\mathbf{E}_{(s, a)\sim \nu^{\pi^*}}[\nabla_\theta\log\pi_{\theta_k}(a\vert s)\cdot(\theta_{k+1}-\theta_k)]-\frac{B}{2}\Vert\theta_{k+1}-\theta_k\Vert^2\\
            &=\alpha\mathbf{E}_{(s, a)\sim \nu^{\pi^*}}[\nabla_{\theta}\log\pi_{\theta_k}(a\vert s)\cdot\omega_k]-\frac{B\alpha^2}{2}\Vert\omega_k\Vert^2\\
            &=\alpha\mathbf{E}_{(s, a)\sim \nu^{\pi^*}}[\nabla_\theta\log\pi_{\theta_k}(a\vert s)\cdot\omega^*_k]\\
            &\hspace{1cm}+\alpha\mathbf{E}_{(s, a)\sim \nu^{\pi^*}}[\nabla_\theta\log\pi_{\theta_k}(a\vert s)\cdot(\omega_k-\omega^*_k)]-\frac{B\alpha^2}{2}\Vert\omega_k\Vert^2\\
            &=\alpha[J_{\mathrm{L}}(\theta^*,\lambda_k)-J_{\mathrm{L}}(\theta_k,\lambda_k)]\\
            &\hspace{1cm}+\alpha\mathbf{E}_{(s, a)\sim \nu^{\pi^*}}[\nabla_\theta\log\pi_{\theta_k}(a\vert s)\cdot\omega^*_k]-\alpha[J_{\mathrm{L}}(\theta^*,\lambda_k)-J_{\mathrm{L}}(\theta_k,\lambda_k)]\\
            &\hspace{1cm}+\alpha\mathbf{E}_{(s, a)\sim \nu^{\pi^*}}[\nabla_\theta\log\pi_{\theta_k}(a\vert s)\cdot(\omega_k-\omega^*_k)]-\frac{B\alpha^2}{2}\Vert\omega_k\Vert^2\\		
            &\overset{(b)}=\alpha[J_{\mathrm{L}}(\theta^*,\lambda_k)-J_{\mathrm{L}}(\theta_k,\lambda_k)]\\
            &\hspace{1cm}+\alpha\mathbf{E}_{(s, a)\sim \nu^{\pi^*}}\bigg[\nabla_\theta\log\pi_{\theta_k}(a\vert s)\cdot\omega^*_k-A_{\mathrm{L},\lambda_k}^{\pi_{\theta_k}}(s,a)\bigg]\\
            &\hspace{1cm}+\alpha\mathbf{E}_{(s, a)\sim \nu^{\pi^*}}[\nabla_\theta\log\pi_{\theta_k}(a\vert s)\cdot(\omega_k-\omega^*_k)]-\frac{B\alpha^2}{2}\Vert\omega_k\Vert^2\\
            &\overset{(c)}\geq\alpha[J_{\mathrm{L}}(\theta^*,\lambda_k)-J_{\mathrm{L}}(\theta_k,\lambda_k)]\\
            &\hspace{1cm}-\alpha\sqrt{\mathbf{E}_{(s, a)\sim \nu^{\pi^*}}\bigg[\bigg(\nabla_\theta\log\pi_{\theta_k}(a\vert s)\cdot\omega^*_k-A_{L,\lambda_k}^{\pi_{\theta_k}}(s,a)\bigg)^2\bigg]}\\
            &\hspace{1cm}-\alpha\mathbf{E}_{(s, a)\sim \nu^{\pi^*}}\Vert\nabla_\theta\log\pi_{\theta_k}(a\vert s)\Vert_2\Vert(\omega_k-\omega^*_k)\Vert-\frac{B\alpha^2}{2}\Vert\omega_k\Vert^2\\
            &\overset{(d)}\geq\alpha[J_{\mathrm{L}}(\theta^*,\lambda_k)-J_{\mathrm{L}}(\theta_k,\lambda_k)]\\
            &\hspace{1cm}-\alpha\sqrt{\epsilon_{\mathrm{bias}}}-\alpha G\Vert(\omega_k-\omega^*_k)\Vert-\frac{B\alpha^2}{2}\Vert\omega_k\Vert^2\\
        \end{aligned}	
    \end{equation}
    where $(a)$ and $(b)$ follows from Assumption \ref{ass_score} and  Lemma \ref{lem_performance_diff} respectively. Step $(c)$ uses the convexity of the function $f(x)=x^2$. Finally, step $(d)$ comes from Assumption \ref{ass_transfer_error}. Rearranging items, we have
    \begin{equation}
        \begin{split}
            &J_{\mathrm{L}}(\theta^*,\lambda_k)-J_{\mathrm{L}}(\theta_k,\lambda_k)\leq \sqrt{\epsilon_{\mathrm{bias}}}+ G\Vert(\omega_k-\omega^*_k)\Vert+\frac{B\alpha}{2}\Vert\omega_k\Vert^2\\
            &+\frac{1}{\alpha}\mathbf{E}_{s\sim d_\rho^{\pi^*}}[KL(\pi^*(\cdot\vert s)\Vert\pi_{\theta_k}(\cdot\vert s))-KL(\pi^*(\cdot\vert s)\Vert\pi_{\theta_{k+1}}(\cdot\vert s))]
        \end{split}
    \end{equation}
    Summing from $k=1$ to $K$, using the non-negativity of KL divergence and dividing the resulting expression by $K$, we get the desired result.
\end{proof}

Lemma \ref{lem_framework} proves that the optimality error of the Lagrange sequence can be bounded by the average first-order and second-order norms of the intermediate gradients. Note the presence of the $\epsilon_{\mathrm{bias}}$ term in the result. If the policy class is severely restricted, the optimality bound loses its importance. Consider the expectation of the second term in \eqref{eq:general_bound}. Note the following set of inequalities.
\begin{equation}
    \label{eq_second_term_bound}
    \begin{split}
        &\bigg(\frac{1}{K}\sum_{k=1}^{K}\mathbf{E}\Vert\omega_k-\omega^*_k\Vert\bigg)^2\leq \frac{1}{K}\sum_{k=1}^{K}\mathbf{E}\bigg[\Vert\omega_k-\omega^*_k\Vert^2\bigg]\\
        &=\frac{1}{K}\sum_{k=1}^{K}\mathbf{E}\bigg[\Vert\omega_k-F_\rho(\theta_k)^\dagger\nabla_\theta J_{\mathrm{L}}(\theta_k,\lambda_k)\Vert^2\bigg]\\
        &\leq \frac{2}{K}\sum_{k=1}^{K}\Bigg\lbrace \mathbf{E}\bigg[\Vert\omega_k-\nabla_\theta J_{\mathrm{L}}(\theta_k,\lambda_k)\Vert^2\bigg]\\
        &\hspace{2cm}+ \mathbf{E}\bigg[\Vert \nabla_\theta J_{\mathrm{L}}(\theta_k,\lambda_k)- F_\rho(\theta_k)^\dagger\nabla_\theta J_{\mathrm{L}}(\theta_k,\lambda_k)\Vert^2\bigg]\Bigg\rbrace \\
        &\overset{(a)}{\leq} \frac{2}{K}\sum_{k=1}^{K}\mathbf{E}\bigg[\Vert \omega_k-\nabla_\theta J_{\mathrm{L}}(\theta_k,\lambda_k)\Vert^2\bigg]\\
        &\hspace{2cm}+\frac{2}{K}\sum_{k=1}^{K}\left(1+\dfrac{1}{\mu_F^2}\right)\mathbf{E}\bigg[\Vert \nabla_\theta J_{\mathrm{L}}(\theta_k,\lambda_k)\Vert^2\bigg]
    \end{split}
\end{equation}
where $(a)$ follows from Assumption \ref{ass_4}. The expectation of the third term in \eqref{eq:general_bound} can be bounded as follows.
\begin{equation}
    \label{eq_third_term_bound}
    \begin{split}
        \dfrac{1}{K}\sum_{k=1}^{K}\mathbf{E}\bigg[\Vert\omega_k\Vert^2\bigg] &\leq \dfrac{1}{K}\sum_{k=1}^{K}\mathbf{E}\left[\norm{\nabla_\theta J_{\mathrm{L}}(\theta_k,\lambda_k)}^2\right] \\
        &+ \dfrac{1}{K}\sum_{k=1}^{K}\mathbf{E}\bigg[\Vert \omega_k-\nabla_\theta J_{\mathrm{L}}(\theta_k,\lambda_k)\Vert^2\bigg] 
    \end{split}
\end{equation}

In both $\eqref{eq_second_term_bound}$, $\eqref{eq_third_term_bound}$,  $\mathbf{E}\norm{\omega_k-\nabla_{\theta}J_{\mathrm{L}}(\theta_k,\lambda_k)}^2$ is bounded by Lemma \ref{lemma_grad_est_bias}. To bound  $\mathbf{E}\norm{\nabla_\theta J_L(\theta_k,\lambda_k)}^2$, the following lemma is applied.

\begin{lemma} 
    \label{lemma:41ss}
    Let $J_{g}(\cdot)$ be $L$-smooth, $\forall g\in\{r, c\}$ and $\alpha = \frac{1}{4L(1+\frac{2}{\delta})}$. Then the following inequality holds.
    \begin{align}
       \label{eq_24_new}
        \dfrac{1}{K}\sum_{k=1}^{K} \norm{\nabla_\theta J_{\mathrm{L}}(\theta_k,\lambda_k)}^2 \leq \frac{288L}{\delta^2 K} + \dfrac{3}{K}\sum_{k=1}^{K}\Vert\nabla_\theta J_{\mathrm{L}}(\theta_k,\lambda_k) - \omega_k\Vert^2+\beta
    \end{align}
\end{lemma}
\begin{proof}
    By the $L$-smooth property of the objective function and constraint function, we know $J_{\mathrm{L}}(\cdot, \lambda)$ is a $L(1+\lambda)$-smooth function. Thus,
    \begin{align}
        \label{eq:3_chap_55}
        \begin{split}
            &J_{\mathrm{L}}(\theta_{k+1},\lambda_k)-J_{\mathrm{L}}(\theta_k,\lambda_k)\\
            &\geq \left<\nabla J_{\mathrm{L}}(\theta_k,\lambda_k),\theta_{k+1}-\theta_k\right>-\frac{L(1+\lambda_k)}{2}\Vert\theta_{k+1}-\theta_k\Vert^2\\
            &\overset{(a)} = \alpha \nabla J_{\mathrm{L}}(\theta_k,\lambda_k)^T \omega_k - \frac{L(1+\lambda_k) \alpha^2}{2} \norm{ \omega_k }^2 \\
            &= \alpha \norm{\nabla J_{\mathrm{L}}(\theta_k,\lambda_k)}^2 - \alpha \langle \nabla J_{\mathrm{L}}(\theta_k,\lambda_k) - \omega_k, \nabla J(\theta_k) \rangle \\
            &\hspace{1cm}- \frac{L(1+\lambda_k) \alpha^2}{2}\Vert \nabla J_{\mathrm{L}}(\theta_k,\lambda_k)-\omega_k-\nabla J_{\mathrm{L}}(\theta_k,\lambda_k)\Vert^2\\
            &\overset{(b)}\geq \alpha \norm{\nabla J_{\mathrm{L}}(\theta_k,\lambda_k)}^2 - \frac{\alpha}{2} \Vert\nabla J_{\mathrm{L}}(\theta_k,\lambda_k) - \omega_k\Vert^2\\
            &\hspace{1cm} -\frac{\alpha}{2}\Vert\nabla J_{\mathrm{L}}(\theta_k,\lambda_k)\Vert^2- L(1+\lambda_k)\alpha^2\Vert \nabla J_{\mathrm{L}}(\theta_k,\lambda_k)-\omega_k\Vert^2\\
            &\hspace{1cm}-L(1+\lambda_k)\alpha^2\Vert\nabla J_{\mathrm{L}}(\theta_k,\lambda_k)\Vert^2\\
            &= \left(\frac{\alpha}{2}-L(1+\lambda_k)\alpha^2\right) \norm{\nabla J_{\mathrm{L}}(\theta_k,\lambda_k)}^2\\
            &\hspace{1cm} - \left(\frac{\alpha}{2}+L(1+\lambda_k)\alpha^2\right) \Vert\nabla J_{\mathrm{L}}(\theta_k,\lambda_k) - \omega_k\Vert^2
        \end{split}
    \end{align}
    where step (a) holds from the fact that $\theta_{t+1} = \theta_k + \alpha \omega_k$ and step (b) holds due to the Cauchy-Schwarz inequality. Note that,
    \begin{align}
        \label{eq:3_chap_56}
        J_{\mathrm{L}}(\theta_{k+1}, \lambda_{k+1}) - J_{\mathrm{L}}(\theta_{k+1}, \lambda_k) \overset{(a)}{=} (\lambda_{k}-\lambda_{k+1})J_c(\theta_{k+1}) \overset{(b)}{\geq} -\beta
    \end{align}
    where (a) holds by the definition of $J_{\mathrm{L}}(\theta,\lambda)$ and step (b) is true because $|J_c(\theta)|\leq 1,\forall \theta$ and $|\lambda_{k+1}-\lambda_k|\leq \beta|\hat{J}_c(\theta_k)|\leq \beta$ where the last inequality uses the fact that $|\hat{J}_c(\theta_k)|\leq 1$. Adding \eqref{eq:3_chap_55} and \eqref{eq:3_chap_56}, we get,
    \begin{equation}
	\begin{aligned}
            &J_{\mathrm{L}}(\theta_{k+1},\lambda_{k+1})-J_{\mathrm{L}}(\theta_k,\lambda_k)\geq  \left(\frac{\alpha}{2}-L(1+\lambda_k)\alpha^2\right) \norm{\nabla J_{\mathrm{L}}(\theta_k,\lambda_k)}^2\\
            &\quad - \left(\frac{\alpha}{2}+L(1+\lambda_k)\alpha^2\right) \Vert\nabla J_L(\theta_k,\lambda_k) - \omega_k\Vert^2 -\beta 
        \end{aligned}
    \end{equation}
    Summing over $k\in\{ 1, \cdots, K\}$, we have,
    \begin{equation}
	\begin{aligned}
            &J_{\mathrm{L}}(\theta_{K+1},\lambda_{K+1})-J_{\mathrm{L}}(\theta_1,\lambda_1)\\
            &\geq -\beta K + \sum_{k=1}^{K}\left(\frac{\alpha}{2}-L(1+\lambda_k)\alpha^2\right) \norm{\nabla J_{\mathrm{L}}(\theta_k,\lambda_k)}^2\\
            &\quad - \sum_{k=1}^{K}\left(\frac{\alpha}{2}+L(1+\lambda_k)\alpha^2\right) \Vert\nabla J_{\mathrm{L}}(\theta_k,\lambda_k) - \omega_k\Vert^2\\
        \end{aligned}
    \end{equation}
    Rearranging the terms and using $\lambda_k\leq \frac{2}{\delta}$ yields, 
    \begin{eqnarray}
        &&\sum_{k=1}^{K}\norm{\nabla J_{\mathrm{L}}(\theta_k,\lambda_k)}^2 \leq \frac{1}{{\frac{\alpha}{2}-L(1+\frac{2}{\delta})\alpha^2} }\bigg[J_{\mathrm{L}}(\theta_{K+1},\lambda_{K+1})-J_{\mathrm{L}}(\theta_1,\lambda_1)\nonumber\\
        && + \beta K + \left(\frac{\alpha}{2}+L\left(1+\frac{2}{\delta}\right)\alpha^2\right) \sum_{k=1}^{K}  \Vert\nabla J_{\mathrm{L}}(\theta_k,\lambda_k) - \omega_k\Vert^2\bigg]
    \end{eqnarray}
    Choosing $\alpha = \frac{1}{4L(1+\frac{2}{\delta})}$ and dividing both sides by $K$, we conclude,
    \begin{align}
        &\frac{1}{K}\sum_{k=1}^{K}\norm{\nabla J_{\mathrm{L}}(\theta_k,\lambda_k)}^2 \leq \frac{16L(1+\frac{2}{\delta})}{K}\left[J_{\mathrm{L}}(\theta_{K+1},\lambda_{K+1})-J_{\mathrm{L}}(\theta_1,\lambda_1)\right]\nonumber\\
        &\hspace{2cm}+\dfrac{3}{K}\sum_{k=1}^{K}\Vert\nabla J_{\mathrm{L}}(\theta_k,\lambda_k) - \omega_k\Vert^2 +\beta
    \end{align}
    Recall that $|J_{\mathrm{L}}(\theta,\lambda)|\leq 1+\lambda\leq 1+\frac{2}{\delta}\le \frac{3}{\delta},\forall \theta,\lambda$, thus
    \begin{align}
        \frac{1}{K}\sum_{k=1}^{K}\norm{\nabla J_{\mathrm{L}}(\theta_k,\lambda_k)}^2 \leq \frac{288 L}{\delta^2 K}+  \dfrac{3}{K}\sum_{k=1}^{K}\Vert\nabla J_{\mathrm{L}}(\theta_k,\lambda_k) - \omega_k\Vert^2 + \beta
    \end{align}
    This concludes the proof of Lemma \ref{lemma:41ss}.
\end{proof}
Note the presence of $\beta$ in \eqref{eq_24_new}. Evidently, to ensure convergence, $\beta$ must be a function of $T$. 
Invoking Lemma \ref{lemma_grad_est_bias}, we get the following relation under the same set of assumptions and the choice of parameters as in Lemma \ref{lemma:41ss}. 
\begin{align}
    \label{eq_33}
    \dfrac{1}{K}\sum_{k=1}^{K} \mathbf{E}\norm{\nabla_\theta J_{\mathrm{L}}(\theta_k,\lambda_k)}^2 \leq \tilde{\mathcal{O}}\left(\dfrac{AG^2t_{\mathrm{mix}}^2}{\delta^2 T^{\xi}}\right)+\tilde{\mathcal{O}}\left(\dfrac{Lt_{\mathrm{mix}}t_{\mathrm{hit}}}{\delta^2 T^{1-\xi}}\right)+\beta
\end{align}
Applying Lemma \ref{lemma_grad_est_bias} and \eqref{eq_33} in \eqref{eq_third_term_bound}, we arrive at,
\begin{align}
    \label{eq_34}
    \dfrac{1}{K}\sum_{k=1}^{K}\mathbf{E}\bigg[\Vert\omega_k\Vert^2\bigg]&\leq \tilde{\mathcal{O}}\left(\dfrac{AG^2t_{\mathrm{mix}}^2}{\delta^2 T^{\xi}}+\dfrac{Lt_{\mathrm{mix}}t_{\mathrm{hit}}}{\delta^2 T^{1-\xi}}\right)+\beta
\end{align}

Similarly, using $(\ref{eq_second_term_bound})$, we deduce the following.
\begin{align}
    \label{eq_35}
    \begin{split}
        \frac{1}{K}&\sum_{k=1}^{K}\mathbf{E}\Vert\omega_k-\omega^*_k\Vert \leq \left(1+\dfrac{1}{\mu_F}\right)\sqrt{\beta} \\
        &+ \left(1+\dfrac{1}{\mu_F}\right)\tilde{\mathcal{O}}\left(\dfrac{\sqrt{A}Gt_{\mathrm{mix}}}{\delta T^{\xi/2}}+\dfrac{\sqrt{Lt_{\mathrm{mix}}t_{\mathrm{hit}}}}{\delta T^{(1-\xi)/2}}\right)
    \end{split}
\end{align}

Inequalities \eqref{eq_34} and \eqref{eq_35} lead to the following global convergence of the Lagrange function.

\begin{lemma}
    \label{Lemma_global_Lagrange}
    Let the parameters $\{\theta_k\}_{k=1}^{K}$ be as stated in Lemma \ref{lem_framework}. If assumptions \ref{ass_1}$-$\ref{ass_4} are true,  $\{J_{g}(\cdot)\}_{g\in\{r,c\}}$ are $L$-smooth functions, $\alpha=\frac{1}{4L(1+\frac{2}{\delta})}$, $K=\frac{T}{H}$, and $H=16t_{\mathrm{mix}}t_{\mathrm{hit}}T^{\xi}(\log_2 T)^2$, then the following inequality holds for sufficiently large $T$.
    \begin{equation*}
        \begin{split}
            &\frac{1}{K}\sum_{k=1}^{K} \mathbf{E}\bigg(J_{\mathrm{L}}(\pi^*,\lambda_k)-J_{\mathrm{L}}(\theta_k,\lambda_k)\bigg)\leq \sqrt{\epsilon_{\mathrm{bias}}}\\
            & + G\left(1+\dfrac{1}{\mu_F}\right)\tilde{\mathcal{O}}\left(\sqrt{\beta}+\dfrac{\sqrt{A}Gt_{\mathrm{mix}}}{\delta T^{\xi/2}}+\dfrac{\sqrt{Lt_{\mathrm{mix}}t_{\mathrm{hit}}}}{\delta T^{(1-\xi)/2}}\right)\\
            &+\dfrac{B}{L}\tilde{\mathcal{O}}\left(\dfrac{AG^2t_{\mathrm{mix}}^2}{\delta^2 T^{\xi}}+\dfrac{Lt_{\mathrm{mix}}t_{\mathrm{hit}}}{\delta^2 T^{1-\xi}}+\beta\right)\\
            &+\mathcal{\tilde{O}}\bigg(\frac{Lt_{\mathrm{mix}}t_{\mathrm{hit}}\mathbf{E}_{s\sim d^{\pi^*}}[KL(\pi^*(\cdot\vert s)\Vert\pi_{\theta_1}(\cdot\vert s))]}{T^{1-\xi}\delta }\bigg)
        \end{split}
    \end{equation*}
\end{lemma}

Lemma \ref{Lemma_global_Lagrange} establishes that the average difference between $J_{\mathrm{L}}(\pi^*, \lambda_k)$ and $J_{\mathrm{L}}(\theta_k, \lambda_k)$ is $\tilde{\mathcal{O}}(\sqrt{\beta}+T^{-\xi/2}+T^{-(1-\xi)/2})$. Expanding the function, $J_{\mathrm{L}}$, and utilizing the update rule of the Lagrange multiplier, we achieve the global convergence for the objective and the constraint in Theorem \ref{thm_global_convergence} (stated below). In its proof, Lemma \ref{lem.constraint} (stated in the appendix) serves as an important tool in disentangling the regret and constraint violation rates. Interestingly, it is built upon the strong duality property of the unparameterized optimization and has no apparent direct connection with the parameterized setup.

\begin{theorem}
    \label{thm_global_convergence}
    Consider the same setup and parameters as in Lemma \ref{Lemma_global_Lagrange} and set $\beta=T^{-2/5}$ and $\xi = 2/5$. We have,
    \begin{align*}
        \begin{split}
            \frac{1}{K}\sum_{k=1}^{K}\mathbf{E}\bigg(J_r^{\pi^*}-J_r(\theta_k)\bigg)&\leq \sqrt{\epsilon_{\mathrm{bias}}}+\dfrac{\sqrt{A}G^2t_{\mathrm{mix}}}{\delta}\left(1+\dfrac{1}{\mu_F}\right)\tilde{\mathcal{O}}\left(T^{-1/5}\right)\\
            \frac{1}{K}\sum_{k=1}^{K}\mathbf{E}\bigg(J_c(\theta_k)\bigg)&\leq \delta\sqrt{\epsilon_{\mathrm{bias}}} + \sqrt{A}G^2t_{\mathrm{mix}}\left(1+\dfrac{1}{\mu_F}\right)\tilde{\mathcal{O}}\left( T^{-1/5}\right)\\
            &+ \tilde{\mathcal{O}}\left(\dfrac{t_{\mathrm{mix}}t_{\mathrm{hit}}}{\delta T^{1/5}}\right)
        \end{split}
    \end{align*}
    where $\pi^*$ is a solution to the unparameterized optimization. In the above bounds, we write only the dominating terms of $T$.
\end{theorem}
Theorem \ref{thm_global_convergence} establishes $\tilde{\mathcal{O}}(T^{-1/5})$ convergence rates for both the objective and the constraint violation.

\begin{proof}[Proof of Theorem \ref{thm_global_convergence}]$\\$
    {\bf Analysis of Objective:}
    Recall the definition that $J_{\mathrm{L}}(\theta,\lambda)=J_r(\theta)- \lambda J_{c}(\theta)$. The following holds due to Lemma \ref{Lemma_global_Lagrange}.
    \begin{eqnarray}
        \label{eq:bound_Jr1}
        &&\frac{1}{K}\sum_{k=1}^{K}\mathbf{E}\bigg(J_r^{\pi^*}-J_r(\theta_k)\bigg)\leq  \sqrt{\epsilon_{\mathrm{bias}}}+\frac{1}{K}\sum_{k=1}^{K}\mathbf{E}\bigg[\lambda_k\bigg(J_{c}^{\pi^*}-J_c(\theta_k)\bigg)\bigg] \nonumber \\
        &&+ G\left(1+\dfrac{1}{\mu_F}\right)\tilde{\mathcal{O}}\bigg(\sqrt{\beta}+\dfrac{\sqrt{A}Gt_{\mathrm{mix}}}{\delta T^{\xi/2}}+\dfrac{\sqrt{Lt_{\mathrm{mix}}t_{\mathrm{hit}}}}{\delta T^{(1-\xi)/2}}\bigg)\nonumber \\
        &&+\mathcal{\tilde{O}}\bigg(\frac{Lt_{\mathrm{mix}}t_{\mathrm{hit}}\mathbf{E}_{s\sim d^{\pi^*}}[KL(\pi^*(\cdot\vert s)\Vert\pi_{\theta_1}(\cdot\vert s))]}{T^{1-\xi}\delta }\bigg)\nonumber\\
        &&+\dfrac{B}{L}\tilde{\mathcal{O}}\left(\dfrac{AG^2t_{\mathrm{mix}}^2}{\delta^2 T^{\xi}}+\dfrac{Lt_{\mathrm{mix}}t_{\mathrm{hit}}}{\delta^2 T^{1-\xi}}+\beta\right)
    \end{eqnarray}
    We need to bound terms related to $J_c$. Note the following inequalities.
    \begin{equation}
   	\label{eq:bound_lambdak}
   	\begin{aligned}
            &0\leq (\lambda_{K+1})^2 \overset{(a)}{=} \sum_{k=1}^{K}\bigg((\lambda_{k+1})^2-(\lambda_{k})^2\bigg)\\
            &=\sum_{k=1}^{K}\bigg(\mathcal{P}_{[0,\frac{2}{\delta}]}\big[\lambda_{k}+\beta\hat{J}_{c}(\theta_k)\big]^2-(\lambda_{k})^2\bigg)\\
            &\overset{(b)}{\leq}\sum_{k=1}^{K}\bigg(\big[\lambda_{k}+\beta\hat{J}_{c}(\theta_k)\big]^2-(\lambda_{k})^2\bigg)\\
            &\overset{(c)}\leq -2\beta\sum_{k=1}^{K}\lambda_{k}(J_{c}^{\pi^*}- \hat{J}_{c}(\theta_k))+\beta^2\sum_{k=1}^{K}\hat{J}^2_{c}(\theta_k)\\
            &\leq -2\beta\sum_{k=1}^{K}\lambda_{k}(J_{c}^{\pi^*}- \hat{J}_{c}(\theta_k))+2\beta^2\sum_{k=1}^{K}\hat{J}^2_{c}(\theta_k)\\
            &=-2\beta\sum_{k=1}^{K}\left\lbrace\lambda_{k}(J_{c}^{\pi^*}- J_{c}(\theta_k))+\lambda_{k}(J_{c}(\theta_k)- \hat{J}_{c}(\theta_k))\right\rbrace + 2\beta^2\sum_{k=1}^{K}\hat{J}^2_{c}(\theta_k)
   	\end{aligned}
   \end{equation}
    Step $(a)$ uses $\lambda_1=0$, $(b)$ is due to the non-expansiveness of the projection operator, and $(c)$ holds since $\pi^*_{\theta}$ is feasible for the conservative problem. Rearranging items and taking the expectation, we have
    \begin{align}
        \label{eq_appndx_71_new}
        \begin{split}
            &\frac{1}{K}\sum_{k=1}^{K}\mathbf{E}\bigg[\lambda_{k}(J_{c}^{\pi^*}- J_{c}(\theta_k))\bigg] \\
            &\leq -\frac{1}{K}\sum_{k=1}^{K}\mathbf{E}\bigg[\lambda_{k}(J_{c}(\theta_k)- \hat{J}_{c}(\theta_k))\bigg]+\frac{\beta}{K}\sum_{k=1}^{K}\mathbf{E}[\hat{J}^2_{c}(\theta_k)]\\
            &\overset{(a)}{\leq} -\frac{1}{K}\sum_{k=1}^{K}\mathbf{E}\left[\lambda_{k}\left(J_{c}(\theta_k)- \hat{J}_{c}(\theta_k)\right)\right]+\beta\\
            &\overset{(b)}{=} -\frac{1}{K}\sum_{k=1}^{K}\mathbf{E}\bigg[\lambda_{k}\left(J_{c}(\theta_k)- \mathbf{E}\left[\hat{J}_{c}(\theta_k)\big|\theta_k\right]\right)\bigg]+\beta\\
            &\leq \frac{1}{K}\sum_{k=1}^{K}\mathbf{E}\bigg[\lambda_{k}\left|J_{c}(\theta_k)- \mathbf{E}\left[\hat{J}_{c}(\theta_k)\big|\theta_k\right]\right|\bigg]+\beta\overset{(c)}{\leq} \frac{2}{\delta T^2}+\beta
        \end{split}
    \end{align}
    where (a) results from $|\hat{J}_{c, \rho}(\theta)|^2\leq 1$, $\forall \theta\in\Theta$ and (b) uses the fact that $\hat{J}_{c, \rho}(\theta_k)$ and $\lambda_k$ are conditionally independent given $\theta_k$. Finally, (c) is a consequence of Lemma \ref{lemma_aux_6}. Combining \eqref{eq_appndx_71_new} with  \eqref{eq:bound_Jr1}, we deduce,
   \begin{equation}
        \label{eq_bound_Jr_final}
        \begin{aligned}
            &\frac{1}{K}\sum_{k=1}^{K}\mathbf{E}\bigg(J_r^{\pi^*}-J_r(\theta_k)\bigg)\leq  \sqrt{\epsilon_{\mathrm{bias}}}+ \mathcal{O}\left(\dfrac{1}{\delta T^2}+\beta\right)\\
            &+G\left(1+\dfrac{1}{\mu_F}\right)\tilde{\mathcal{O}}\left(\sqrt{\beta}+\dfrac{\sqrt{A}Gt_{\mathrm{mix}}}{\delta T^{\xi/2}}+\dfrac{\sqrt{Lt_{\mathrm{mix}}t_{\mathrm{hit}}}}{\delta T^{(1-\xi)/2}}\right)\\
            &+\dfrac{B}{L}\tilde{\mathcal{O}}\left(\dfrac{AG^2t_{\mathrm{mix}}^2}{\delta^2 T^{\xi}}+\dfrac{Lt_{\mathrm{mix}}t_{\mathrm{hit}}}{\delta^2 T^{1-\xi}}+\beta\right)\\
            &+ \mathcal{\tilde{O}}\bigg(\frac{Lt_{\mathrm{mix}}t_{\mathrm{hit}}\mathbf{E}_{s\sim d^{\pi^*}}[KL(\pi^*(\cdot\vert s)\Vert\pi_{\theta_1}(\cdot\vert s))]}{T^{1-\xi}\delta }\bigg)\\
            &\leq\sqrt{\epsilon_{\mathrm{bias}}} + G\left(1+\dfrac{1}{\mu_F}\right)\tilde{\mathcal{O}}\left(\sqrt{\beta}+\dfrac{\sqrt{A}Gt_{\mathrm{mix}}}{\delta T^{\xi/2}}+\dfrac{\sqrt{Lt_{\mathrm{mix}}t_{\mathrm{hit}}}}{\delta T^{(1-\xi)/2}}\right)
        \end{aligned}
    \end{equation}
    The last inequality presents only the dominant terms of $\beta$ and $T$.
	
    {\bf Analysis of Constraint: }
     Since $\{\lambda_k\}_{k=1}^{K}$ are derived by applying the dual update in Algorithm \ref{alg:PG_MAG}, we have,
    \begin{equation}
        \label{eq_appndx_73_new}
        \begin{aligned}
            \mathbf{E}&\left\vert\lambda_{k+1} - \dfrac{2}{\delta}\right\vert^2\overset{(a)}{\leq} \mathbf{E}\left|\lambda_{k} + \beta \hat{J}_c(\theta_{k}) - \dfrac{2}{\delta}\right|^2\\
            &=\mathbf{E}\left|\lambda_{k} -\dfrac{2}{\delta}\right|^2 +2\beta \mathbf{E}\left[\hat{J}_c(\theta_k)\left(\lambda_{k}  -\dfrac{2}{\delta}\right)\right] +\beta^2 \mathbf{E}\left[\hat{J}^2_c(\theta_{k})\right]\\
            &\overset{(b)}\leq\mathbf{E}\left|\lambda_{k} -\dfrac{2}{\delta}\right|^2 + 2\beta \mathbf{E}\left[J_c(\theta_{k})\left(\lambda_{k} -\dfrac{2}{\delta}\right)\right]\\
            &+2\beta \mathbf{E}\left[\left(\hat{J}_c(\theta_{k})-J_c(\theta_{k})\right)\left(\lambda_{k}-\dfrac{2}{\delta}\right)\right] + \beta^2\\
            &\overset{(c)}{=}\mathbf{E}\left|\lambda_{k} -\dfrac{2}{\delta}\right|^2 +2\beta \mathbf{E}\left[J_c(\theta_{k})\left(\lambda_{k} -\dfrac{2}{\delta}\right)\right]\\
            &+2\beta \mathbf{E}\left[\left(\mathbf{E}\left[\hat{J}_c(\theta_{k})\big|\theta_k\right]-J_c(\theta_{k})\right)\left(\lambda_{k}-\dfrac{2}{\delta}\right)\right] + \beta^2\\
            &\leq \mathbf{E}\left|\lambda_{k} -\dfrac{2}{\delta}\right|^2 +2\beta \mathbf{E}\left[J_c(\theta_{k})\left(\lambda_{k} -\dfrac{2}{\delta}\right)\right]\\
            &+2\beta \mathbf{E}\left[\left|\mathbf{E}\left[\hat{J}_c(\theta_{k})\big|\theta_k\right]-J_c(\theta_{k})\right|\left|\lambda_{k}-\dfrac{2}{\delta}\right|\right] + \beta^2\\
            &\overset{(d)}{\leq} \mathbf{E}\left|\lambda_{k} -\dfrac{2}{\delta}\right|^2 +2\beta \mathbf{E}\left[J_c(\theta_{k})\left(\lambda_{k} -\dfrac{2}{\delta}\right)\right] + \dfrac{4\beta}{\delta T^2} + \beta^2
	\end{aligned}
    \end{equation}
    where $(a)$ is due to the non-expansiveness of the projection $\mathcal{P}_{[0, \frac{2}{\delta}]}$ and $(b)$ holds because $\hat{J}_c(\theta)\in[0,1]$, $\forall \theta\in\Theta$ according to its definition in Algorithm \ref{alg:PG_MAG}. Finally, $(c)$ is a consequence of the fact that $\hat{J}_c(\theta_k)$ and $\lambda_k$ are conditionally independent given $\theta_k$ whereas $(d)$ applies $|\lambda_k-\frac{2}{\delta}|\leq \frac{2}{\delta}$ and Lemma \ref{lemma_aux_6}. Averaging \eqref{eq_appndx_73_new} over $k\in\{1,\ldots,K\}$, we get,
    \begin{equation}
        \begin{aligned}
            -\frac{1}{K}&\sum_{k=1}^{K} \mathbf{E}\left[J_c(\theta_k)\left(\lambda_{k} -\frac{2}{\delta}\right)\right] \\
            &\leq \frac{\left\vert\lambda_{1} - \frac{2}{\delta}\right\vert^2 -\left\vert\lambda_{K+1} - \frac{2}{\delta}\right\vert^2 }{2\beta K} + \dfrac{2}{\delta T^2} + \dfrac{\beta}{2}\\
            &\overset{(a)}{\leq} \dfrac{2}{\delta^2 \beta K} + \dfrac{2}{\delta T^2} + \dfrac{\beta}{2}
        \end{aligned}
    \end{equation}
    where (a) uses the presumption that $\lambda_1=0$. Note that $\lambda_k J_c^{\pi^*}\leq 0$, $\forall k$. Adding the above inequality to \eqref{eq:bound_Jr1} at both sides, we, therefore, have,
    \begin{equation}
        \label{eq:bound_Jc2}
        \begin{aligned}
            &\mathbf{E}\bigg[J_r^{\pi^*}-\frac{1}{K}\sum_{k=1}^{K}J_r(\theta_k)\bigg]+\frac{2}{\delta}\mathbf{E}\bigg[\frac{1}{K}\sum_{k=1}^{K}J_c(\theta_k)\bigg]\\
            &\leq \sqrt{\epsilon_{\mathrm{bias}}} +\frac{2}{\delta^2\beta K} +\frac{2}{T^2\delta}+ \frac{\beta}{2}\\
            &+G\left(1+\dfrac{1}{\mu_F}\right)\tilde{\mathcal{O}}\left(\sqrt{\beta}+\dfrac{\sqrt{A}Gt_{\mathrm{mix}}}{\delta T^{\xi/2}}+\dfrac{\sqrt{Lt_{\mathrm{mix}}t_{\mathrm{hit}}}}{\delta T^{(1-\xi)/2}}\right) \\
            &+ \dfrac{B}{L}\tilde{\mathcal{O}}\left(\dfrac{AG^2t_{\mathrm{mix}}^2}{\delta^2 T^{\xi}}+\dfrac{Lt_{\mathrm{mix}}t_{\mathrm{hit}}}{\delta^2 T^{1-\xi}}+\beta\right)\\
            &+ \mathcal{\tilde{O}}\bigg(\frac{Lt_{\mathrm{mix}}t_{\mathrm{hit}}\mathbf{E}_{s\sim d^{\pi^*}}[KL(\pi^*(\cdot\vert s)\Vert\pi_{\theta_1}(\cdot\vert s))]}{T^{1-\xi}\delta }\bigg)
        \end{aligned}
    \end{equation}
    Since the functions $\{J_g(\theta_k)\}, k\in\{0, \cdots, K-1\}, g\in\{r, c\}$ are linear in occupancy measure, there exists a policy $\bar{\pi}$ such that the following equation hols $\forall g\in\{r, c\}$.
    \begin{align*}
        \frac{1}{K}\sum_{k=1}^{K}J_g(\theta_k)=J_g^{\bar\pi}
    \end{align*}
    Injecting the above relation to \eqref{eq:bound_Jc2}, we have
    \begin{equation}
        \begin{aligned}
            &\mathbf{E}\bigg[J_r^{\pi^*}-J_r^{\bar\pi}\bigg]+\frac{2}{\delta}\mathbf{E}\bigg[J_c^{\bar\pi}\bigg]\leq \sqrt{\epsilon_{\mathrm{bias}}}+\frac{2}{\delta^2\beta K} +\frac{2}{T^2\delta}+ \frac{\beta}{2} \\
            &+G\left(1+\dfrac{1}{\mu_F}\right)\tilde{\mathcal{O}}\left(\sqrt{\beta}+\dfrac{\sqrt{A}Gt_{\mathrm{mix}}}{\delta T^{\xi/2}}+\dfrac{\sqrt{Lt_{\mathrm{mix}}t_{\mathrm{hit}}}}{\delta T^{(1-\xi)/2}}\right)\\
            &+ \mathcal{\tilde{O}}\bigg(\frac{Lt_{\mathrm{mix}}t_{\mathrm{hit}}\mathbf{E}_{s\sim d^{\pi^*}}[KL(\pi^*(\cdot\vert s)\Vert\pi_{\theta_1}(\cdot\vert s))]}{T^{1-\xi}\delta }\bigg)\\
            &+\dfrac{B}{L}\tilde{\mathcal{O}}\left(\dfrac{AG^2t_{\mathrm{mix}}^2}{\delta^2 T^{\xi}}+\dfrac{Lt_{\mathrm{mix}}t_{\mathrm{hit}}}{\delta^2 T^{1-\xi}}+\beta\right) 
        \end{aligned}
    \end{equation}
    By Lemma \ref{lem.constraint}, we arrive at,
    \begin{equation}
        \label{eq_83}
        \begin{aligned}
            \mathbf{E}\bigg[J_c^{\bar\pi}\bigg]&\leq  \delta\sqrt{\epsilon_{\mathrm{bias}}}+\frac{2}{\delta\beta K} +\frac{2}{T^2}+\frac{\delta\beta}{2}\\
            &+G\left(1+\dfrac{1}{\mu_F}\right)\tilde{\mathcal{O}}\left(\delta\sqrt{\beta}+\dfrac{\sqrt{A}Gt_{\mathrm{mix}}}{ T^{\xi/2}}+\dfrac{\sqrt{Lt_{\mathrm{mix}}t_{\mathrm{hit}}}}{ T^{(1-\xi)/2}}\right)\\
            &+ \mathcal{\tilde{O}}\bigg(\frac{Lt_{\mathrm{mix}}t_{\mathrm{hit}}\mathbf{E}_{s\sim d^{\pi^*}}[KL(\pi^*(\cdot\vert s)\Vert\pi_{\theta_1}(\cdot\vert s))]}{T^{1-\xi}}\bigg)\\
            &+\dfrac{B}{L}\tilde{\mathcal{O}}\left(\dfrac{AG^2t_{\mathrm{mix}}^2}{\delta T^{\xi}}+\dfrac{Lt_{\mathrm{mix}}t_{\mathrm{hit}}}{\delta T^{1-\xi}}+\delta\beta\right) \\
            &\leq \delta\sqrt{\epsilon_{\mathrm{bias}}} + \tilde{\mathcal{O}}\left(\dfrac{2t_{\mathrm{mix}}t_{\mathrm{hit}}}{\delta \beta T^{1-\xi}}\right) \\
            &+ G\left(1+\dfrac{1}{\mu_F}\right)\tilde{\mathcal{O}}\left(\delta\sqrt{\beta}+\dfrac{\sqrt{A}Gt_{\mathrm{mix}}}{ T^{\xi/2}}+\dfrac{\sqrt{Lt_{\mathrm{mix}}t_{\mathrm{hit}}}}{ T^{(1-\xi)/2}}\right)
        \end{aligned}
    \end{equation}

    The last inequality presents only the dominant terms of $\beta$ and $T$. If we choose $\beta=T^{-\eta}$ for some $\eta\in(0,1)$, then following \eqref{eq_bound_Jr_final} and \eqref{eq_83}, we can compactly write the following inequalities.
    \begin{align*}
        &\frac{1}{K}\sum_{k=1}^{K}\mathbf{E}\bigg(J_r^{\pi^*}-J_r(\theta_k)\bigg)\leq  \sqrt{\epsilon_{\mathrm{bias}}} + \tilde{\mathcal{O}}\left(T^{-\eta/2}+T^{-\xi/2}+T^{-(1-\xi)/2}\right),\\
        &\mathbf{E}\left[\frac{1}{K}\sum_{k=1}^{K} J_c(\theta_k)\right]\leq \delta\sqrt{\epsilon_{\mathrm{bias}}} \\
        &\hspace{2cm}+ \tilde{\mathcal{O}}\left(T^{-(1-\xi-\eta)}+T^{-\eta/2}+T^{-\xi/2}+T^{-(1-\xi)/2}\right)
    \end{align*}

    The optimal values of $\eta$ and $\xi$ can be obtained by solving the following optimization.
    \begin{align}
        {\max}_{(\eta, \xi)\in (0,1)^2} \min \left\lbrace 1-\xi-\eta, \dfrac{\eta}{2}, \dfrac{\xi}{2}, \dfrac{1-\xi}{2} \right\rbrace 
    \end{align}

    One can easily verify that $(\xi, \eta) = \left(2/5, 2/5\right)$ is the solution of the above optimization. Therefore, the convergence rate of the objective function can be written as follows.
    \begin{equation}
        \label{eq_bound_Jr_final_par_substituted}
        \begin{aligned}
            &\frac{1}{K}\sum_{k=1}^{K}\mathbf{E}\bigg(J_r^{\pi^*}-J_r(\theta_k)\bigg)\leq  \sqrt{\epsilon_{\mathrm{bias}}} + \mathcal{O}\left(\dfrac{1}{\delta T^2}+\dfrac{1}{T^{2/5}}\right)\\
            &+ G\left(1+\dfrac{1}{\mu_F}\right)\tilde{\mathcal{O}}\left(\dfrac{1}{T^{1/5}}+\dfrac{\sqrt{A}Gt_{\mathrm{mix}}}{\delta T^{1/5}}+\dfrac{\sqrt{Lt_{\mathrm{mix}}t_{\mathrm{hit}}}}{\delta T^{3/10}}\right)\\
            &+\dfrac{B}{L}\tilde{\mathcal{O}}\left(\dfrac{\delta^2+AG^2t_{\mathrm{mix}}^2}{\delta^2 T^{2/5}}+\dfrac{Lt_{\mathrm{mix}}t_{\mathrm{hit}}}{\delta^2 T^{3/5}}\right)\\
            &+ \mathcal{\tilde{O}}\bigg(\frac{Lt_{\mathrm{mix}}t_{\mathrm{hit}}\mathbf{E}_{s\sim d^{\pi^*}}[KL(\pi^*(\cdot\vert s)\Vert\pi_{\theta_1}(\cdot\vert s))]}{T^{3/5}\delta }\bigg)\\
            &\leq \sqrt{\epsilon_{\mathrm{bias}}} + \dfrac{\sqrt{A}G^2t_{\mathrm{mix}}}{\delta}\left(1+\dfrac{1}{\mu_F}\right)\tilde{\mathcal{O}}\left(T^{-1/5}\right)
        \end{aligned}
    \end{equation}

    The last expression only considers the dominant terms of $T$. Similarly, the constraint violation rate can be computed as,
    \begin{align}
        \label{eq_87_new}
        \begin{split}
            &\mathbf{E}\bigg[\frac{1}{K}\sum_{k=1}^{K}J_c(\theta_k)\bigg] \leq \delta\sqrt{\epsilon_{\mathrm{bias}}}+ \tilde{\mathcal{O}}\left(\dfrac{t_{\mathrm{mix}}t_{\mathrm{hit}}}{\delta T^{1/5}}+\dfrac{1}{ T^2}+\dfrac{\delta}{T^{2/5}}\right) \\
            &+G\left(1+\dfrac{1}{\mu_F}\right)\tilde{\mathcal{O}}\left(\dfrac{\delta+\sqrt{A}Gt_{\mathrm{mix}}}{ T^{1/5}}+\dfrac{\sqrt{Lt_{\mathrm{mix}}t_{\mathrm{hit}}}}{ T^{3/10}}\right)\\
            &+\dfrac{B}{L}\tilde{\mathcal{O}}\left(\dfrac{\delta^2+ AG^2t_{\mathrm{mix}}^2}{\delta T^{2/5}}+\dfrac{Lt_{\mathrm{mix}}t_{\mathrm{hit}}}{\delta T^{3/5}}\right)\\
            &+ \mathcal{\tilde{O}}\bigg(\frac{Lt_{\mathrm{mix}}t_{\mathrm{hit}}\mathbf{E}_{s\sim d^{\pi^*}}[KL(\pi^*(\cdot\vert s)\Vert\pi_{\theta_1}(\cdot\vert s))]}{T^{3/5}}\bigg)\\
            &\leq \delta\sqrt{\epsilon_{\mathrm{bias}}} + \tilde{\mathcal{O}}\left(\dfrac{t_{\mathrm{mix}}t_{\mathrm{hit}}}{\delta T^{1/5}}\right) + \sqrt{A}G^2t_{\mathrm{mix}}\left(1+\dfrac{1}{\mu_F}\right)\tilde{\mathcal{O}}\left(T^{-1/5}\right)
        \end{split}
    \end{align}
    where the last expression contains only the dominant terms of $T$. This concludes the theorem.
\end{proof}


\section{Regret Analysis and Constraint Violation Analysis}\label{mf:reg}

In this section, we use the convergence analysis in the previous section to bound the expected regret and constraint violation of Algorithm \ref{alg:PG_MAG}. Note that the regret and constraint violation can be written as follows.
\begin{equation}
    \label{reg_vio_decompose}
    \begin{aligned}
        &R(T) = H\sum_{k=1}^{K}\left(J_r^*-J_r({\theta_k})\right)+\sum_{k=1}^{K}\sum_{t\in\mathcal{I}_k} \left(J_r(\theta_k)-r(s_t, a_t)\right)\\
        &C(T) = H\sum_{k=1}^{K}\left(J_c({\theta_k})\right)-\sum_{k=1}^{K}\sum_{t\in\mathcal{I}_k} \left(J_c(\theta_k)-c(s_t, a_t)\right)
    \end{aligned}
\end{equation}
where $\mathcal{I}_k\triangleq \{(k-1)H, \cdots, kH-1\}$. Note that the expectation of the first terms in regret and violation can be bounded by Theorem \ref{thm_global_convergence}. The expectation of the second term can be expanded as follows,
\begin{equation}\label{eq_38}
    \begin{aligned}
        &\mathbf{E}\left[\sum_{k=1}^{K}\sum_{t\in\mathcal{I}_k} \left(J_g(\theta_k)-g(s_t, a_t)\right)\right]\\
        &\overset{(a)}{=}\mathbf{E}\left[\sum_{k=1}^{K}\sum_{t\in\mathcal{I}_k} \mathbf{E}_{s'\sim P(\cdot|s_t, a_t)}[V_g^{\pi_{\theta_k}}(s')]-Q_g^{\pi_{\theta_k}}(s_t, a_t)\right]\\
        &\overset{(b)}{=}\mathbf{E}\left[\sum_{k=}^{K}\sum_{t\in\mathcal{I}_k} V_g^{\pi_{\theta_k}}(s_{t+1})-V_g^{\pi_{\theta_k}}(s_t)\right]\\
        &=\mathbf{E}\left[\sum_{k=1}^K V_g^{\pi_{\theta_k}}(s_{kH})-V_g^{\pi_{\theta_k}}(s_{(k-1)H})\right]\\
        &=\mathbf{E}\left[\sum_{k=1}^{K-1} V_g^{\pi_{\theta_{k+1}}}(s_{kH})-V_g^{\pi_{\theta_k}}(s_{kH})\right] + \mathbf{E}\left[V_g^{\pi_{\theta_K}}(s_{T})-V_g^{\pi_{\theta_0}}(s_{0})\right]
    \end{aligned}
\end{equation}
where $g\in\{r, c\}$. Equality $(a)$ uses the Bellman equation and $(b)$ follows from the definition of $Q_g$. The first and second terms in the last line of \eqref{eq_38} can be bounded by Lemma \ref{lemma_last} and \ref{lemma_aux_2} (forthcoming) respectively.

\begin{lemma}
    \label{lemma_last}
    If assumptions \ref{ass_1} and \ref{ass_score} hold, then for $K=\frac{T}{H}$ where $H=16t_{\mathrm{mix}}t_{\mathrm{hit}}T^{\frac{2}{5}}(\log_2 T)^2$, the following  hold $\forall k$, $\forall (s, a)$, $\forall g\in\{r, c\}$, sufficiently large $T$, and an arbitrary sequence of states $\{s_k\}_{k=1}^K$.
    \begin{align*}
        &(a) ~|\pi_{\theta_{k+1}}(a|s)-\pi_{\theta_{k}}(a|s)|\leq G\norm{\theta_{k+1}-\theta_k}\\
        &(b) ~\sum_{k=1}^{K}\mathbf{E}|J_g(\theta_{k+1})-J_g(\theta_k)|
        \leq \dfrac{\alpha AG}{\delta t_{\mathrm{hit}}}f(T)\\
        &(c) ~\sum_{k=1}^K\mathbf{E}|V_g^{\pi_{\theta_{k+1}}}(s_k) - V_g^{\pi_{\theta_{k}}}(s_k)| 
        \leq  t_{\mathrm{mix}}\dfrac{\alpha AG}{\delta t_{\mathrm{hit}}}f(T)
    \end{align*}
    where $f(T)\triangleq \tilde{\mathcal{O}}\left(\left[\sqrt{A}G t_{\mathrm{mix}}+\delta\right]T^{\frac{2}{5}}+\sqrt{Lt_{\mathrm{mix}}t_{\mathrm{hit}}}T^{\frac{3}{10}}\right)$.
\end{lemma}
\begin{proof}
    Using Taylor's expansion, we write the following $\forall (s, a)$, $\forall k$.
    \begin{align}
	\label{eq_pi_lipschitz}
	\begin{split}
            &|\pi_{\theta_{k+1}}(a|s)-\pi_{\theta_{k}}(a|s)|=\left|(\theta_{k+1}-\theta_k)^T\nabla_{\theta}\pi_{\bar\theta}(a|s) \right|\\
            &=\pi_{\bar{\theta}_k}(a|s)\left|(\theta_{k+1}\theta_k)^T\nabla_{\theta}\log \pi_{\bar{\theta}_k}(a|s) \right|\overset{(a)}{\leq} G\norm{\theta_{k+1}-\theta_k}
	\end{split} 
    \end{align}
    where $\bar{\theta}_k$ is some convex combination of $\theta_{k}$ and $\theta_{k+1}$ and $(a)$ follows from Assumption \ref{ass_score}. Applying \eqref{eq_pi_lipschitz} and Lemma \ref{lemma_aux_5}, we obtain,
    \begin{align}
	\label{eq_long_49}
	\begin{split}
            &\sum_{k=1}^{K}\mathbf{E}\left|J_g(\theta_{k+1}) - J_g(\theta_{k})\right|\\
            &=\sum_{k=1}^{K}\mathbf{E}\left|\sum_{s,a}d^{\pi_{\theta_{k+1}}}(s)(\pi_{\theta_{k+1}}(a|s)-\pi_{\theta_{k}}(a|s))Q_g^{\pi_{\theta_{k}}}(s, a)\right|\\
            &\leq \sum_{k=1}^{K}\mathbf{E}\bigg[\sum_{s,a}d^{\pi_{\theta_{k+1}}}(s)\left|\pi_{\theta_{k+1}}(a|s)-\pi_{\theta_{k}}(a|s)\right|\left|Q_g^{\pi_{\theta_{k}}}(s, a)\right|\bigg]\\
            &\leq G\sum_{k=1}^{K}\mathbf{E}\left[\sum_{s,a}d^{\pi_{\theta_{k+1}}}(s)\Vert\theta_{k+1}-\theta_{k}\Vert|Q_g^{\pi_{\theta_{k-1}}}(s, a)|\right]\\
            &\overset{(a)}{\leq} G\alpha\sum_{k=1}^{K}\mathbf{E}\bigg[\sum_{a} \sum_{s}d^{\pi_{\theta_{k+1}}}(s)\Vert\omega_k\Vert\cdot 6t_{\mathrm{mix}}\bigg]\\
            & \overset{(b)}\leq 6AG\alpha t_{\mathrm{mix}}\sqrt{K}\left(\sum_{k=1}^{K}\mathbf{E}\norm{\omega_k}^2\right)^{\frac{1}{2}}\\
            &\overset{(c)}{\leq} \mathcal{\tilde{O}}\left(\dfrac{\alpha AG}{\delta t_{\mathrm{hit}}}\left[\left(\sqrt{A}G t_{\mathrm{mix}}+\delta\right)T^{\frac{2}{5}}+\sqrt{Lt_{\mathrm{mix}}t_{\mathrm{hit}}}T^{\frac{3}{10}}\right]\right)
	\end{split}
    \end{align}
	
    Inequality $(a)$ uses Lemma \ref{lemma_aux_2} and the rule $\theta_{k+1}=\theta_k+\alpha \omega _k$. Step $(b)$ holds by the Cauchy inequality and Jensen inequality whereas $(c)$ can be derived using $\eqref{eq_34}$ and substituting $K=T/H$. This establishes the second statement. Next, recall from $(\ref{eq_r_pi_theta})$ that  $g^{\pi_{\theta}}(s) \triangleq \sum_a\pi_{\theta}(a|s)g(s, a)$. Note that, for any $\theta$, and $s\in\mathcal{S}$, the following holds.
    \begin{align}
	\label{eq_49}
        \begin{split}
            &V_g^{\pi_{\theta}}(s)=\sum_{t=0}^{\infty}\left<(P^{\pi_{\theta}})^t(s,\cdot) - d^{\pi_{\theta}}, g^{\pi_{\theta}}\right> \\
            &=\sum_{t=0}^{N-1}\left<(P^{\pi_{\theta}})^t({s,\cdot}),g^{\pi_{\theta}}\right> - NJ(\theta) +\sum_{t=N}^{\infty}\left<(P^{\pi_{\theta}})^t({s,\cdot}) - d^{\pi_{\theta}},g^{\pi_{\theta}}\right>
        \end{split}
    \end{align}
    where $\left<\cdot, \cdot\right>$ denotes the dot product. Define the following quantity.
    \begin{align}
	\label{def_error}
        \delta^{\pi_{\theta}}(s, T) \triangleq \sum_{t=N}^{\infty}\norm{(P^{\pi_{\theta}})^t({s,\cdot}) - d^{\pi_{\theta}}}_1 ~\text{where} ~N=4t_{\mathrm{mix}}(\log_2 T)
    \end{align}
	
    Lemma \ref{lemma_aux_3} states that for sufficiently large $T$, $\delta^{\pi_{\theta}}(s, T)\leq \frac{1}{T^3}$ for any $\theta$ and $s$. Combining this result with the fact that the reward function is bounded in $[0, 1]$, we obtain,	
    \begin{align}
	\label{eq_exp_diff_v}
	\begin{split}    
            &\sum_{k=1}^K\mathbf{E}|V_g^{\pi_{\theta_{k+1}}}(s_k) - V_g^{\pi_{\theta_{k}}}(s_k)|\\
            &\leq \sum_{k=1}^K\mathbf{E}\left|\sum_{t=0}^{N-1}\left<(P^{\pi_{\theta_{k+1}}})^t({s_k,\cdot}) - (P^{\pi_{\theta_k}})^t({s_k,\cdot}), g^{\pi_{\theta_{k+1}}}\right>\right|\\
            &+ \sum_{k=1}^K\mathbf{E}\left|\sum_{t=0}^{N-1}\left<(P^{\pi_{\theta_k}})^t({s_k,\cdot}), g^{\pi_{\theta_{k+1}}}-g^{\pi_{\theta_k}}\right>\right| \\
            &+ N\sum_{k=1}^K\mathbf{E}|J_g(\theta_{k+1}) - J_g(\theta_{k})| + \frac{2K}{T^3}\\
            &\overset{(a)}{\leq} \sum_{k=1}^K\sum_{t=0}^{N-1}\mathbf{E}\norm{ (P^{\pi_{\theta_{k+1}}})^t - (P^{\pi_{\theta_k}})^t)g^{\pi_{\theta_{k+1}}} }_{\infty}\\
            &+\sum_{k=1}^K\sum_{t=0}^{N-1} \mathbf{E}\norm{g^{\pi_{\theta_{k+1}}}-g^{\pi_{\theta_k}}}_{\infty} \\
            &+ \mathcal{\tilde{O}}\left(\dfrac{\alpha AG t_{\mathrm{mix}}}{\delta t_{\mathrm{hit}}}\left[\left(\sqrt{A}G t_{\mathrm{mix}}+\delta\right)T^{\frac{2}{5}}+\sqrt{Lt_{\mathrm{mix}}t_{\mathrm{hit}}}T^{\frac{3}{10}}\right]\right)
	\end{split}
    \end{align}
    where $(a)$ follows from \eqref{eq_long_49} and using $N=4t_{\mathrm{mix}}(\log_2 T)$. Note that,
    \begin{align}
	\label{eq_long_recursion}
	\begin{split}
            &\norm{ ((P^{\pi_{\theta_{k+1}}})^t - (P^{\pi_{\theta_k}})^t)g^{\pi_{\theta_{k+1}}} }_{\infty}\\ &\leq \norm{ P^{\pi_{\theta_{k+1}}}((P^{\pi_{\theta_{k+1}}})^{t-1} - (P^{\pi_{\theta_k}})^{t-1})g^{\pi_{\theta_{k+1}}} }_{\infty}\\
            & \hspace{1cm}+ \norm{ (P^{\pi_{\theta_{k+1}}} - P^{\pi_{\theta_k}})(P^{\pi_{\theta_k}})^{t-1}g^{\pi_{\theta_{k+1}}} }_{\infty}\\
            &\overset{(a)}{\leq} \norm{ ((P^{\pi_{\theta_{k+1}}})^{t-1} - (P^{\pi_{\theta_k}})^{t-1})g^{\pi_{\theta_{k+1}}} }_{\infty} \\
            &\hspace{1cm} + \max_s\norm{P^{\pi_{\theta_{k+1}}}({s,\cdot})-P^{\pi_{\theta_k}}({s,\cdot})}_1
        \end{split}
    \end{align}
    where $(a)$ follows from $\norm{(P^{\pi_{\theta_k}})^{t-1}g^{\pi_{\theta_{k+1}}}}_{\infty}\leq 1$ and the fact that each row of $P^{\pi_{\theta_k}}$ sums to $1$. Invoking \eqref{eq_pi_lipschitz}, and $\theta_{k+1}=\theta_k + \alpha \omega_k$, we get,
    \begin{align*}
        &\max_s\Vert P^{\pi_{\theta_{k+1}}}({s,\cdot})-P^{\pi_{\theta_k}}({s,\cdot})\Vert_1 \\
        &=\max_s\left| \sum_{s'}\sum_a(\pi_{\theta_{k+1}}(a|s)-\pi_{\theta_k}(a|s))P(s'|s, a)\right| \\
        &\leq G \norm{\theta_{k+1}-\theta_k}\max_s\left| \sum_{s'}\sum_a P(s'|s, a)\right| \leq \alpha AG\norm{\omega_k}  
    \end{align*}
	
    Plugging the above result into $(\ref{eq_long_recursion})$ and using a recursion, we get,
    \begin{align*}
        \norm{ ((P^{\pi_{\theta_{k+1}}})^t - (P^{\pi_{\theta_k}})^t)g^{\pi_{\theta_{k+1}}} }_{\infty} &\leq \sum_{t'=1}^{t} \max_s\norm{P^{\pi_{\theta_{k+1}}}({s,\cdot})-P^{\pi_{\theta_k}}({s,\cdot})}_1\\
        &\leq \sum_{t'=1}^{t}\alpha AG\norm{\omega_k}  \leq \alpha At G\norm{\omega_k}
    \end{align*}
	
    Finally, we arrive at the following.
    \begin{align}
        \label{eq_app_54}
        \begin{split}
            &\sum_{k=1}^K\sum_{t=0}^{N-1} \mathbf{E}\norm{ ((P^{\pi_{\theta_{k+1}}})^t - (P^{\pi_{\theta_k}})^t)g^{\pi_{\theta_{k+1}}} }_{\infty}\\
            &\leq \sum_{k=1}^K\sum_{t=0}^{N-1}\alpha t AG\mathbf{E}\norm{\omega_k} \\
            &\overset{(a)}{\leq} \mathcal{O}(\alpha AG N^2\sqrt{K}) \left(\sum_{k=1}^K\mathbf{E}\norm{\omega_k}^2\right)^{\frac{1}{2}}\\
            &\overset{(b)}{\leq}\mathcal{\tilde{O}}\left(\dfrac{\alpha AG t_{\mathrm{mix}}}{\delta t_{\mathrm{hit}}}\left[\left(\sqrt{A}G t_{\mathrm{mix}}+\delta\right)T^{\frac{2}{5}}+\sqrt{Lt_{\mathrm{mix}}t_{\mathrm{hit}}}T^{\frac{3}{10}}\right]\right)
	\end{split}
    \end{align}
    where $(a)$ is a result of the Cauchy-Schwarz inequality and $(b)$ follows from \eqref{eq_34}. Moreover, notice that,
    \begin{equation}
        \label{eq_app_55}
        \begin{aligned}
            &\sum_{k=1}^{K}\sum_{t=0}^{N-1}\mathbf{E}\norm{g^{\pi_{\theta_{k+1}}}-g^{\pi_{\theta_{k}}}}_{\infty}\\
            &\overset{}{\leq} \sum_{k=1}^{K}\sum_{t=0}^{N-1}\mathbf{E}\left[\max_s\left|\sum_a(\pi_{\theta_{k+1}}(a|s)-\pi_{\theta_{k}}(a|s))g(s,a)\right|\right]\\
            &\overset{(a)}{\leq}\alpha AGN \sum_{k=1}^{K} \mathbf{E}\norm{\omega_k}\leq \alpha AGN\sqrt{K} \left(\sum_{k=1}^{K} \mathbf{E}\norm{\omega_k}^2\right)^{\frac{1}{2}}\\
            &\overset{(b)}{\leq}\mathcal{\tilde{O}}\left(\dfrac{\alpha AG}{\delta t_{\mathrm{hit}}}\left[\left(\sqrt{A}G t_{\mathrm{mix}}+\delta\right)T^{\frac{2}{5}}+\sqrt{Lt_{\mathrm{mix}}t_{\mathrm{hit}}}T^{\frac{3}{10}}\right]\right)
	\end{aligned}
    \end{equation}
    where $(a)$ follows from \eqref{eq_pi_lipschitz} and the parameter update rule $\theta_{k+1}=\theta_k + \alpha \omega_k$ while $(b)$ is a consequence of \eqref{eq_34}. Combining \eqref{eq_exp_diff_v}, \eqref{eq_app_54}, and  \eqref{eq_app_55}, we establish the third statement.
\end{proof}

Lemma \ref{lemma_last} states that the obtained parameters are such that the average consecutive difference in the sequence $\{J_g(\theta_k)\}_{k=1}^K$, $g\in\{r, c\}$ decreases with time horizon, $T$. We would like to emphasize that Lemma \ref{lemma_last} works for both reward and constraint functions. Thus, we can prove our regret guarantee and constraint violation as follows.

\begin{theorem}
    \label{thm_regret}
    If assumptions \ref{ass_1}$-$\ref{ass_4} hold, $J_g(\cdot)$'s are $L$-smooth, $\forall g\in\{r, c\}$ and $T$ are sufficiently large, then our proposed Algorithm \ref{alg:PG_MAG} achieves the following expected regret and constraint violation bounds with learning rates $\alpha=\frac{1}{4L(1+\frac{2}{\delta})}$ and $\beta=T^{-2/5}$.
    \begin{align}
        \label{eq_regret}
        \begin{split}
            &\mathbf{E}\left[R(T)\right] \leq T\sqrt{\epsilon_{\mathrm{bias}}} +\tilde{\mathcal{O}}(T^{4/5})+\mathcal{O}(t_{\mathrm{mix}})\\
            &\mathbf{E}\left[C(T)\right] \leq T\delta\sqrt{\epsilon_{\mathrm{bias}}} +\tilde{\mathcal{O}}(T^{4/5})+\mathcal{O}(t_{\mathrm{mix}})\\
        \end{split}
    \end{align}
\end{theorem}
\begin{proof}
    Combining \eqref{reg_vio_decompose} and \eqref{eq_38}, we get
    \begin{align*}
        \begin{split}
            &\mathbf{E}[R(T)] = H\sum_{k=1}^{K}\left(J_r^{\pi^*}-\mathbf{E}[J_r({\theta_k})]\right) +\mathbf{E}\left[ V_r^{\pi_{\theta_K}}(s_{T})-V_r^{\pi_{\theta_0}}(s_{0})\right] \\&\hspace{1.9cm} +\mathbf{E}\left[\sum_{k=1}^{K-1} V_r^{\pi_{\theta_{k+1}}}(s_{kH})-V_r^{\pi_{\theta_k}}(s_{kH})\right]
   	\end{split}
   \end{align*}
    Using the result in Theorem \ref{thm_global_convergence}, Lemma \ref{lemma_last} and Lemma \ref{lemma_aux_2}, we get,
    \begin{align}
        \begin{split}
            &\mathbf{E}[\mathrm{Reg}_T]\leq T\sqrt{\epsilon_{\mathrm{bias}}}+ \mathcal{O}\left(\dfrac{1}{T}+T^{\frac{3}{5}}\right) +\mathcal{O}(t_{\mathrm{mix}})\\
            &+G\left(1+\dfrac{1}{\mu_F}\right)\tilde{\mathcal{O}}\left(T^{\frac{4}{5}}+\dfrac{\sqrt{A}Gt_{\mathrm{mix}}}{\delta }T^{\frac{4}{5}}+\dfrac{\sqrt{Lt_{\mathrm{mix}}t_{\mathrm{hit}}}}{\delta}T^{\frac{7}{10}}\right)\\
            &+\dfrac{B}{L}\tilde{\mathcal{O}}\left(\dfrac{\delta^2+AG^2t_{\mathrm{mix}}^2}{\delta^2}T^{\frac{3}{5}}+\dfrac{Lt_{\mathrm{mix}}t_{\mathrm{hit}}}{\delta^2}T^{\frac{2}{5}}\right)\\
            &+\mathcal{\tilde{O}}\bigg(\frac{Lt_{\mathrm{mix}}t_{\mathrm{hit}}\mathbf{E}_{s\sim d^{\pi^*}}[KL(\pi^*(\cdot\vert s)\Vert\pi_{\theta_1}(\cdot\vert s))]}{\delta }T^{\frac{2}{5}}\bigg) \\
            &+\mathcal{\tilde{O}}\left(\dfrac{\alpha AG t_{\mathrm{mix}}}{\delta t_{\mathrm{hit}}}\left[\left(\sqrt{A}G t_{\mathrm{mix}}+\delta\right)T^{\frac{2}{5}}+\sqrt{Lt_{\mathrm{mix}}t_{\mathrm{hit}}}T^{\frac{3}{10}}\right]\right)
        \end{split}
    \end{align}
    Similarly, for the constraint violation, we have
    \begin{align}
        \begin{split}
            \mathbf{E}[C(T)]&=H\sum_{k=1}^{K}\mathbf{E}[J_c({\theta_k})]-\mathbf{E}\left[\sum_{k=1}^{K-1} V_c^{\pi_{\theta_{k+1}}}(s_{kH})-V_c^{\pi_{\theta_k}}(s_{kH})\right]\\
            &-\mathbf{E}\left[ V_c^{\pi_{\theta_K}}(s_{T})-V_c^{\pi_{\theta_0}}(s_{0})\right]
        \end{split}
    \end{align}
    Using the result in Theorem \ref{thm_global_convergence}, Lemma \ref{lemma_last} and Lemma \ref{lemma_aux_2}, we get,
    \begin{align}
        \begin{split}
            &\mathbf{E}[C(T)]\leq T\delta\sqrt{\epsilon_{\mathrm{bias}}} + +\mathcal{O}(t_{\mathrm{mix}})\\
            &+G\left(1+\dfrac{1}{\mu_F}\right)\tilde{\mathcal{O}}\left(\left[\delta+\sqrt{A}Gt_{\mathrm{mix}}\right]T^{\frac{4}{5}} + \sqrt{Lt_{\mathrm{mix}}t_{\mathrm{hit}}}T^{\frac{7}{10}}\right)\\
            &+\mathcal{O}\left(\dfrac{t_{\mathrm{mix}}t_{\mathrm{hit}}}{\delta}T^{\frac{4}{5}}+\frac{1}{\delta T}+\delta T^{\frac{3}{5}}\right)\\
            &+\dfrac{B}{L}\tilde{\mathcal{O}}\left(\dfrac{\delta^2+AG^2t_{\mathrm{mix}}^2}{\delta}T^{\frac{3}{5}}+\dfrac{Lt_{\mathrm{mix}}t_{\mathrm{hit}}}{\delta}T^{\frac{2}{5}}\right)\\
            &+\mathcal{\tilde{O}}\bigg(Lt_{\mathrm{mix}}t_{\mathrm{hit}}\mathbf{E}_{s\sim d^{\pi^*}}[KL(\pi^*(\cdot\vert s)\Vert\pi_{\theta_1}(\cdot\vert s))] T^{\frac{2}{5}}\bigg) \\
            &+\mathcal{\tilde{O}}\left(\dfrac{\alpha AG t_{\mathrm{mix}}}{\delta t_{\mathrm{hit}}}\left[\left(\sqrt{A}G t_{\mathrm{mix}}+\delta\right)T^{\frac{2}{5}}+\sqrt{Lt_{\mathrm{mix}}t_{\mathrm{hit}}}T^{\frac{3}{10}}\right]\right)
        \end{split}
    \end{align}
    This completes the proof of Theorem \ref{thm_regret}.
\end{proof}

We note that this result significantly improves the best-known regret results for model-free average reward setup with constraints even in the tabular setup \citep{wei2022provably}, where the best-known regret was $\tilde{\mathcal{O}}(T^{5/6})$ (see Table \ref{table2}). 


\section{Notes and Open Problems}
\label{sec:mf:notes}

The results in this chapter are taken from  \cite{bai2023learning} which presents the state-of-the-art results for average reward CMDPs with general parameterization. However, this is far from the $\Omega(\sqrt{T})$ lower bound. For unconstrained average reward MDPs, the general parameterization was first tackled in \cite{bai2023regret} where a regret of $\tilde{\mathcal{O}}(T^{3/4})$ was achieved. Later \cite{ganesh2024variance} improved it to $\tilde{\mathcal{O}}(\sqrt{T})$. The state-of-the-art results for discounted reward unconstrained and constrained MDPs are achieved by \cite{mondal2023improved} and \cite{mondal2024sample} respectively. While the unconstrained result achieves the theoretical lower bound, the constrained one does not. In summary, there are ample opportunities to improve the state-of-the-art results of both discounted and average-reward CMDPs.
 
\begin{table*}[t]
    \centering
    \resizebox{0.99\textwidth}{!}
    {
	\begin{tabular}{|c|c|c|c|c|}
		\hline
		Algorithm & Regret & Violation & Model-free & Setting\\
		\hline
            Algorithm 1 in \citep{chen2022learning} & $\tilde{\mathcal{O}}\bigg(\sqrt{T}\bigg)$  & $\tilde{\mathcal{O}}\bigg(\sqrt{T}\bigg)$  & No & Tabular\\
		\hline
            Algorithm 2 in \citep{chen2022learning} & $\tilde{\mathcal{O}}\bigg(T^{2/3}\bigg)$ & $\tilde{\mathcal{O}}\bigg(T^{2/3}\bigg)$ & No & Tabular\\
		\hline
            UC-CURL and PS-CURL \citep{agarwal2022concave} & $\tilde{\mathcal{O}}\bigg(\sqrt{T}\bigg)$ & $0$ & No & Tabular\\
		\hline
            Algorithm 2 in \citep{ghosh2022achieving} & $\tilde{\mathcal{O}}\bigg((dT)^{3/4}\bigg)$ & $\tilde{\mathcal{O}}\bigg((dT)^{3/4}\bigg)$ & No & Linear MDP\\
		\hline
            Algorithm 3 in \citep{ghosh2022achieving} & $\tilde{\mathcal{O}}\bigg(\sqrt{T}\bigg)$ & $\tilde{\mathcal{O}}\bigg(\sqrt{T}\bigg)$ & No & Linear MDP\\
		\hline
            Triple-QA \citep{wei2022provably} & $\tilde{\mathcal{O}}\bigg(T^{5/6}\bigg)$ & $0$ & Yes & Tabular\\
		\hline
            \cite{bai2023learning} & $\tilde{\mathcal{O}}\bigg(T^{\frac{4}{5}}\bigg)$ & $\tilde{\mathcal{O}}\bigg(T^{\frac{4}{5}}\bigg)$ & Yes & General Parameterization\\
		\hline
	\end{tabular}
    }
    \caption{This table summarizes the different model-based and mode-free state-of-the-art algorithms available in the literature for average reward CMDPs. We note that the presented algorithm in \cite{bai2023learning}  is the first to analyze the regret and constraint violation for average reward CMDP with general parametrization. Here the parameter $d$ refers to the dimension of the feature map for linear MDPs.}
    \label{table2}
    \vspace{-0.25in}
\end{table*}
 
We note that the proposed algorithm in this chapter uses the knowledge of mixing time. Recently, a study for removing such dependence has been conducted in the absence of constraints. In order to do that, Multi-level Monte-Carlo (MLMC) gradient estimator is incorporated in \cite{patel2024global}, where the global convergence rate of $\tilde{\mathcal{O}}(T^{-1/4})$ is derived. The result has been further extended to the optimal global convergence rate of $\tilde{\mathcal{O}}(T^{-1/2})$ in \cite{ganesh2024ranac}, where momentum-based acceleration and Natural Actor Critic based approaches are used. However, we note that these approaches are for unconstrained setup, while extension to constrained setup is open. 

We further note that the presentation in this chapter considers linear utility function with linear constraints. For the discounted model-free setup, these assumptions have been relaxed. The impact of concave utility in the parametrized setup was studied in \cite{bai2022joint}. The problem with concave objectives and convex constraints was studied for tabular setup in \cite{bai2022achievingc}. Other related works include \cite{bai2022achieving,bai2023achieving}, which studied constrained MDP for tabular and parametrized setup, respectively.   Extension of the framework in this chapter to concave utility and convex constraints is an open direction with average rewards. 

Finally, we note that the algorithm developed in this chapter uses several parameters that might not be easily estimated (e.g., the Lipschitz constant). Designing a parameter-free algorithm for CMDPs is an important avenue to explore in the future. 


\section{Some Auxiliary Lemmas for the Proofs}
\begin{lemma}
    \label{lemma_aux_2}
    \citep[Lemma 14]{wei2020model} For any ergodic MDP with mixing time $t_{\mathrm{mix}}$, the following holds $\forall (s, a)$, $\forall \pi$. $\forall g\in\{r, c\}$.
    \begin{align*}
        (a) |V_g^{\pi}(s)|\leq 5 t_{\mathrm{mix}},~~
        (b) |Q_g^{\pi}(s, a)|\leq 6 t_{\mathrm{mix}}
    \end{align*}
\end{lemma}

\begin{lemma}
    \label{lemma_aux_3}
    \citep[Corollary 13.2]{wei2020model} Let $\delta^{\pi}(\cdot, T)$ be defined as written below for an arbitrary policy $\pi$. 
    \begin{align}
        \label{def_error_aux_1}
        \delta^{\pi}(s, T) \triangleq \sum_{t=N}^{\infty}\norm{(P^{\pi})^t({s,\cdot}) - d^{\pi}}_1, \forall s ~\text{where} ~N=4t_{\mathrm{mix}}(\log_2 T)
    \end{align}
    
    If $t_{\mathrm{mix}}<T/4$, we have the following inequality $\forall s$: $\delta^{\pi}(s, T)\leq \frac{1}{T^3}$.
\end{lemma}

\begin{lemma}
    \label{lemma_aux_4}
    \citep[Lemma 16]{wei2020model} Let $\mathcal{I}=\{t_1+1,t_1+2,\cdots,t_2\}$ be a certain period of an epoch $k$ of Algorithm \ref{alg:estQ} with length $N$. Then for any $s$, the probability that the algorithm never visits $s$ in $\mathcal{I}$ is upper bounded by
    \begin{equation}
        \left(1-\frac{3d^{\pi_{\theta_k}}(s)}{4}\right)^{\left\lfloor\frac{\lfloor \mathcal{I}\rfloor}{N}\right\rfloor}
    \end{equation}
\end{lemma}

\begin{lemma}
    \label{lemma_aux_5}
    \citep[Lemma 15]{wei2020model} The difference of the values of the function $J_g$, $g\in\{r, c\}$ at policies $\pi$ and $\pi'$, is
    \begin{equation}
        J_g^{\pi}-J_g^{\pi'}=\sum_{s}\sum_{a}d^{\pi}(s)(\pi(a|s)-\pi'(a|s))Q_g^{\pi'}(s,a)
    \end{equation}
\end{lemma}

\begin{lemma}
    \label{lemma_aux_6}
    \citep[Lemma 7]{wei2020model} The term $\hat{J}_c(\theta)$ for any $\theta\in\Theta$ is a good estimator of $J_c(\theta)$, which means
    \begin{equation}
        \big|\mathbf{E}[\hat{J}_c(\theta)]-J_c(\theta)\big|\leq \frac{1}{T^2}
    \end{equation}
\end{lemma}

\begin{lemma}
    \label{lemma_aux_7}
    \citep[Lemma A.6]{dorfman2022adapting} Let  $\theta\in\Theta$ be a policy parameter. Fix a trajectory $z=\{(s_t, a_t, r_t, s_{t+1})\}_{t\in\mathbb{N}}$ generated by following the policy $\pi_{\theta}$ starting from some initial state $s_0\sim\rho$. Let, $\nabla L(\theta)$ be the gradient that we wish to estimate over $z$, and $l(\theta, \cdot)$ is a function such that $\mathbf{E}_{z\sim d^{\pi_{\theta}}, \pi_{\theta}}l(\theta, z)=\nabla L(\theta)$. Assume that $\norm{l(\theta, z)}, \norm{\nabla L(\theta)}\leq G_L$, $\forall \theta\in\Theta$, $\forall z\in \mathcal{S}\times \mathcal{A}\times \mathbb{R}\times \mathcal{S}$. Define $l^{Q}=\frac{1}{Q}\sum_{i=1}^Q l(\theta, z_i)$. If $P=2t_{\mathrm{mix}}\log T$, then the following holds as long as $Q\leq T$,
    \begin{align}
        \mathbf{E}\left[\norm{l^{Q}-\nabla L(\theta)}^2\right]\leq \mathcal{O}\left(G_L^2\log\left(PQ\right)\dfrac{P}{Q}\right)
    \end{align}
\end{lemma}

\begin{lemma}[Strong duality]
    \label{lem.duality}
    For convenience, we rewrite the unparameterized problem below.
    \begin{equation}\label{eq:rewrite_unparameterized}
	\begin{aligned}
		\max_{\pi\in\Pi} ~& J_r^{\pi} ~~
		\text{s.t.} ~ J_c^{\pi}\leq 0
	\end{aligned}
    \end{equation} 
	
    Define $\pi^*$ as the optimal solution to the above problem. Define the associated dual function as
    \begin{equation}
        J_D^{\lambda}\triangleq\max_{\pi\in\Pi} J_r^{\pi}-\lambda J_c^{\pi}
    \end{equation}
    and denote $\lambda^*=\arg\min_{\lambda\geq 0} J_D^{\lambda}$. We have the following strong duality property for the unparameterized problem.
    \begin{equation}\label{eq:duality}
	J_r^{\pi^*} = J_D^{\lambda^*} 
    \end{equation}	
\end{lemma}

Although the strong duality holds for the unparameterized problem, the same is not true for parameterized class $\{\pi_\theta|\theta\in \Theta\}$. To formalize this statement, define the dual function associated with the parameterized problem as follows.
\begin{equation}
    J_{D,\Theta}^{\lambda}\triangleq\max_{\theta\in \Theta} J_r(\theta)-\lambda J_c(\theta)
\end{equation}
and denote $\lambda_\Theta^*=\arg\min_{\lambda\geq 0} J_{D,\Theta}^{\lambda}$. The lack of strong duality states that, in general, $J_{D, \Theta}^{\lambda_{\Theta}^*}\neq J_r(\theta^*)$ where $\theta^*$ is a solution of the parameterized constrained optimization \eqref{eq:def_constrained_optimization}. However, $\lambda_\Theta^*$, as we demonstrate below, must obey some restrictions.
\begin{lemma}[Bound on $\lambda_{\Theta}$]
    \label{lem.boundness}
    Under Assumption \ref{ass_slater}, the optimal dual variable for the parameterized problem is bounded as
    \begin{equation}
        0 \leq \lambda_\Theta^* \leq \frac{J_r^{\pi^*}-J_r(\bar{\theta})}{\delta}\leq \dfrac{1}{\delta}
    \end{equation}
\end{lemma}

\begin{proof}
    The proof follows the approach in \citep[Lemma 3]{ding2023convergence}, but is revised to the general parameterization setup.	Let $\Lambda_a\triangleq\{ \lambda\geq 0\,\vert\, J_{D,\Theta}^\lambda \leq a \}$ be a sublevel set of the dual function for $a\in\mathbb{R}$. If $\Lambda_a$ is non-empty, then for any $\lambda \in\Lambda_a$, 
    \begin{equation}
        a\geq J_{D,\Theta}^\lambda\geq J_r(\bar{\theta})-\lambda J_c(\bar{\theta})\geq J_r(\bar{\theta})+\lambda \delta
    \end{equation}
    where $\bar{\theta}$ is a Slater point in Assumption \ref{ass_slater}. Thus, $\lambda \leq (a -J_r(\bar{\theta}))/\delta$.	If we take $a= J_{D,\Theta}^{\lambda_\Theta^*}\leq J_{D,\Theta}^{\lambda^*} \leq J_D^{\lambda^*}=J_r^{\pi^*}$, then $\lambda_\Theta^*\in \Lambda_a$, which proves the Lemma. The last inequality holds since $J_r^{\pi}\in [0,1]$ for any $\pi$.
\end{proof}

Since the above inequality holds for any $\Theta$, we also have, $0\leq \lambda^*\leq \frac{1}{\delta}$. Let $v(\tau)\triangleq\max_{\pi\in\Pi}\{J_r^\pi|J_c^\pi\leq -\tau\}$. Using the strong duality property of the unparameterized problem \eqref{eq:rewrite_unparameterized}, we establish the following property of the function, $v(\cdot)$.

\begin{lemma}
    \label{lem:bridge}
    If Assumption \ref{ass_slater} holds, the following is true $\forall\tau\in\mathbb{R}$,
    \begin{equation}
	v(0)-\tau\lambda^* \geq	v(\tau)
    \end{equation}
\end{lemma}

\begin{proof}
    By the definition of $v(\tau)$, we get $v(0) = J_r^{\pi^*}$. With a slight abuse of notation, denote $J_{\mathrm{L}}(\pi,\lambda)=J_r^{\pi}-\lambda J_c^{\pi}$. By the strong duality stated in Lemma \ref{lem.duality}, we have the following for any $\pi\in\Pi$.
    \begin{equation}
        J_{\mathrm{L}}(\pi,\lambda^*)\leq \max_{\pi\in\Pi} J_{\mathrm{L}}(\pi,\lambda^*)\overset{Def}=J_D^{\lambda^*}\overset{\eqref{eq:duality}}=J_r^{\pi^*}=v(0)
    \end{equation}
    Thus, for any $\pi\in\{ \pi\in\Pi \,\vert\,J_c^{\pi} \leq -\tau \}$,
    \begin{equation}
	\begin{aligned}
		v(0)-\tau\lambda^*&\geq J_L(\pi,\lambda^*)-\tau\lambda^*\\
		&=J_r^{\pi}-\lambda^*(J_c^\pi+\tau) \geq J_r^{\pi}
	\end{aligned}
    \end{equation}
    Maximizing the R.H.S of this inequality over $\{ \pi\in\Pi \vert J_{c}^{\pi}\leq -\tau \}$ yields
    \begin{equation}
	\label{eq.opt1}
	v(0)- \tau\lambda^* \geq v(\tau)
    \end{equation}
    This completes the proof of the lemma.
\end{proof}

We note that a similar result was shown in \citep[Lemma 15]{bai2023provably}. However, the setup of the stated paper is different from that of ours. Specifically, \cite{bai2023provably} considers a tabular setup with peak constraints. Note that Lemma \ref{lem:bridge} has no direct connection with the parameterized setup since its proof uses strong duality and the function, $v(\cdot)$, is defined via a constrained optimization over the entire policy set, $\Pi$, rather than the parameterized policy set. Interestingly, however, the relationship between $v(\tau)$ and $v(0)$ leads to the lemma stated below which turns out to be pivotal in establishing regret and constraint violation bounds in the parameterized setup.

\begin{lemma}\label{lem.constraint}
    Let Assumption \ref{ass_slater} hold. For any constant $C\geq 2\lambda^*$, if there exists a $\pi\in\Pi$ and $\zeta>0$ such that $J_r^{\pi^*}-J_r^{\pi}+C[J_c^{\pi}]\leq \zeta$, then 
    \begin{equation}
        J_c^{\pi}\leq 2\zeta/C
    \end{equation}
\end{lemma}
\begin{proof}
    Let $\tau = -J^{\pi}_c$. Using the definition of $v(\tau)$, one can write,
    \begin{equation}\label{eq.opt2}
        J_r^{\pi}\leq v(\tau)
    \end{equation}
    Combining Eq. \eqref{eq.opt1} and \eqref{eq.opt2}, we obtain the following.
    \begin{equation}
        J_r^{\pi}-J_r^{\pi^*}\leq v(\tau)-v(0)\leq -\tau\lambda^*
    \end{equation}
    The condition in the Lemma leads to,
    \begin{equation}
        (C - \lambda^*) (-\tau) = {\tau} \lambda^*+C (-\tau) \leq J_r^{\pi^*}-J_r^{\pi}+C [J_c^{\pi}]\leq \zeta
    \end{equation}
    Finally, we have,
    \begin{equation}
        -\tau\leq \frac{\zeta}{C-\lambda^*}\leq\frac{2\zeta}{C}
    \end{equation}
    which completes the proof.
\end{proof}

\chapter{Beyond Ergodic MDPs}\label{chpt:beyerg}

The previous chapters considered the case where the MDP was ergodic. However, the key condition needed for efficient guarantees is that the underlying MDP is at least weakly communicating \cite{bartlett2009regal,jaksch2010near}. This chapter provides some results for weakly communicating MDPs. The main results in this chapter are for model-based reinforcement learning (in Section \ref{bey:mb}), where a regret guarantee of $\Tilde{\mathcal{O}}(T^{2/3})$ is derived. These guarantees are far from the theoretical lower bound, leaving significant room for improvement. Such possibilities are discussed in Section \ref{bey:notes}.


\section{Algorithm for Model-Based RL}\label{bey:mb}

Similar to the previous two chapters, we consider a constrained Markov Decision Process (CMDP) characterized as $\mathcal{M} = (\mathcal{S}, \mathcal{A}, r, c, P)$ where $\mathcal{S}$, $\mathcal{A}$ are the state and action spaces with sizes $S$ and $A$ respectively, $r:\mathcal{S}\times\mathcal{A}\rightarrow [0, 1]$ is the reward function, $c:\mathcal{S}\times \mathcal{A}\rightarrow [-1, 1]$ is the cost function, and finally, $P:\mathcal{S}\times \mathcal{A}\rightarrow \Delta(\mathcal{S})$ is the transition kernel where $\Delta(\cdot)$ is the probability simplex over its argument set. A policy is defined to be a map of the following form $\pi:\mathcal{S}\rightarrow \Delta(\mathcal{A})$. For a given policy, $\pi$, the long-term reward and cost functions, $J_{g}^{\pi, P}$, $g\in\{r, c\}$ are defined as follows.
\begin{align}
    \label{eq_chap_4_def_j_pi_P_g}
    J^{\pi, P}_g(s) = \underset{\tau\rightarrow\infty}{{\lim\inf}} ~\mathbf{E}_{\pi, P}\left[\dfrac{1}{\tau}\sum_{t=1}^\tau g(s_t, a_t) \bigg| s_1=s\right]
\end{align}
where the expectation is computed over all trajectories generated by the policy $\pi$ and the kernel $P$ from the initial state $s$. As stated earlier, this chapter considers the underlying CMDP to be weakly communicating, which essentially means the state space $\mathcal{S}$ can be segregated into two categories: the states in the first group are transient under all (stationary) policies while any two states in the second group are communicating under some policy. Note that the weakly communicating class subsumes the ergodic class. One problem with this generalization is that the notion of a stationary distribution may no longer be valid and the long-term reward and costs are generally dependent on the initial state (unlike that in the ergodic case). \cite{puterman1994markov} demonstrated that for any $s\in\mathcal{S}$, there exists $\optpi$ that solves the following constrained optimization
\begin{align}
    \label{eq_last_chap_const_opt}
    \max_{\pi} J^{\pi, P}_r(s)~~\text{s. t.}~J^{\pi, P}_c(s)\leq 0
\end{align}
and ensures that the long-term reward and cost associated with $\optpi$ are independent of $s$ i.e., $J^{\optpi, P}_g(s) = \optJ_g$, for some constants $\optJ_g$, $g\in\{r, c\}$.
Additionally, \citep[Theorem 8.2.6]{puterman1994markov} also showed that for any $g\in \mathbb{R}^{\mathcal{S}\times\mathcal{A}}$ (i.e., $g$ need not be restricted to reward or cost function), there exists a \emph{bias function} $q^{\pi, P}_g\in\fR^{\SA}$ that obeys the Bellman equation $\forall (s, a)\in \calS \times \calA$,
\begin{equation}\label{eq:Bellman}
    q^{\pi, P}_g (s, a)+ J^{\pi, P}_g(s) = g(s,a) + \E_{s'\sim P(s, a)}[v^{\pi, P}_g(s')],
\end{equation}
where $v^{\pi, P}_g(s) = \sum_{a\in\calA}\pi(a|s)q^{\pi,P}_g(s, a)$,
and also $J^{\pi,P}_{g'}(s)=0$, $g'=q^{\pi,P}_g$, $\forall s$. The functions $q$ and $v$ are analogous to the well-known $Q$-function and state value function for the discounted or the finite-horizon setting. 

In the rest of this chapter, we utilize the following notations. For any $f\in\fR^{\calS}$, we define its span as $\sp(f)=\max_{s\in\calS}f(s) - \min_{s\in\calS}f(s)$. When there is no confusion, we simplify the notations $J^{\pi, P}_g$,  $q^{\pi, P}_g$, and $v^{\pi, P}_g$ by dropping the dependence on $P$. Given a policy $\pi$ and a transition kernel $P$, we define the $\pi$-induced transition kernel $P^{\pi}$ such that $P^{\pi}(s, s') = \sum_a\pi(a|s)P(s'|s, a)$, $\forall (s, s')$. For any $\epsilon\in(0, 1)$, $\optpieps$ is given in the same way as $\optpi$ but with the constraint threshold $-\epsilon$. Let $\optJeps_g$ denote $J^{\optpieps}_g$ (that is $s$-independent as mentioned before), $\forall g\in\{r, c\}$.

Weakly communicating CMDPs imposes challenges in learning as compared to ergodic CMDPs.  Specifically, there is no uniform bound for $\sp(v^{\pi}_g)$, $g\in\{r,c\}$ for all $\pi$ (while in the ergodic case, they are bounded by $\tilo{\tmix}$ where $\tmix$ is the mixing time defined in the earlier chapter). It is also unclear how to obtain an accurate estimate of a policy's bias function as in ergodic MDPs, which is an important step for a policy optimization algorithm. In this section, we discuss the approach given by \cite{chen2022learning} to solve the above problem. 

The algorithm runs by dividing $T$ into $K$ episodes, each of length $H$. Each episode considers the problem of finding an optimal non-stationary policy for an episodic finite-horizon MDP through the lens of occupancy measure, in which expected reward and cost are both linear functions and easy to optimize over. Concretely, consider a fixed starting state $s$, a non-stationary policy $\pi\in(\Delta_{\calA})^{\calS\times[H]}$ whose behavior can change in different steps within an episode, and an inhomogeneous state transition kernel $P=\{P_{h}\}_{h\in [H]}$ where $P_{h}\in (\Delta_\calS)^{\SA}$ specifies the probability kernel at step $h$. The corresponding occupancy measure $\nu^{\pi,P}_s\in[0, 1]^{\SA\times[H]\times\calS}$ is then such that $\nu^{\pi,P}_s(s', a, h, s'')$ is the probability of visiting state $s'$ at step $h$, taking action $a$, and then transiting to state $s''$, if the learner starts from state $s$, executes $\pi$ for the next $H$ steps, and the transition kernel is $P$. Conversely, a function $\nu \in[0, 1]^{\SA\times[H]\times\calS}$ is defined to be an occupancy measure with respect to a starting state $s$, some policy $\pi_\nu$, and transition $P_\nu$ if and only if it satisfies the following conditions. 
\begin{enumerate}
    \item Initial state is $s$: $\sum_a\sum_{s''}\nu(s', a, 1, s'')=\Ind\{s'=s\}$.
    \item Total mass is $1$: $\sum_{s'}\sum_a\sum_{s''}\nu(s', a, h, s'')=1, \;\forall h$.
    \item Flow conservation: $\sum_{s''}\sum_{a}\nu(s'', a, h, s')=\sum_{a}\sum_{s''}\nu(s', a, h+1, s'')$ for all $s'\in\calS, h\in[H-1]$.
\end{enumerate}

Let the set of all such $\nu$ be $\calV_s$.
For notational convenience, for a $\nu \in \calV_s$, we define $\nu(s', a, h)=\sum_{s''}\nu(s', a, h, s'')$, $\nu(s', a)=\sum_h\nu(s', a, h)$, and $\nu(s', h)=\sum_a\nu(s', a, h)$.
Also, note that the corresponding policy $\pi_\nu$ and transition $P_\nu$ can be extracted using
$\pi_{\nu}(a|s', h)=\frac{\nu(s', a, h)}{\nu(s', h)}$ and $P_{\nu, s', a, h}(s'')=\frac{\nu(s', a, h, s'')}{\nu(s', a, h)}$. Observe that $\calV_s$ is a convex polytope with polynomial constraints. If one enforces $P_\nu$ to be homogeneous across different steps of an episode, $\calV_s$ would become non-convex. This forces us to consider inhomogeneous transitions even though the true transition is indeed homogeneous.

\DontPrintSemicolon 
\setcounter{AlgoLine}{0}
\begin{algorithm}[t]
    \caption{Finite Horizon Approximation for CMDP}
    \label{alg:weak}
    \textbf{Define:} $H=\lceil(T/S^2A)^{1/3} \rceil$, $K=T/H$.
	
    \For{$k=1,\ldots,K$}{
	Observe current state $s^k_1 = s_{(k-1)H+1}$.
		
	Compute occupancy measure:
	\begin{equation}\label{eq:OPT1}
            \nu_k=\argmax_{\nu\in\calV_{k, s^k_1}: \inner{\nu}{c}\leq  \sp^\star_c}\inner{\nu}{r},
	\end{equation}
        where $\calV_{k, s}=\{\nu\in\calV_s: P_{\nu}\in\calP_k\}$ (see \pref{eq:conf}). 
		
	Extract policy $\pi_k=\pi_{\nu_k}$ from $\nu_k$.
		
	\For{$h=1,\ldots,H$}{
            Play action $a^k_h\sim\pi_k(\cdot|s^k_h, h)$ and transit to $s^k_{h+1}$.
		}
	}	
\end{algorithm}

The entire procedure is described in \pref{alg:weak}. At the beginning of the $k$th episode, the learner observes the current state $s^k_1$ and finds an occupancy measure $\nu_k\in \calV_{k, s^k_1}$ that maximizes $\inner{\nu}{r} = \sum_{s,a} \nu(s,a) r(s,a)$ while satisfying  $\inner{\nu}{c}\leq  \sp^\star_c$ where $\sp^\star_c = \sp(v^{\optpi}_c)$, which is assumed to be known to the learner. Here, $\calV_{k, s^k_1}\subset \calV_{s^k_1}$ is such that $P_{\nu}$, $\forall \nu \in \calV_{k, s^k_1}$ lies in a standard Bernstein-type confidence set $\calP_k$ defined as
\begin{align}
    \calP_k &= \Big\{ P'=\{P'_{h}\}_{h\in[H]}, P'_{h}\in(\Delta_{\calS})^{\SA}:\abr{P'_{h}(s'|s, a) - \P_{k}(s'|s, a)}\notag\\
    &\leq 4\sqrt{\P_{k}(s'|s, a)\alpha_k(s, a) } + 28\alpha_k(s, a), \forall (s, a) \Big\},\label{eq:conf}
\end{align}
where $\P_{k}(s'|s, a)=\frac{\N_k(s, a, s')}{\Np_k(s, a)}$ is the empirical transition, $\N_k(s, a, s')$ is the number of visits to triplet $(s, a, s')$ before episode $k$, $\alpha_k(s, a)=\frac{\iota'}{\Np_k(s, a)}$, $\Np_k(s,a)=\max\{1,\N_k(s, a)\}$, $\N_k(s, a)$ is the number of visits to $(s, a)$ before episode $k$, 
and $\iota'=\ln\frac{2SAT}{\delta}$. $\calP_k$ is constructed in a way such that it contains the true transition with high probability (\pref{lem:conf}).
With $\nu_k$, we follow $\pi_k=\pi_{\nu_k}$ extracted from $\nu_k$ for the next $H$ steps.
 
Note that the key optimization problem~\eqref{eq:OPT1} in this algorithm can be efficiently solved since the objective function is linear and the domain is a convex polytope with polynomial constraints, thanks to the use of occupancy measures. We now state the main guarantee of \pref{alg:weak}.
\begin{theorem}
    \label{thm:weak}
    \pref{alg:weak} ensures the following regret and constraint violation bounds (defined in a manner similar to that in the previous chapter) with probability at least $1-10\delta$.
    \begin{align*}
        R_T &= \tilO{ (1+\sp^\star_r)(S^2A)^{1/3}T^{2/3} },\\
        C_T &= \tilO{ (1+\sp^\star_c)(S^2A)^{1/3}T^{2/3} }.
    \end{align*}
\end{theorem}

The rest of this section provides the key steps to the result, while for the detailed proof, the readers are referred to \cite{chen2022learning}. 

For any non-stationary policy $\pi\in(\Delta_{\calA})^{\calS\times[H]}$, and an inhomogeneous transition $P=\{P_{h}\}_{h\in [H]}$, define the following $\forall h\in [H]$.
\begin{align}
    &V^{\pi, P}_{g, h}(s)=\mathbf{E}_{\pi, P}\left[\sum_{h'=h}^H g(s_{h'}, a_{h'})\bigg|s_h=s\right],\forall s\\
    & Q^{\pi, P}_{g, h}(s, a)=\mathbf{E}_{\pi, P}\left[\sum_{h'=h}^H g(s_{h'}, a_{h'})\bigg|s_h=s, a_h=a\right], \forall (s, a)
\end{align}
Additionally, $V^{\pi,P}_{g, H+1}(s)=Q^{\pi,P}_{g, H+1}(s, a)=0$, $\forall (s, a)$. Let $\tilP=\{\tilP_{ h}\}_{h\in [H]}$ be such that $\tilP_{h}(\cdot|s, a) = P(\cdot|s,a)$ (the true transition function), $\forall h$. We ignore the dependency on $\tilP$ for simplicity when there is no confusion e.g., $V^{\pi}_{g, h}$ denotes $V^{\pi,\tilP}_{g, h}$. For a stationary policy $\pi\in(\Delta_{\calA})^{\calS}$, let $\tilpi$ be the policy that mimics $\pi$ in the finite-horizon setting, that is, $\tilpi_h(\cdot|s)=\pi(\cdot|s)$, $\forall h$.

We first show that the true transition lies in the transition confidence sets with high probability and provide some key related lemmas.
\begin{lemma}
    \label{lem:conf}
    With probability at least $1-\delta$, $\tilP\in\calP_k$,$\forall k$.
\end{lemma}
\begin{proof}
    For any $(s, a)\in\SA, s'\in\calS$, by \pref{lem:bernstein} and $N_{K+1}(s, a)\leq T$, we have with probability at least $1-\frac{\delta}{S^2A}$,
    \begin{align*}
        \abr{P(s'|s, a) - \P_{k}(s'|s, a)} \leq 4\sqrt{\P_{k}(s'|s, a)\alpha_k(s, a)} + 28\alpha_k(s, a).
    \end{align*}
    Applying a union bound over $(s, a)\in\SA$, $s'\in\calS$ and utilizing $\tilP_{h}=P$, the statement is proved.
\end{proof}

\begin{lemma}
    \label{lem:conf eps}
    If $\tilP\in\calP_k$, $\abr{P'_{h}(s'|s, a) - P(s'|s, a)} \leq 8\sqrt{P(s'|s, a)\alpha_k(s, a)} + 136\alpha_k(s, a) \triangleq \epsilon^{\star}_k(s, a, s')$, $\forall P'\in\calP_k$, $\forall h$.
\end{lemma}
\begin{proof}
    Using $\tilP\in\calP_k$, we get the following $\forall (s, a, s')$:
    \begin{align*}
        \P_{k}(s'|s, a) \leq P(s'|s, a) + 4\sqrt{\P_{k}(s'|s, a)\alpha_k(s, a)} + 28\alpha_k(s, a).
    \end{align*}
    Applying $x^2\leq a x+b\implies x\leq a+\sqrt{b}$ with $b=P(s'|s, a)+28\alpha_k(s, a)$, and $a=4\sqrt{\alpha_k(s, a)}$, we have
    \begin{align*}
        \sqrt{\P_{k}(s'|s, a)} &\leq 4\sqrt{\alpha_k(s, a)}+\sqrt{P(s'|s, a)+28\alpha_k(s, a)}\\
        &\leq \sqrt{P(s'|s, a)} + 10\sqrt{\alpha_k(s, a)}
    \end{align*}
    Substituting this, the right-hand side of \pref{eq:conf} can be bounded as, 
    $$4\sqrt{\P_{k}(s'|s, a)\alpha_k(s, a) } + 28\alpha_k(s, a)\leq 4\sqrt{P(s'|s, a)\alpha_k(s, a)} + 68\alpha_k(s, a).$$
    Using $\tilP, P'\in\calP_k$, \pref{eq:conf}, and  $|P'_{h}(s'|s, a) - P(s'|s, a)|\leq |P'_{h}(s'|s, a) - P_{k}(s'|s, a)|+|\P_{k}(s'|s, a)-P(s'|s, a)|$, the statement is proved.
\end{proof}

\begin{lemma}
    \label{lem:VJ}
    For a stationary policy $\pi\in(\Delta_{\calA})^{\calS}$, if $J^{\pi}_g(s)=J^{\pi}_g$, $\forall s\in\calS$, we have $|V^{\tilpi}_{g,h}(s) - (H-h+1)J^{\pi}_g|\leq \sp(v^{\pi}_g)$, $\forall s\in\calS$ and $\forall h\in[H]$.
\end{lemma}
\begin{proof}
    For any state $s$ and $h\in[H]$, we have:
    \begin{align*}
	&V^{\tilpi}_{g,h}(s) - (H-h+1)J^{\pi}_g \\
        &=\mathbf{E}_{\tilpi, \tilP}\sbr{\left. \sum_{h'=h}^H(g(s_{h'}, a_{h'}) - J^{\pi}_g) \right|s_h=s}\\ 
        &= \mathbf{E}_{\tilpi, \tilP}\sbr{\left. \sum_{h'=h}^H (q^{\pi}_g(s_{h'}, a_{h'}) - \sum_{s'}P(s'|s_{h'}, a_{h'})v^{\pi}_g(s') \right| s_h=s} \tag{\ref{eq:Bellman}}\\
        &= \mathbf{E}_{\tilpi, \tilP}\sbr{\left. \sum_{h'=h}^H (v^{\pi}_g(s_{h'}) - v_g^{\pi}(s_{h'+1})) \right| s_h=s}\tag{definition of $\tilpi$ and $\tilP$}\\
        &= v^{\pi}_g(s) - \mathbf{E}_{\tilpi, \tilP}\sbr{\left.v^{\pi}_g(s_{H+1})\right| s_h=s}.
    \end{align*}
    Hence, $|V^{\tilpi}_{g,h}(s) - (H-h+1)J_g^{\pi}|\leq \sp(v^{\pi}_g)$ and the proof is complete.
\end{proof}

Next, we show that the occupancy measure corresponding to $\tiloptpi$ and $\tilP$ obeys the constraint of ~\eqref{eq:OPT1} with high probability.
\begin{lemma}
    \label{lem:opt}
    If $\tilP\in\calP_k$, $\nu^{\tiloptpi,\tilP}_{s^k_1}$ lies in the constraint set of \eqref{eq:OPT1}.
\end{lemma}
\begin{proof}
    By \pref{lem:VJ}, we have:
    \begin{align}
        \inner{\nu^{\tiloptpi,\tilP}_{s^k_1}}{c} &= V^{\tiloptpi}_{c, 1}(s^k_1) \leq \sp^{\star}_c + HJ_c^{\optpi}\leq \sp^{\star}_c \label{eq:opt}
    \end{align}
    Then by $\tilP\in\calP_k$, the statement is proved.
\end{proof}

As the last preliminary step, we bound the bias in value function caused by imperfect transition estimation, that is, the difference between $\tilP$ and $P_k = P_{\nu_k}$. The result uses Lemma \ref{lem:sum var}- \ref{lem:val diff H}, which are described after the main proof outline. In the following, we utilize the notation that for any $P\in\Delta(\mathcal{S})$ and $V\in\mathbb{R}^\mathcal{S}$, $P\cdot V=\sum_{s\in\mathcal{S}}P(s)V(s)$. Moreover, $\fV(P, V)$ denotes the variance of $V(\cdot)$ with respect to $P(\cdot)$, i.e., $\fV(P, V) = P\cdot (V-P\cdot V)^2$ where subtraction and square operations are element-wise.

\begin{lemma}
    \label{lem:V diff}
    With probability at least $1-4\delta$, we have $|\sumk (V^{\pi_k,\tilP}_{g, 1}(s^k_1) - V^{\pi_k,P_k}_{g,1}(s^k_1))| = \tilo{\sqrt{S^2AH^2K} + H^2S^2A}$, $\forall g\in\{r, c\}$.
\end{lemma}

\begin{proof}
    We condition on the event of \pref{lem:conf}, which happens with probability at least $1-\delta$. Note that with probability at least $1-\delta$:
    \begin{align*}
        &\abr{\sumk (V^{\pi_k, P_k}_{g, 1}(s^k_1) - V^{\pi_k}_{g, 1}(s^k_1)) } \\
        &= \abr{\sumk\mathbf{E}_{\pi_k, P}\sbr{\sumh (P_{k, h} - P)(s^k_h, a^k_h)\cdot V^{\pi_k, P_k}_{g, h+1}}} \tag{\pref{lem:val diff H}}\\ 
        &\leq \sumk\mathbf{E}_{\pi_k, P}\sbr{\sumh \abr{(P_{k, h} - P)(s^k_h, a^k_h)\cdot V^{\pi_k, P_k}_{g, h+1}}} \\
        &\leq 2\sumk\sumh \abr{(P_{k, h} - P)(s^k_h, a^k_h) \cdot V^{\pi_k, P_k}_{g, h+1}} + \tilO{H^2}\tag{\pref{lem:e2r}}\\
        &= \tilO{\sqrt{S^2A\sumk\sumh\fV(P(s^k_h, a^k_h), V^{\pi_k, P_k}_{g, h+1})} + H^2S^2A } \tag{\pref{lem:dPV}}
    \end{align*}
    \pref{lem:sum var} establishes that the following happens with probability at least $1-2\delta$ if $H=(T/S^2A)^{1/3}$.
    \begin{align*}
        &\abr{\sumk V^{\pi_k}_{g,1}(s^k_1) - V^{\pi_k,P_k}_{g, 1}(s^k_1)} = \tilO{\sqrt{S^2AH^2K} + H^2S^2A}
    \end{align*}
    This completes the proof.
\end{proof}

We are now ready to prove \pref{thm:weak}.
\begin{proof}[\pfref{thm:weak}]
    We decompose $R_T$ into three terms:
    \begin{align*}
        &R_T = \sumt \optJ - r(s_t, a_t) = \sumk\rbr{H\optJ - \sumh r(s^k_h, a^k_h)}\\ 
        &= \sumk\rbr{H\optJ - V^{\tiloptpi}_{r, 1}(s^k_1)} + \sumk\rbr{V^{\tiloptpi}_{r, 1}(s^k_1) - V^{\pi_k}_{r, 1}(s^k_1)}\\ 
        &\qquad + \sumk\rbr{V^{\pi_k}_{r, 1}(s^k_1) - \sumh r(s^k_h, a^k_h)}
    \end{align*}
    The first term above is upper bounded by $K\sp^\star_r$ by \pref{lem:VJ}. For the second term, by \pref{lem:opt} and \pref{lem:V diff}, we have
    \begin{align*}
        &\sumk\rbr{V^{\tiloptpi}_{r, 1}(s^k_1) - V^{\pi_k}_{r, 1}(s^k_1)} \leq \sumk\rbr{V^{\pi_k, P_k}_{r, 1}(s^k_1) - V^{\pi_k}_{r, 1}(s^k_1)}\\
        &= \tilO{\sqrt{S^2AH^2K} + H^2S^2A}.
    \end{align*}
    The last term is of order $\tilo{H\sqrt{K}}$ by Azuma's inequality (\pref{lem:azuma}). Using the definition of $H$ and $K$, we arrive at
    \begin{align*}
	R_T = \tilO{ (1+\sp^\star_r))(S^2A)^{1/3}T^{2/3}}.
    \end{align*}
    For constraint violations, we decompose $C_T$ as:
    \begin{align*}
        &\sumt c(s_t, a_t) = \sumk\rbr{\sumh c(s^k_h, a^k_h) - V^{\pi_k}_{c, 1}(s^k_1)}\\
        &\qquad+ \sumk\rbr{ V^{\pi_k}_{c, 1}(s^k_1) - V^{\pi_k, P_k}_{c, 1}(s^k_1)} + \sumk\rbr{ V^{\pi_k, P_k}_{c, 1}(s^k_1)}
    \end{align*}
    The first term is $\tilo{H\sqrt{K}}$ by Azuma's inequality. The second term is $\tilO{\sqrt{S^2AH^2K} + H^2S^2A}$ by \pref{lem:V diff}. The third term is bounded by $K\sp^\star_c$ due to the constraint $\inner{\nu}{c}\leq \sp^\star_c$ in the optimization problem~\eqref{eq:OPT1}. Using the definition of $H$ and $K$, we get:
    \begin{align*}
	C_T = \tilO{ (1+\sp^\star_c)(S^2A)^{1/3}T^{2/3} }.
    \end{align*}
    This completes the proof.
\end{proof}

We now provide the three auxiliary Lemmas that were used earlier. 
\begin{lemma}
    \label{lem:sum var}
    Under the event of \pref{lem:conf}, for any utility function $g\in[-1, 1]^{\SA}$, with probability at least $1-2\delta$, $\sumk\sumh\fV(P^k_h, V^{\pi_k,P_k}_{g, h+1}) = \tilo{ H^2(K+\sqrt{T}) + H^3S^2A }$ where $P_h^k\triangleq P(s_h^k, a_h^k)$.
\end{lemma}
\begin{proof}
    We decompose the variance into four terms:
    \begin{align*}
        &\sumk\sumh\fV(P^k_h, V^{\pi_k, P_k}_{g, h+1}) = \sumk\sumh \rbr{P^k_h\cdot \left(V^{\pi_k, P_k}_{g, h+1}\right)^2 - \left(P^k_h\cdot V^{\pi_k, P_k}_{g, h+1}\right)^2}\\
        &= \sumk\sumh \rbr{P^k_h\cdot \left(V^{\pi_k, P_k}_{g, h+1}\right)^2 - \left(V^{\pi_k, P_k}_{g, h+1}(s^k_{h+1})\right)^2} \\
        &+ \sumk\sumh \rbr{\left(V^{\pi_k, P_k}_{g, h+1}(s^k_{h+1})\right)^2 - \left(V^{\pi_k, P_k}_{g, h}(s^k_h)\right)^2}\\
        & + \sumk\sumh \rbr{\left(V^{\pi_k, P_k}_{g, h}(s^k_h)\right)^2 - \left(Q^{\pi_k, P_k}_{g, h}(s^k_h, a^k_h)\right)^2} \\
        &+ \sumk\sumh \rbr{\left(Q^{\pi_k, P_k}_{g, h}(s^k_h, a^k_h)\right)^2 - \left(P^k_h\cdot V^{\pi_k, P_k}_{g, h+1}\right)^2}
    \end{align*}
    For the first term, by \pref{lem:freedman}, with probability at least $1-\delta$,
    \begin{align*}
        &\sumk\sumh P^k_h\cdot\left(V^{\pi_k, P_k}_{g, h+1}\right)^2 - \left(V^{\pi_k, P_k}_{g, h+1}(s^k_{h+1})\right)^2 \\
        &=\tilO{\sqrt{\sumk\sumh \fV\left(P^k_h, (V^{\pi_k, P_k}_{g, h+1})^2\right)} + H^2}\\ 
        &=\tilO{H\sqrt{\sumk\sumh \fV\left(P^k_h, V^{\pi_k, P_k}_{g, h+1}\right)} + H^2} \tag{\pref{lem:var XY}}
    \end{align*}
    The second term is upper bounded by $0$ since $V^{\pi_k,P_k}_{g, H+1}(s)=0$ for $s\in\calS$. The third term, with probability at least $1-\delta$, can be upper bounded as follows using the Cauchy-Schwarz inequality and \pref{lem:azuma}.
    \begin{align*}
        &\sumk\sumh \left(V^{\pi_k, P_k}_{g, h}(s^k_h)\right)^2 - \left(Q^{\pi_k, P_k}_{g, h}(s^k_h, a^k_h)\right)^2 \\
        &\leq \sumk\sumh \rbr{\sum_a\pi_k(a|s^k_h, h)\left(Q^{\pi_k, P_k}_{g, h}(s^k_h, a)\right)^2 -  \left(Q^{\pi_k, P_k}_{g, h}(s^k_h, a^k_h)\right)^2}\\
        &= \tilO{H^2\sqrt{T}}
    \end{align*}
    For the fourth term can be upper bounded using $a^2-b^2=(a+b)(a-b)$ and the facts that $\norm{V^{\pi_k,P_k}_{g,h}}_{\infty},\norm{Q^{\pi_k,P_k}_{g,h}}_{\infty}\leq H$:
    \begin{align*}
        &\sumk\sumh \left(Q^{\pi_k, P_k}_{g, h}(s^k_h, a^k_h)\right)^2 - \left(P^k_h\cdot V^{\pi_k, P_k}_{g, h+1}\right)^2 \\
        &\leq 2H\sumk\sumh\abr{Q^{\pi_k, P_k}_{g, h}(s^k_h, a^k_h) - P^k_h\cdot V^{\pi_k, P_k}_{g, h+1}}\\
        &\leq 2H^2K + 2H\sumk\sumh\abr{(P_{k, h}(s^k_h, a^k_h) - P^k_h)\cdot V^{\pi_k,P_k}_{g, h+1}}\\
        & = \tilO{H^2K + H\sqrt{S^2A\sumk\sumh\fV(P^k_h, V^{\pi_k, P_k}_{g, h+1})} + H^3S^2A } \tag{\pref{lem:dPV}}
    \end{align*}
    Putting everything together, we have
    \begin{align*}
        \sumk\sumh &\fV(P^k_h, V^{\pi_k, P_k}_{g, h+1}) = \mathcal{O} \left(H\sqrt{\sumk\sumh \fV(P^k_h, V^{\pi_k, P_k}_{g, h+1})} \right.\\
        &\left.+ H^2(K+\sqrt{T}) + H\sqrt{S^2A\sumk\sumh\fV(P^k_h, V^{\pi_k, P_k}_{g, h+1}) } + H^3S^2A \right).
    \end{align*}
    Solving the quadratic inequality, we obtain $\sumk\sumh\fV(P^k_h, V^{\pi_k, P_k}_{g, h+1}) = \tilo{ H^2(K+\sqrt{T}) + H^3S^2A }$. This completes the proof.
\end{proof}

\begin{lemma}
    \label{lem:dPV}
    If $V_h\in[-B, B]^{\calS},\forall h\in[H]$, and $P_h^k \triangleq P(s_h^k, a_h^k)$, then under the event of \pref{lem:conf}, we have the following.
    \begin{align*}
        &\sumk\sumh\abr{(P_{k, h}(s^k_h,a^k_h)-P^k_h)\cdot V_{h+1}}\\
        &=\tilO{\sqrt{S^2A\sumk\sumh\fV(P^k_h, V_{h+1})} + BHS^2A }
    \end{align*}
\end{lemma}
\begin{proof}
    Let $\Ind_k=\Ind\{(s, a)\in \SA: N_{k+1}(s, a)\leq 2N_k(s, a) \}$ and $z^k_h(s')=V_{h+1}(s') - P^k_h\cdot V_{h+1}$. By \pref{lem:conf eps},
    \begin{align*}
        &\sumk\sumh\abr{(P_{k,h}(s^k_h,a^k_h)-P^k_h)\cdot V_{h+1}} = \sumk\sumh\abr{(P_{k,h}(s^k_h,a^k_h)-P^k_h)\cdot z^k_h} \\
        &\leq \sumk\sumh\min\cbr{B, \sum_{s'}\epsilon^{\star}_k(s^k_h, a^k_h, s')|z^k_h(s')|}\\
        &\leq 2\sumk\sumh\min\cbr{B, \sum_{s'}\epsilon^{\star}_{k+1}(s^k_h, a^k_h, s')|z^k_h(s')|} + BH\sumk\Ind_k^c.
    \end{align*}
    Observe that $\sumk\Ind_k^c=\tilo{SA}$ by definition. Moreover,
    \begin{align*}
        &\sumk\sumh\sum_{s'}\epsilon^{\star}_{k+1}(s^k_h, a^k_h, s')|z^k_h(s')| \\
        &= \tilO{\sumk\sumh\sum_{s'}\sqrt{\frac{P^k_h(s')z^k_h(s')^2}{\Np_{k+1}(s^k_h, a^k_h)}} + \sumk\sumh\frac{SB}{\Np_{k+1}(s^k_h, a^k_h)}} \tag{definition of $\epsilon_k^{\star}$}\\
        &= \tilO{ \sumk\sumh\sqrt{\frac{S\fV(P^k_h, V_{h+1})}{\Np_{k+1}(s^k_h, a^k_h)}} + BS^2A }\\
        &= \tilO{ \sqrt{\sumk\sumh\frac{S}{\Np_{k+1}(s^k_h, a^k_h)}}\sqrt{\sumk\sumh\fV(P^k_h, V_{h+1})} + BS^2A } \tag{Cauchy-Schwarz inequality}\\
        &= \tilO{\sqrt{S^2A\sumk\sumh\fV(P^k_h, V_{h+1})} + BS^2A }.
    \end{align*}
    Plugging these back completes the proof.
\end{proof}

\begin{lemma}\citep[Lemma 1]{efroni2020optimistic}
    \label{lem:val diff H}
    For any policy $\pi\in(\Delta_{\calA})^{\calS\times[H]}$, two transitions $P, P'\in(\Delta(\mathcal{S}))^{\SA\times[H]}$, and $g\in\fR^{\SA}$, we have $V^{\pi,P}_{g, 1}(s) - V^{\pi,P'}_{g, 1}(s)=\mathbf{E}_{\pi, P'}\left[\sumh (P_{h}-P'_{h})(s_h, a_h)\cdot V^{\pi,P}_{g, h+1}|s_1=s\right]$.
\end{lemma}


\section{Notes and Open Problems}\label{bey:notes}

This chapter discussed the results of weakly communicating CMDPs. The analysis of the model-based algorithm is taken from \cite{chen2022learning}. Readers may refer to the cited work for more detailed proof and insights. No baseline comparisons are provided since the work mentioned here is the first in the weakly communicating (WC) setting. It is to be noted that no model-free algorithm is currently known that provides provable guarantees for WC CMDPs.  

It is known that the regret lower bound in CMDP is $\Omega(\sqrt{T})$. Results achieving this order have been exhibited for model-based and model-free unconstrained tabular setups in \cite{fruit2018efficient} and \cite{zhang2023sharper}, respectively. Specifically, \citep{zhang2023sharper} establishes the said result for WC MDPs. Moreover, as shown in Table \ref{table2}, lower bound achieving regret is also achievable in the tabular model-based constrained setup. However, the results apply only to ergodic CMDPs. Given this state-of-the-art, we have the following two open questions in the context of WC CMDPs: how to design $(i)$ a model-based algorithm having a regret guarantee better than $\tilde{\mathcal{O}}(T^{2/3})$ and $(ii)$ a model-free algorithm having any sublinear guarantee. Finally, we also note that the algorithm discussed in this chapter requires the knowledge of some parameters of the underlying CMDP (e.g., upper bound on the span). Designing a parameter-free algorithm is another interesting open direction.


\newpage
{\bf Acknowledgements}\label{chpt:ack}

The authors would like to thank Mridul Agarwal,  Amrit Singh Bedi, Alec Koppel, Ather Gattami, and Swetha Ganesh for discussions and comments on the manuscript. 


\backmatter  

\printbibliography

@inproceedings{wei2022provably,
	title={A provably-efficient model-free algorithm for infinite-horizon average-reward constrained Markov decision processes},
	author={Wei, Honghao and Liu, Xin and Ying, Lei},
	booktitle={Proceedings of the AAAI Conference on Artificial Intelligence},
	volume={36},
	number={4},
	pages={3868--3876},
	year={2022}
}

@article{dragomir2003survey,
	title={A survey on Cauchy-Bunyakovsky-Schwarz type discrete inequalities},
	author={Dragomir, Sever S},
	journal={J. Inequal. Pure Appl. Math},
	volume={4},
	number={3},
	pages={1--142},
	year={2003}
}

@article{serfling1974probability,
	title={Probability Inequalities for the Sum in Sampling Without Replacement},
	author={Serfling, Robert J},
	journal={The Annals of Statistics},
	volume={2},
	number={1},
	pages={39--48},
	year={1974},
	publisher={Institute of Mathematical Statistics}
}

@inproceedings{chen2022learning,
	title={Learning infinite-horizon average-reward Markov decision process with constraints},
	author={Chen, Liyu and Jain, Rahul and Luo, Haipeng},
	booktitle={International Conference on Machine Learning},
	pages={3246--3270},
	year={2022},
	organization={PMLR}
}

@inproceedings{bartlett2009regal,
	title={REGAL: a regularization based algorithm for reinforcement learning in weakly communicating MDPs},
	author={Bartlett, Peter L and Tewari, Ambuj},
	booktitle={Proceedings of the Twenty-Fifth Conference on Uncertainty in Artificial Intelligence},
	pages={35--42},
	year={2009}
}

@inproceedings{fruit2018efficient,
	title={Efficient bias-span-constrained exploration-exploitation in reinforcement learning},
	author={Fruit, Ronan and Pirotta, Matteo and Lazaric, Alessandro and Ortner, Ronald},
	booktitle={International Conference on Machine Learning},
	pages={1578--1586},
	year={2018},
	organization={PMLR}
}

@article{weissman2003inequalities,
  title={Inequalities for the L1 deviation of the empirical distribution},
  author={Weissman, Tsachy and Ordentlich, Erik and Seroussi, Gadiel and Verdu, Sergio and Weinberger, Marcelo J},
  year={2003}
}

@article{brantley2020constrained,
  title={Constrained episodic reinforcement learning in concave-convex and knapsack settings},
  author={Brantley, Kiant{\'e} and Dudik, Miro and Lykouris, Thodoris and Miryoosefi, Sobhan and Simchowitz, Max and Slivkins, Aleksandrs and Sun, Wen},
  journal={Advances in Neural Information Processing Systems},
  volume={33},
  pages={16315--16326},
  year={2020}
}

@book{lattimore2020bandit,
  title={Bandit algorithms},
  author={Lattimore, Tor and Szepesv{\'a}ri, Csaba},
  year={2020},
  publisher={Cambridge University Press}
}

@article{agarwal2022concave,
	title={Concave Utility Reinforcement Learning with Zero-Constraint Violations},
	author={Agarwal, Mridul and Bai, Qinbo and Aggarwal, Vaneet},
	journal={Transactions on Machine Learning Research},
	year={2022}
}

@inproceedings{agrawal2017optimistic,
  title={Optimistic posterior sampling for reinforcement learning: worst-case regret bounds},
  author={Agrawal, Shipra and Jia, Randy},
  booktitle={Advances in Neural Information Processing Systems},
  pages={1184--1194},
  year={2017}
}

@inproceedings{osband2013more,
  title={(More) efficient reinforcement learning via posterior sampling},
  author={Osband, Ian and Russo, Daniel and Van Roy, Benjamin},
  booktitle={Advances in Neural Information Processing Systems},
  pages={3003--3011},
  year={2013}
}

@book{puterman1994markov,
 author = {Puterman, Martin L.},
 title = {Markov Decision Processes: Discrete Stochastic Dynamic Programming},
 year = {1994},
 isbn = {0471619779},
 edition = {1st},
 publisher = {John Wiley \& Sons, Inc.},
 address = {New York, NY, USA},
}

@article{jaksch2010near,
  title={Near-optimal regret bounds for reinforcement learning},
  author={Jaksch, Thomas and Ortner, Ronald and Auer, Peter},
  journal={Journal of Machine Learning Research},
  volume={11},
  number={Apr},
  pages={1563--1600},
  year={2010}
}

@book{lawler2018introduction,
  title={Introduction to stochastic processes},
  author={Lawler, Gregory F},
  year={2018},
  publisher={Chapman and Hall/CRC}
}

@misc{ding2023convergence,
      title={Convergence and sample complexity of natural policy gradient primal-dual methods for constrained MDPs}, 
      author={Dongsheng Ding and Kaiqing Zhang and Jiali Duan and Tamer Başar and Mihailo R. Jovanović},
      year={2023},
      eprint={2206.02346},
      archivePrefix={arXiv},
      primaryClass={math.OC}
}

@article{bai2023learning,
    title={Learning General Parameterized Policies for Infinite Horizon Average Reward Constrained MDPs via Primal-Dual Policy Gradient Algorithm},
    author={Bai, Qinbo and Mondal, Washim Uddin and Aggarwal, Vaneet},
    journal={arXiv preprint arXiv:2402.02042},
    year={2024}
}

@article{ganesh2024variance,
  title={Variance-Reduced Policy Gradient Approaches for Infinite Horizon Average Reward Markov Decision Processes},
  author={Ganesh, Swetha and Mondal, Washim Uddin and Aggarwal, Vaneet},
  journal={arXiv preprint arXiv:2404.02108},
  year={2024}
}

@article{ganesh2024ranac,
  title={An Accelerated Multi-level Monte Carlo Approach for Average Reward Reinforcement Learning with General Policy Parametrization},
  author={Ganesh, Swetha and Aggarwal, Vaneet},
  journal={arXiv},
  year={2024}
}

@article{mondal2024sample,
  title={Sample-Efficient Constrained Reinforcement Learning with General Parameterization},
  author={Mondal, Washim Uddin and Aggarwal, Vaneet},
  journal={arXiv preprint arXiv:2405.10624},
  year={2024}
}

@article{bai2023provably,
	title={Provably Sample-Efficient Model-Free Algorithm for MDPs with Peak Constraints},
	author={Bai, Qinbo and Aggarwal, Vaneet and Gattami, Ather},
	journal={Journal of Machine Learning Research},
	volume={24},
	number={60},
	pages={1--25},
	year={2023}
}

@article{chen2023two,
	title={Two-tiered Online Optimization of Region-wide Datacenter Resource Allocation via Deep Reinforcement Learning},
	author={Chen, Chang-Lin and Zhou, Hanhan and Chen, Jiayu and Pedramfar, Mohammad and Aggarwal, Vaneet and Lan, Tian and Zhu, Zheqing and Zhou, Chi and Gasser, Tim and Ruiz, Pol Mauri and others},
	journal={arXiv preprint arXiv:2306.17054},
	year={2023}
}

@article{al2019deeppool,
  title={Deeppool: Distributed model-free algorithm for ride-sharing using deep reinforcement learning},
  author={Al-Abbasi, Abubakr O and Ghosh, Arnob and Aggarwal, Vaneet},
  journal={IEEE Transactions on Intelligent Transportation Systems},
  volume={20},
  number={12},
  pages={4714--4727},
  year={2019},
  publisher={IEEE}
}

@article{manchella2021passgoodpool,
  title={Passgoodpool: Joint passengers and goods fleet management with reinforcement learning aided pricing, matching, and route planning},
  author={Manchella, Kaushik and Haliem, Marina and Aggarwal, Vaneet and Bhargava, Bharat},
  journal={IEEE Transactions on Intelligent Transportation Systems},
  volume={23},
  number={4},
  pages={3866--3877},
  year={2021},
  publisher={IEEE}
}

@inproceedings{geng2020multi,
  title={A multi-agent reinforcement learning perspective on distributed traffic engineering},
  author={Geng, Nan and Lan, Tian and Aggarwal, Vaneet and Yang, Yuan and Xu, Mingwei},
  booktitle={2020 IEEE 28th International Conference on Network Protocols (ICNP)},
  pages={1--11},
  year={2020},
  organization={IEEE}
}

@inproceedings{gattami2021reinforcement,
  title={Reinforcement learning for constrained markov decision processes},
  author={Gattami, Ather and Bai, Qinbo and Aggarwal, Vaneet},
  booktitle={International Conference on Artificial Intelligence and Statistics},
  pages={2656--2664},
  year={2021},
  organization={PMLR}
}

@article{bai2022joint,
	title={Joint Optimization of Concave Scalarized Multi-Objective Reinforcement Learning with Policy Gradient Based Algorithm},
	author={Bai, Qinbo and Agarwal, Mridul and Aggarwal, Vaneet},
	journal={Journal of Artificial Intelligence Research},
	volume={74},
	pages={1565--1597},
	year={2022}
}

@inproceedings{bai2022achieving,
  title={Achieving zero constraint violation for constrained reinforcement learning via primal-dual approach},
  author={Bai, Qinbo and Bedi, Amrit Singh and Agarwal, Mridul and Koppel, Alec and Aggarwal, Vaneet},
  booktitle={Proceedings of the AAAI Conference on Artificial Intelligence},
  volume={36},
  number={4},
  pages={3682--3689},
  year={2022}
}

@inproceedings{bai2023regret,
	title={Regret analysis of policy gradient algorithm for infinite horizon average reward markov decision processes},
	author={Bai, Qinbo and Mondal, Washim Uddin and Aggarwal, Vaneet},
	booktitle={Proceedings of the AAAI Conference on Artificial Intelligence},
	year={2024}
}

@inproceedings{bai2023achieving,
	title={Achieving zero constraint violation for constrained reinforcement learning via conservative natural policy gradient primal-dual algorithm},
	author={Bai, Qinbo and Bedi, Amrit Singh and Aggarwal, Vaneet},
  booktitle={Proceedings of the AAAI Conference on Artificial Intelligence},
	year={2023}
}

@inproceedings{agarwal2022multi,
  title={Multi-objective reinforcement learning with non-linear scalarization},
  author={Agarwal, Mridul and Aggarwal, Vaneet and Lan, Tian},
  booktitle={Proceedings of the 21st International Conference on Autonomous Agents and Multiagent Systems},
  pages={9--17},
  year={2022}
}

@inproceedings{agarwal2022regret,
  title={Regret guarantees for model-based reinforcement learning with long-term average constraints},
  author={Agarwal, Mridul and Bai, Qinbo and Aggarwal, Vaneet},
  booktitle={Uncertainty in Artificial Intelligence},
  pages={22--31},
  year={2022},
  organization={PMLR}
}

@inproceedings{patel2024global,
  title={Global Optimality without Mixing Time Oracles in Average-reward RL via Multi-level Actor-Critic},
  author={Patel, Bhrij and Suttle, Wesley A and Koppel, Alec and Aggarwal, Vaneet and Sadler, Brian M and Bedi, Amrit Singh and Manocha, Dinesh},
  booktitle={International Conference on Machine Learning},
  year={2024}
}

@article{bai2022achievingc,
  	title={Achieving zero constraint violation for concave utility constrained reinforcement learning via primal-dual approach},
	author={Bai, Qinbo and Bedi, Amrit Singh and Agarwal, Mridul and Koppel, Alec and Aggarwal, Vaneet},
	journal={Journal of Artificial Intelligence Research},
	volume={78},
	pages={975--1016},
	year={2023}
}

@article{mondal2023mean,
  title={Mean-Field Control based Approximation of Multi-Agent Reinforcement Learning in Presence of a Non-decomposable Shared Global State},
  author={Mondal, Washim Uddin and Aggarwal, Vaneet and Ukkusuri, Satish V},
  journal={Transactions on Machine Learning Research},
  year={2023}
}

@article{gonzalez2023asap,
  title={ASAP: A Semi-Autonomous Precise System for Telesurgery during Communication Delays},
  author={Gonzalez, Glebys and Balakuntala, Mythra and Agarwal, Mridul and Low, Tomas and Knoth, Bruce and Kirkpatrick, Andrew W and McKee, Jessica and Hager, Gregory and Aggarwal, Vaneet and Xue, Yexiang and others},
  journal={IEEE Transactions on Medical Robotics and Bionics},
  year={2023},
  publisher={IEEE}
}

@inproceedings{mondal2023improved,
  title={Improved sample complexity analysis of natural policy gradient algorithm with general parameterization for infinite horizon discounted reward markov decision processes},
  author={Mondal, Washim U and Aggarwal, Vaneet},
  booktitle={International Conference on Artificial Intelligence and Statistics},
  pages={3097--3105},
  year={2024},
  organization={PMLR}
}

@article{russo2014learning,
	title={Learning to optimize via posterior sampling},
	author={Russo, Daniel and Van Roy, Benjamin},
	journal={Mathematics of Operations Research},
	volume={39},
	number={4},
	pages={1221--1243},
	year={2014},
	publisher={INFORMS}
}

@inproceedings{osband2017posterior,
	title={Why is posterior sampling better than optimism for reinforcement learning?},
	author={Osband, Ian and Van Roy, Benjamin},
	booktitle={International conference on machine learning},
	pages={2701--2710},
	year={2017},
	organization={PMLR}
}

@article{singh2020learning,
  title={Learning in Markov decision processes under constraints},
  author={Singh, Rahul and Gupta, Abhishek and Shroff, Ness B},
  journal={arXiv preprint arXiv:2002.12435},
  year={2020}
}

@inproceedings{ding2021provably,
  title={Provably efficient safe exploration via primal-dual policy optimization},
  author={Ding, Dongsheng and Wei, Xiaohan and Yang, Zhuoran and Wang, Zhaoran and Jovanovic, Mihailo},
  booktitle={International Conference on Artificial Intelligence and Statistics},
  pages={3304--3312},
  year={2021},
  organization={PMLR}
}

@book{altman1999constrained,
  title={Constrained Markov decision processes},
  author={Altman, Eitan},
  volume={7},
  year={1999},
  publisher={CRC Press}
}

@inproceedings{zheng2020constrained,
  title={Constrained upper confidence reinforcement learning},
  author={Zheng, Liyuan and Ratliff, Lillian},
  booktitle={Learning for Dynamics and Control},
  pages={620--629},
  year={2020},
  organization={PMLR}
}

@book{puterman2014markov,
  title={Markov decision processes: discrete stochastic dynamic programming},
  author={Puterman, Martin L},
  year={2014},
  publisher={John Wiley \& Sons}
}

@InProceedings{Alekh2020, 
 	title = {Optimality and Approximation with Policy Gradient Methods in Markov Decision Processes}, 
 	author = {Agarwal, Alekh and Kakade, Sham M and Lee, Jason D and Mahajan, Gaurav}, 
 	booktitle = {Proceedings of Thirty Third Conference on Learning Theory}, 
 	pages = {64--66}, 
 	year = {2020}, 
 	editor = {Jacob Abernethy and Shivani Agarwal}, 
 	volume = {125}, 
 	series = {Proceedings of Machine Learning Research}, 
 	address = {}, 
 	month = {09--12 Jul}, 
 	publisher = {PMLR}, url = {http://proceedings.mlr.press/v125/agarwal20a.html}, 
}

@ARTICLE{altman1991constrained,
  author={Altman, E. and Schwartz, A.},
  journal={IEEE Transactions on Automatic Control}, 
  title={Adaptive control of constrained Markov chains}, 
  year={1991},
  volume={36},
  number={4},
  pages={454-462},
  doi={10.1109/9.75103}
}

@article{efroni2020exploration,
  title={Exploration-exploitation in constrained mdps},
  author={Efroni, Yonathan and Mannor, Shie and Pirotta, Matteo},
  journal={arXiv preprint arXiv:2003.02189},
  year={2020}
}

@article{agarwal2023reinforcement,
	title={Reinforcement Learning for Joint Optimization of Multiple Rewards},
	author={Agarwal, Mridul and Aggarwal, Vaneet},
	journal={Journal of Machine Learning Research},
	volume={24},
	number={49},
	pages={1--41},
	year={2023}
}

@article{ding2020natural,
  title={Natural Policy Gradient Primal-Dual Method for Constrained Markov Decision Processes},
  author={Ding, Dongsheng and Zhang, Kaiqing and Basar, Tamer and Jovanovic, Mihailo},
  journal={Advances in Neural Information Processing Systems},
  volume={33},
  year={2020}
}

@inproceedings{wei2021provably,
  title={Triple-Q: A Model-Free Algorithm for Constrained Reinforcement Learning with Sublinear Regret and Zero Constraint Violation},
  author={Wei, Honghao and Liu, Xin and Ying, Lei},
  booktitle={International Conference on Artificial Intelligence and Statistics},
  pages={3274--3307},
  year={2022},
  organization={PMLR}
}

@inproceedings{yu2021morl,
  title = 	 {Provably Efficient Algorithms for Multi-Objective Competitive RL},
  author =       {Yu, Tiancheng and Tian, Yi and Zhang, Jingzhao and Sra, Suvrit},
  booktitle = 	 {Proceedings of the 38th International Conference on Machine Learning},
  pages = 	 {12167--12176},
  year = 	 {2021},
  editor = 	 {Meila, Marina and Zhang, Tong},
  volume = 	 {139},
  series = 	 {Proceedings of Machine Learning Research},
  month = 	 {18--24 Jul},
  publisher =    {PMLR}
}

@inproceedings{langford2002approximately,
  title={Approximately optimal approximate reinforcement learning},
  author={Langford, J and Kakade, S},
  booktitle={Proceedings of ICML},
  year={2002}
}

@article{ouyang2017learning,
  title={Learning unknown markov decision processes: A thompson sampling approach},
  author={Ouyang, Yi and Gagrani, Mukul and Nayyar, Ashutosh and Jain, Rahul},
  journal={Advances in neural information processing systems},
  volume={30},
  year={2017}
}

@article{liu2021learning,
  title={Learning policies with zero or bounded constraint violation for constrained mdps},
  author={Liu, Tao and Zhou, Ruida and Kalathil, Dileep and Kumar, Panganamala and Tian, Chao},
  journal={Advances in Neural Information Processing Systems},
  volume={34},
  pages={17183--17193},
  year={2021}
}

@inproceedings{chen2021improved,
  title={Improved no-regret algorithms for stochastic shortest path with linear mdp},
author={Chen, Liyu and Jain, Rahul and Luo, Haipeng},
booktitle={International Conference on Machine Learning},
pages={3204--3245},
year={2022},
organization={PMLR}
}

@article{chen2021implicit,
  title={Implicit Finite-Horizon Approximation and Efficient Optimal Algorithms for Stochastic Shortest Path},
  author={Chen, Liyu and Jafarnia-Jahromi, Mehdi and Jain, Rahul and Luo, Haipeng},
  journal={Advances in Neural Information Processing Systems},
  year={2021}
}

@InProceedings{efroni2020optimistic, title = {Optimistic Policy Optimization with Bandit Feedback}, author = {Shani, Lior and Efroni, Yonathan and Rosenberg, Aviv and Mannor, Shie}, booktitle = {Proceedings of the 37th International Conference on Machine Learning}, pages = {8604--8613}, year = {2020}}

@article{cohen2021minimax,
  title={Minimax Regret for Stochastic Shortest Path},
  author={Cohen, Alon and Efroni, Yonathan and Mansour, Yishay and Rosenberg, Aviv},
  journal={Advances in Neural Information Processing Systems},
  year={2021}
}

@inproceedings{wei2020model,
  title={Model-free reinforcement learning in infinite-horizon average-reward {M}arkov decision processes},
  author={Wei, Chen-Yu and Jahromi, Mehdi Jafarnia and Luo, Haipeng and Sharma, Hiteshi and Jain, Rahul},
  booktitle={International Conference on Machine Learning},
  pages={10170--10180},
  year={2020},
  organization={PMLR}
}

@InProceedings{cohen2020near, title = {Near-optimal Regret Bounds for Stochastic Shortest Path}, author = {Cohen, Alon and Kaplan, Haim and Mansour, Yishay and Rosenberg, Aviv}, booktitle = {Proceedings of the 37th International Conference on Machine Learning}, pages = {8210--8219}, year = {2020}, volume = {119}, publisher = {PMLR} }

@inproceedings{zhang2023sharper,
	title={Sharper Model-free Reinforcement Learning for Average-reward Markov Decision Processes},
	author={Zhang, Zihan and Xie, Qiaomin},
	booktitle={The Thirty Sixth Annual Conference on Learning Theory},
	pages={5476--5477},
	year={2023},
	organization={PMLR}
}

@inproceedings{dorfman2022adapting,
  title={Adapting to mixing time in stochastic optimization with Markovian data},
  author={Dorfman, Ron and Levy, Kfir Yehuda},
  booktitle={International Conference on Machine Learning},
  pages={5429--5446},
  year={2022},
  organization={PMLR}
}

@article{liu2020improved,
  title={An improved analysis of (variance-reduced) policy gradient and natural policy gradient methods},
  author={Liu, Yanli and Zhang, Kaiqing and Basar, Tamer and Yin, Wotao},
  journal={Advances in Neural Information Processing Systems},
  volume={33},
  pages={7624--7636},
  year={2020}
}

@inproceedings{ghosh2022achieving,
  title={Achieving Sub-linear Regret in Infinite Horizon Average Reward Constrained MDP with Linear Function Approximation},
  author={Ghosh, Arnob and Zhou, Xingyu and Shroff, Ness},
  booktitle={The Eleventh International Conference on Learning Representations},
  year={2023}
}

@inproceedings{pesquerel2022imed,
  title={IMED-RL: Regret optimal learning of ergodic Markov decision processes},
  author={Pesquerel, Fabien and Maillard, Odalric-Ambrym},
  booktitle={NeurIPS 2022-Thirty-sixth Conference on Neural Information Processing Systems},
  year={2022}
}

@inproceedings{gong2020duality,
  title={A Duality Approach for Regret Minimization in Average-Award Ergodic Markov Decision Processes},
  author={Gong, Hao and Wang, Mengdi},
  booktitle={Learning for Dynamics and Control},
  pages={862--883},
  year={2020},
  organization={PMLR}
}

@article{agarwal2021theory,
  title={On the theory of policy gradient methods: Optimality, approximation, and distribution shift},
  author={Agarwal, Alekh and Kakade, Sham M and Lee, Jason D and Mahajan, Gaurav},
  journal={The Journal of Machine Learning Research},
  volume={22},
  number={1},
  pages={4431--4506},
  year={2021},
  publisher={JMLRORG}
}

@article{Mengdi2021,
 title={On the convergence and sample efficiency of variance-reduced policy gradient method},
author={Zhang, Junyu and Ni, Chengzhuo and Szepesvari, Csaba and Wang, Mengdi and others},
journal={Advances in Neural Information Processing Systems},
volume={34},
pages={2228--2240},
year={2021}
}

@InProceedings{Lingxiao2019,
    title={Neural Policy Gradient Methods: Global Optimality and Rates of Convergence},
    author={Lingxiao Wang and Qi Cai and Zhuoran Yang and Zhaoran Wang},
    booktitle={International Conference on Learning Representations},
    year={2020},
}

@InProceedings{Chi2019, 
	title = {Provably efficient reinforcement learning with linear function approximation}, author = {Jin, Chi and Yang, Zhuoran and Wang, Zhaoran and Jordan, Michael I}, 
	booktitle = {Proceedings of Thirty Third Conference on Learning Theory}, 
	pages = {2137--2143}, 
	year = {2020}, 
	editor = {Jacob Abernethy and Shivani Agarwal}, 
	volume = {125}, 
	series = {Proceedings of Machine Learning Research}, 
	address = {}, 
	month = {09--12 Jul}, 
	publisher = {PMLR}, 
}

@article{ling2023cooperating,
  title={Cooperating graph neural networks with deep reinforcement learning for vaccine prioritization},
  author={Ling, Lu and Mondal, Washim Uddin and Ukkusuri, Satish V},
  journal={IEEE Journal of Biomedical and Health Informatics},
  year={2024},
  publisher={IEEE}
}

@article{sutton1999policy,
  title={Policy gradient methods for reinforcement learning with function approximation},
  author={Sutton, Richard S and McAllester, David and Singh, Satinder and Mansour, Yishay},
  journal={Advances in neural information processing systems},
  volume={12},
  year={1999}
}

\end{document}